%% file: STOC_submission_revised.tex
\documentclass[11pt]{article} 

\usepackage{amsfonts}
\usepackage{xcolor}
\usepackage{amsthm}

\usepackage{amsmath}
\usepackage{amssymb}
\usepackage[numbers,sort&compress]{natbib}

\usepackage{multirow}
\usepackage{mathtools}
\usepackage{graphicx}
\usepackage{url}
\usepackage{algorithm2e}
\PassOptionsToPackage{algo2e,ruled}{algorithm2e}
\usepackage{amsthm}
\usepackage[margin=1in]{geometry}
\usepackage{hyperref}
\usepackage[capitalize]{cleveref}
\hypersetup{colorlinks=true,linkcolor=blue,citecolor=blue}

\usepackage[mathscr]{eucal}
\usepackage{thm-restate}

\newtheorem{theorem}{Theorem}
\newtheorem{lemma}[theorem]{Lemma}
  
    \newtheorem{definition}{Definition}
  \newtheorem{proposition}[theorem]{Proposition}
  \newtheorem{remark}[theorem]{Remark}

\newcommand{\comment}[1]{}

\title{Agnostic Smoothed Online Learning\\ without Knowledge of the Base Measure}

\author{
  \textbf{Mo\"ise Blanchard}\\
  Columbia University\\
  \small{\texttt{mb5414@columbia.edu}}
}

\date{}

\usepackage{thmtools}
\usepackage{thm-restate}
\usepackage[mathscr]{eucal}
\usepackage{bbm}

\newcommand{\acks}[1]{\section*{Acknowledgments}#1}

\newcommand{\nonl}
{\renewcommand{\nl}{\let\nl\oldnl}}

\usepackage{latexsym,tikz,graphicx}
\usetikzlibrary{calc}
\usetikzlibrary{decorations.pathreplacing}
\usetikzlibrary{patterns}

\renewenvironment{proof}[1][]{\par\noindent{\bf Proof #1\ }}{\hfill$\blacksquare$\\[2mm]}

\begin{document}

\include{shortcuts}

\maketitle
\thispagestyle{empty}

\begin{abstract}
    Classical results in statistical learning typically consider two extreme data-generating models: i.i.d.\ instances from an unknown distribution, or fully adversarial instances, often much more challenging statistically. To bridge the gap between these models, recent work introduced the \emph{smoothed} framework, in which at each iteration an adversary generates an instance from a distribution constrained to have density bounded by $\sigma^{-1}$ compared to some fixed base measure $\mu$. This framework interpolates between the i.i.d. and adversarial cases, depending on the value of $\sigma$. For the classical online prediction problem, most prior results in smoothed online learning rely on the arguably strong assumption that the base measure $\mu$ is \emph{known} to the learner, contrasting with standard settings in the PAC learning or consistency literature. We consider the general \emph{agnostic} problem in which the base measure is \emph{unknown} and values are arbitrary. In this direction, \cite{block2024performance} showed that empirical risk minimization has sublinear regret under the \emph{well-specified} assumption. We propose an algorithm \textsc{R-Cover} based on recursive coverings which is the first to guarantee sublinear regret for agnostic smoothed online learning without prior knowledge of $\mu$ and without the well-specified assumption. For classification, we prove that \textsc{R-Cover} has adaptive regret $\tilde\Ocal(\sqrt{dT/\sigma})$ for function classes with VC dimension $d$, which is optimal up to logarithmic factors. For regression, we establish that \textsc{R-Cover} has sublinear oblivious regret for function classes with polynomial fat-shattering dimension growth.
\end{abstract}

\tableofcontents

\section{Introduction}

We study the classical prediction problem in which a learner sequentially observes an instance $x_t\in\Xcal$ and makes a prediction about a value $y_t\in \Ycal$ before observing the true value. The learner's objective is to minimize the error of its predictions $\hat y_t$ compared to the true value $y_t$, given by some known loss function. We focus on both classification with $\Ycal=\{0, 1\}$ and regression with $\Ycal=[0,1]$, but for ease of presentation the present discussion mostly concerns classification. A major question in statistical learning theory is to understand under which assumptions on the data generating process and in particular on the process generating instances $(x_t)_{t\geq 1}$, can one give learning guarantees in the sense that the learner incurs low excess loss compared to some benchmark function class $\Fcal$. Most of the literature focused on either of the two following settings.

On one extreme, one can consider that the sequence $(x_t)_{t\geq 1}$ is fully adversarial and may depend on the actions of the learner. In this case, classical results \citep{littlestone1988learning,ben2009agnostic} show that the best one can hope for is to achieve low excess loss compared to function classes with finite \emph{Littlestone dimension}. This is quite restrictive, for instance, this precludes positive results even for the simple function class of threshold functions $\{x\in[0,1]\mapsto\1_{x\geq x_0}, x_0\in[0,1]\}$ on $\Xcal=[0,1]$.

On the other extreme, one can suppose that the instance sequence $(x_t)_{t\geq 1}$ is i.i.d.\ typically under some unknown distribution $\mu$. In the PAC learning setting \citep{vapnik1971uniform,vapnik1974theory,valiant1984theory}, one can ensure low excess error compared to function classes with finite \emph{VC dimension} (see \cref{def:VC_dimension}) which is significantly weaker than having finite Littlestone dimension. For instance, this covers the class of linear separators for say $\Xcal=\Rbb^d$ for $d\geq 1$. In regression, this can be replaced with the notion of \emph{fat-shattering dimension} (see \cref{def:fat_shattering}) \citep{bartlett1994fat,kearns1994efficient}, which is a scale-dependent version of the VC dimension. 
In fact, when the data generating process is i.i.d.\ one can achieve consistency---vanishing average excess loss---without further function class assumptions\footnote{Note that this differs from the PAC learning setting in the sense that guarantees are asymptotic.}. For instance, in classification and when the instance space $\Xcal$ is Euclidean, the simple $k$-nearest neighbor algorithm is already consistent \citep{devroye1994strong,devroye2013probabilistic,gyorfi:02} under reasonable choices of $k(t)$. Similar consistency results can also be achieved for general spaces \citep{hanneke2021bayes,gyorfi2021universal}.

Ideally, one would aim to obtain similar guarantees as for the more amenable i.i.d.\ case under weaker statistical assumptions. Notably, there has been significant work to establish consistency results under non-i.i.d.\ instance processes $(x_t)_{t\geq 1}$, including relaxations of the i.i.d.\ assumption such as stationary ergodic processes \citep{morvai1996nonparametric,gyorfi1999simple,gyorfi:02} or processes satisfying some form of law of large numbers \citep{gray2009probability,steinwart2009learning}. More recently, \cite{hanneke2021learning} initiated a line of work on universal learning to characterize minimal assumptions on instance processes $(x_t)_{t\geq 1}$ for consistency \cite{blanchard2021universal, blanchard2022universal,blanchard2023universal,blanchard2023contextual,blanchard2023adversarial}. These results are however mostly asymptotic in nature.

\paragraph{Smoothed online learning.} To interpolate between the adversarial and i.i.d.\ case while preserving quantitative convergence rates, \cite{rakhlin2011online} introduced the setting of \emph{smoothed online learning}. In this setting, one supposes that the process $(x_t)_{t\geq 1}$ is generated from some limited adversary that samples $x_t\sim \mu_t$ according to some distribution $\mu_t$ conditional on the history, constrained to have density bounded by $1/\sigma$ with respect to some fixed distribution $\mu$ (see \cref{def:smooth_sequence}). Here, $\sigma\in[0,1]$ is a parameter quantifying the smoothness of the adversary. Effectively this corresponds to a setting where the instances chosen by the adversary do not put too much mass on regions with low $\mu$-probability, which restricts the power of the adversary to explore unrelated regions of the space. Depending on the smoothness parameter $\sigma$, smoothed online learning interpolates exactly between the adversarial setting ($\sigma=0$) and the i.i.d.\ setting ($\sigma=1$). Recent works showed that many of the positive results from the i.i.d.\ case can be achieved under smooth adversaries up to paying a reasonable price in the smoothness constraint $1/\sigma$, covering a wide variety of settings from standard classification and regression \cite{rakhlin2011online,block2022smoothed,haghtalab2022oracle,block2023sample,haghtalab2024smoothed}, sequential probability assignment \cite{bhatt2024smoothed}, learning in auctions \cite{durvasula2023smoothed,cesa2023repeated}, robotics \cite{block2022efficient,block2023oracle}, differential privacy \cite{haghtalab2020smoothed}, and reinforcement learning \cite{xie2022role}.

In particular, \cite{rakhlin2011online} presented a general framework for analyzing minimax regret against smooth adversaries in terms of a distribution-dependent sequential Rademacher complexity. Then, \cite{haghtalab2024smoothed,block2022smoothed} provided tight regret bounds for smoothed online learning for classification and regression respectively, under the core assumption that the base measure is known. As an important note, the notion of smoothness in terms of bounded Radon-Nikodym density with respect to the base measure can usually be generalized to general divergence balls as studied in \cite{block2023sample}.

\paragraph{Agnostic smoothed online learning.} Crucially, the above-mentioned works on the standard smooth online learning problem assume that the base measure $\mu$ is \emph{known} to the learner. Arguably, this is a somewhat strong assumption both in practice and in theory. Knowing the base measure significantly diverges from classical results in the PAC learning setting for which knowing the distribution of the data is unnecessary, or from results from the literature on consistency which require no prior knowledge on the data-generating process. Hence, we aim to answer the following:
\begin{center}
    \textit{Can we achieve sublinear regret for smoothed online learning without prior knowledge of the base measure? If so, which algorithm achieves the optimal excess error guarantee?}
\end{center}

Along this direction, \cite{block2024performance} notably showed that if the values $(y_t)_{t\geq 1}$ are well-specified, i.e., given a function class $\Fcal$, there exists some $f^\star\in\Fcal$ such that $\Ebb[y_t\mid x_t] = f^\star(x_t)$ for all $t\geq 1$, then empirical risk minimization (ERM) has a regret guarantee of the form $\sigma^{-1}  \sqrt{\text{comp}(\Fcal)\cdot T}$ for some complexity notion for the function class $\text{comp}(\Fcal)$ (see \cref{thm:main_result_block} for a complete statement). Importantly, ERM does not require any prior knowledge of the base measure. In terms of lower bounds, \cite{block2022smoothed} showed that some polynomial dependency of the regret in $\sigma^{-1}$ is necessary as opposed to the setting in which $\mu$ in known for which the regret usually depends on $\ln (\sigma^{-1})$. 

We focus on the general setting in which no assumptions are made on the values $(y_t)_{t\geq 1}$ selected by the adversary, and the learner has no prior knowledge on the base measure, which we refer to as the \emph{agnostic} smoothed online learning setting. As before, the goal is to achieve low regret compared to a fixed function class $\Fcal$.

\paragraph{Contributions.}
We answer positively to the previous question by providing a \emph{proper} algorithm \textsc{R-Cover} (Recursive Covering) that achieves the optimal regret guarantee in classification for function classes $\Fcal$ with finite VC dimension up to logarithmic factors (\cref{thm:main_thm}), and sublinear regret in regression for function classes with standard fat-shattering dimension growth (\cref{thm:main_regret_regression}). To the best of our knowledge, this is the first algorithm with sublinear regret guarantees for the general agnostic online learning problem without prior knowledge of the base measure. $\textsc{R-Cover}$ also does not require the knowledge of the smoothness parameter $\sigma$. 

Our main result is easiest to present for classification. Namely, when $\Fcal:\Xcal\to\{0, 1\}$ has VC dimension $d$, we prove that $\textsc{R-Cover}$ achieves the following regret guarantee:
\begin{equation*}
    \Ebb\sqb{\sum_{t=1}^T \ell_t(\hat y_t) - \inf_{f\in\Fcal} \sum_{t=1}^T \ell_t(f(x_t)) } =\tilde\Ocal\paren{\sqrt {\frac{dT }{\sigma}} },
\end{equation*}
where $\tilde\Ocal$ hides poly-logarithmic factors in $T$ only. This matches a lower bound for VC classes up to logarithmic factors concurrently obtained by the authors from \cite{block2024performance}.\comment{(confirmed via personal communication). }
In particular, \textsc{R-Cover} has optimal dependency in $T$, $d$, but also in the smoothness parameter $\sigma$. More precisely, there is a function class of VC dimension $d$ for which any learning algorithm must incur an expected regret $\sqrt{dT/\sigma}$ for some smooth adversary (\cref{thm:lower_bound}). This lower bound holds even in the \emph{realizable} setting (well-specified and noiseless) in which there exists some function $f^\star\in\Fcal$ fixed a priori for which $y_t=f^\star(x_t)$ for all $t\geq 1$, and the loss is fixed over time.

The proof of the regret guarantees of \textsc{R-Cover} crucially relies on a novel property that we prove for smooth adversaries (see \cref{prop:simplified,lemma:main_bound_modified}). At a high level, this tightly bounds the possible amount of exploration of unknown regions of the instance space for smooth adversaries, which may be of broader interest.
\comment{This may be of broader interest for smoothed analysis without prior knowledge of the base measure, or for understanding which relaxations of the smoothness assumption could be made while preserving regret guarantees.}

\section{Preliminaries}

\paragraph{Formal setup.}

Let $\Xcal$ be an instance space equipped with some sigma-algebra. The function class $\Fcal$ is a set of measurable functions $f:\Xcal\to[0,1]$. We fix a horizon $T\geq 1$ and consider the following sequential prediction task. At each iteration $t\in[T]$,
\begin{enumerate}
    \item An adversary chooses a distribution $\mu_t$ on $\Xcal$ depending on all history, samples $x_t\sim \mu_t$ independently from the history, then chooses a $1$-Lipschitz loss function $\ell_t:[0,1]\to [0,1]$ depending on $x_t$ and the history.
    \item The learner observes $x_t$ and makes a prediction $\hat y_t\in[0,1]$.
    \item The learner observes $\ell_t$ and incurs the loss $\ell_t(\hat y_t)$.
\end{enumerate}

This captures the standard online prediction problem in which there is a fixed $1$-Lipschitz loss $\ell:[0,1]\times [0,1]\to [0,1]$ and the loss of the learner is equal to $\ell(\hat y_t,y_t)$ for some value $y_t$ that is revealed after the prediction $\hat y_t$. 
\comment{Indeed, the adversary may choose the loss $\ell_t(\cdot) = \ell(\cdot,y_t)$ in step $1$. }
Next, we say that the learner is \emph{proper} if at each iteration $t\in[T]$, before observing the query $x_t$, the learner first commits to a function $\hat f_t\in \Fcal$ then, upon observing $x_t$, predicts the value $\hat y_t=\hat f_t(x_t)$. Our proposed algorithms will enjoy this property.

\comment{
\begin{definition}[Proper learners]\label{def:proper_learner}
    We say that the learner is \emph{proper} if at each iteration $t\in[T]$, before observing the query $x_t$, they first commit to a function $\hat f_t\in \Fcal$ then, upon observing $x_t$, predict the value $\hat y_t=\hat f_t(x_t)$.
\end{definition}
}

The smoothness assumption constrains the distributions $\mu_t$ chosen by the adversary.

\begin{definition}[Smooth distributions and smooth adversaries]\label{def:smooth_sequence}
    Let $\mu,p$ be probability measures on $\Xcal$. We say that $p$ is \emph{$\sigma$-smooth} with respect to $\mu$ if $\|\frac{dp}{d\mu}\|_\infty \leq 1/\sigma$,
    where $\|\cdot\|_\infty$ denotes the essential supremum.
    We say that an adversary is \emph{$\sigma$-smooth} with respect to the base measure $\mu$ if for any $t\in[T]$, the distribution $\mu_t$ selected by the adversary in step 1 above is $\sigma$-smooth with respect to $\mu$.
\end{definition}


The goal of the learner is to minimize their regret, that is, the excess error compared to the benchmark functions in $\Fcal$. Precisely, we distinguish between the expected \emph{adaptive} regret
\begin{equation*}
    \Ebb\sqb{\sum_{t=1}^T \ell_t(\hat y_t) - \inf_{f\in\Fcal} \sum_{t=1}^T \ell_t(f(x_t)) },
\end{equation*}
in which the benchmark function may depend on the specific realizations of the learning process, and the expected \emph{oblivious} regret in which the benchmark function is fixed a priori:
\begin{equation*}
    \Ebb\sqb{\sum_{t=1}^T \ell_t(\hat y_t) }- \inf_{f\in\Fcal}\Ebb\sqb{ \sum_{t=1}^T \ell_t(f(x_t)) },
\end{equation*}
Adaptive benchmarks are known to require significantly stronger analysis than oblivious benchmarks for smoothed online learning (e.g. see \cite{haghtalab2024smoothed}).

\paragraph{Complexity notions for the function class and prior results.}
\label{subsec:complexity_notion}

In classification, i.e., when the functions take value in $\{0, 1\}$, and when the instance process is i.i.d.\ ($\sigma=0$) it is known that in our setup, learnability is characterized by the VC dimension \citep{vapnik1971uniform,vapnik1974theory,valiant1984theory}.

\begin{definition}[VC dimension]\label{def:VC_dimension}
    Let $\Fcal:\Xcal\to\{0, 1\}$ be a function class. We say that $\Fcal$ \emph{shatters} a set of points $\{x_1,\ldots,x_m\}\subset\Xcal$ if for any choice of values $\epsilon\in\{0, 1\}^m$ there exists $f_\epsilon\in\Fcal$ such that $f_\epsilon(x_i)=\epsilon_i$ for all $i\in[m]$. The \emph{VC dimension} of $\Fcal$ is the size of the largest shattered set.
\end{definition}

In the regression setting for which functions take value on the interval $[0,1]$, a scale-dependent analog characterizes the learnability of the function class $\Fcal$. This is known as the fat-shattering dimension of the class \citep{bartlett1994fat,kearns1994efficient}.

\begin{definition}[Fat-shattering dimension]\label{def:fat_shattering}
    Let $\Fcal:\Xcal\to[0,1]$ be a function class. Fix $r>0$. We say that $\Fcal$ $r$-shatters a set $\{x_1,\ldots,x_m\} \subset \Xcal$ if there exist $s_1,\ldots,s_m\in [0,1]$ such that for any signs $\epsilon\in\{\pm 1\}^m$ there exists $f_\epsilon\in\Fcal$ such that $\epsilon(f_\epsilon(x_i) - s_i) \geq r$ for all $i\in[m]$. The \emph{fat-shattering dimension} of $\Fcal$ at scale $r>0$, denoted $\mathsf{fat}_\Fcal(r)$ is the size of the largest $r$-shattered set.
\end{definition}

Next, we define the notion of covering set and covering numbers. We voluntarily restrict these notions to the empirical infinite norm, which is sufficient for this work.

\begin{definition}[Covering set and covering numbers]\label{def:covering_numbers}
    Let $\Fcal:\Xcal\to[0,1]$ be a function class for regression. Fix a set $S=\{x_1,\ldots,x_n\}\subset \Xcal$ and $\epsilon\geq 0$. We say that $\Ccal\subset \Fcal$ is an $\epsilon$-cover of $\Fcal$ on $S$ if for all $f\in \Fcal$ there exists $g\in \Ccal$ such that for all $i\in[n]$, 
    $|f(x_i)-g(x_i)| \leq \epsilon.$
    The $\epsilon$-covering number of $\Fcal$ on $S$, denoted $\Ncal(\Fcal;\epsilon,S)$ is the size of the smallest $\epsilon$-cover of $\Fcal$ on $S$.
\end{definition}

To state related results, we also need to define the Wills functional \cite{wills1973gitterpunktanzahl,hadwiger1975will} of $\Fcal$, which is a less standard complexity measure. The definition below uses the formulation from \cite{block2024performance}.

\begin{definition}[Wills functional]
    Fix values $Z_1,\ldots,Z_m\in\Xcal$ and let $\xi=(\xi_1,\ldots,\xi_m)$ be a vector of i.i.d.\ standard Gaussian random variables. The Wills functional of $\Fcal$ on $Z_1,\ldots,Z_m$ is defined as
    \begin{equation*}
        W_{m,Z}(\Fcal) := \Ebb_\xi\sqb{\exp\paren{\sup_{f\in\Fcal} \sum_{i=1}^m \xi_i f(Z_i) - \frac{1}{2} f(Z_i)^2}}.
    \end{equation*}
\end{definition}

The above definition depends on the choice of $Z_1,\ldots,Z_m$. For simplicity we may omit this dependency---most of the time we take its expectation for $Z_1,\ldots,Z_m\overset{iid}{\sim}\mu$.
Properties of the Wills functional have been extensively studied \cite{wills1973gitterpunktanzahl,hadwiger1975will,mcmullen1991inequalities,mourtada2023universal}. We refer to \cite{mourtada2023universal} for detailed connections with metric complexities and universal coding. \cref{sec:wills_functional} gives a brief overview of links between Wills functional and more standard measures complexities that are most relevant to this work. 
In particular, for any choice of $Z_1,\ldots,Z_m$, $\ln W_m(\Fcal) \lesssim d\ln m$ for classes $\Fcal$ with finite VC dimension. \cite{mourtada2023universal} also showed that $\ln W_m(\Fcal)\leq \Gcal_m(\Fcal)$ where $\Gcal_m(\Fcal)$ is the Gaussian complexity of $\Fcal$ (see \cref{sec:wills_functional} for a definition). Last, \cite{block2024performance} showed that having $\ln W_m(\Fcal)=o(m)$ is necessary and sufficient to ensure learnability with polynomially many samples when the data is i.i.d.

Now that we have defined the Wills functional, we can formally state the main result from \cite{block2024performance} which shows that empirical risk minimization (ERM) achieves sublinear regret without knowledge of the base measure for the \emph{well-specified} setting.

\begin{theorem}[Theorem 1 of \cite{block2024performance}]
\label{thm:main_result_block}
    Let $\Fcal:\Xcal\to[0,1]$ be a function a function class. Consider the squared loss regression setting in which $\ell_t(\cdot) = (\cdot - y_t)^2$ for a value $y_t\in \Rbb$. Suppose that there exists some function $f^\star\in\Fcal$ such that $(x_t)_{t\geq 1}$ is a $\sigma$-smooth sequence on $\Xcal$ and that the values are given via $y_t=f^\star(x_t) + \eta_t$ where $\eta_t$ is a mean-zero subgaussian random variable with variance proxy $\nu^2$, conditionally on the history up to time $t$. Then, ERM makes predictions $\hat y_t$ such that
    \begin{equation*}
        \Ebb\sqb{\sum_{t=1}^T (\hat y_t-f^\star(x_t))^2 } \leq  \frac{ 20 \ln^3 T}{\sigma} \sqrt {T(1+\nu)(1+\ln \Ebb_\mu\sqb{W_{2T\ln(T)/\sigma}(256\Fcal)} )}.
    \end{equation*}
\end{theorem}

\paragraph{Further definitions and notations.}

We define the notion of tangent sequence \cite{de2012decoupling} which will be useful within the proofs.
\begin{definition}[Tangent sequence]
    Let $(Z_t)_{t\geq 1}$ be a sequence of random variables adapted to a filtration $(\Fcal_t)_{t\geq 1}$. A \emph{tangent sequence} $(Z_t')_{t\geq 1}$ is a sequence of random variables such that $Z_t$ and $Z_t'$ are i.i.d.\ conditionally on $\Fcal_{t-1}$ (and independently of $Z_{t'}$ for $t'>t$).
\end{definition}

Throughout this work, we will use this notation with primes to denote tangent sequences. We also denote by $\Hcal_t$ the history at the end of iteration $t\geq 0$ of the learning process, which is the sigma-algebra generated by $(x_l,\hat y_l,\ell_l)_{l\leq t}$. In particular, $x_t\mid\Hcal_{t-1}\sim\mu_t$ where $\mu_t$ is the distribution selected by the adversary in step 1 of the learning process. 
We use the notation $[T]:=\{1,\ldots,T\}$. We write $\lesssim$ to signify that the inequality holds up to universal constants. Last unless mentioned otherwise, the notation $\tilde\Ocal$ only hides poly-logarithmic factors in $T$.

\section{Main results}

While our analysis provides regret bounds for general regression function classes, these are more easily stated for classification.

\begin{theorem}\label{thm:main_thm}
    Fix $T\geq 1$. Let $\Fcal:\Xcal\to\{0, 1\}$ be a function class with VC dimension $d$. Suppose that $(x_t)_{t\geq 1}$ is a $\sigma$-smooth sequence on $\Xcal$ with respect to some unknown base measure $\mu$. Then, $\textsc{R-Cover}$ makes predictions $\hat y_t$ such that
    \begin{equation*}
        \Ebb\sqb{\sum_{t=1}^T \ell_t(\hat y_t ) - \inf_{f\in\Fcal} \sum_{t=1}^T \ell_t(f(x_t) ) } \leq  C\ln^{5/2} T \sqrt {\frac{dT }{\sigma}},
    \end{equation*}
    for some universal constant $C>0$.
\end{theorem}

As a by-product of the analysis, we also provide a high-probability version of the above regret bound (see \cref{eq:high_probability_adaptive_bound_vc}).
Compared to the regret bound \cref{thm:main_result_block} which becomes $\sigma^{-1}\ln^{7/2} (T)\sqrt {dT}$ for VC classes, our regret bound holds for adversarial values $(y_t)_{t\in[T]}$ and has an improved dependency in $\sigma$: it grows as $1/\sqrt \sigma$ instead of $1/\sigma$. 
Our regret bound for \textsc{R-Cover} is complemented by a matching lower bound up to logarithmic factors, which holds even in the realizable noiseless setting. \comment{Confirmed by personal communication, }The authors from \cite{block2024performance} also generalized their lower bound for the regret empirical risk minimization (ERM) (Theorem 3) to general algorithms for VC classes, leading to the same result as below. We include the proof in \cref{sec:proof_lower_bound} for completeness. \comment{The proof strategy is also of independent interest and can be used to show that some of the properties we develop on smooth adversaries (\cref{prop:simplified,lemma:main_bound_modified}) are essentially tight. We refer to \cref{subsec:simpler_algo} for further discussion.}

\begin{restatable}{theorem}{LowerBound}
\label{thm:lower_bound}
    Fix $d\geq 1$. There exists a function class $\Fcal:\Xcal\to\{0, 1\}$ with VC dimension $d$ such that for any $\sigma\in(0,1)$, $T\geq 1$, and any learning algorithm, there is a function $f^\star \in \Fcal$ and a $\sigma$-smooth adversary such that the responses are realizable, that is, $y_t=f^\star(x_t)$ for all $t\in[T]$, and denoting by $\hat y_t$ the predictions of the algorithm,
    \begin{equation*}
        \Ebb\sqb{\sum_{t=1}^T \1 [\hat y_t \neq f^\star(x_t)] } \geq \min\paren{\frac{1}{12} \sqrt{\frac{dT(1-\sigma)}{\sigma}},\;  \frac{T}{24} }.
    \end{equation*}
\end{restatable}

\textsc{R-Cover} uses a somewhat complex recursive construction to achieve the optimal regret guarantee from \cref{thm:main_thm}. For intuition, we also describe in \cref{subsec:simpler_algo} a very simple and intuitive algorithm \textsc{Cover} that achieves (worse) sublinear regret without prior knowledge of the base measure. This algorithm which essentially corresponds to the single-depth version of \textsc{R-Cover} and for instance enjoys a $\approx T^{2/3}$ regret guarantee in classification, which may be of independent interest. Obtaining regret guarantees in regression for \textsc{Cover} is also possible with the same tools developed for \textsc{R-Cover} and we omit details for simplicity. 

We next turn to the general regression setting. At the high level, our algorithm for classification is generalized to regression by constructing $\epsilon$-coverings of the function class for some scale $\epsilon$ that is used as a parameter (for VC classes we simply use $\epsilon=0$). In practice, the optimal choice of the scale $\epsilon$ lies in $[1/T,1]$ and only depends on the growth of the fat-shattering dimensions of $\Fcal$. However, tuning this parameter $\epsilon$ can be fully side-stepped by performing any learning with expert advice algorithm using as experts the algorithms $\textsc{R-Cover}$ for different choice of parameters $\epsilon \in\{2^{-l},l\leq \log_2 T\}$. The resulting algorithm would enjoy the same regret guarantees as for the optimally-tuned algorithm.
The full version of our regret bound is stated in \cref{thm:oblivious_adversary}. For readability, we instantiate the bound for standard growth scenarios of the fat-shattering dimension.

\begin{theorem}\label{thm:main_regret_regression}
    Fix $T\geq 1$. Let $\Fcal:\Xcal\to[0,1]$ be a function class
    and suppose that $(x_t)_{t\geq 1}$ is a $\sigma$-smooth sequence on $\Xcal$ with respect to some unknown base measure $\mu$. 
    There exists a universal constant $C>0$ such that we have the following bounds on the oblivious regret of $\textsc{R-Cover}$, where we denote by $\hat y_t$ the predictions of the algorithm.

    If $\mathsf{fat}_\Fcal(r)\leq d\ln\frac{1}{r}$ for all $r>0$, then $\textsc{R-Cover}$ run with the parameter $\epsilon=1/T$ yields
    \begin{equation*}
        \Ebb\sqb{\sum_{t=1}^T \ell_t(\hat y_t ) } - \inf_{f\in\Fcal} \Ebb\sqb{\sum_{t=1}^T \ell_t(f(x_t) ) } \leq C \ln^3 T \sqrt{\frac{dT}{\sigma}}.
    \end{equation*}

    If $\mathsf{fat}_\Fcal(r)\lesssim r^{-p}$ for $p>0$, then $\textsc{R-Cover}$ run with the parameter $\epsilon = \paren{\frac{\ln T}{T}}^{\frac{1}{p+1}}$ yields
    \begin{equation*}
        \Ebb\sqb{\sum_{t=1}^T \ell_t(\hat y_t ) } - \inf_{f\in\Fcal} \Ebb\sqb{\sum_{t=1}^T \ell_t(f(x_t) ) } \leq C \frac{\ln^3 T}{\sqrt \sigma} \cdot T^{1-\frac{1}{2(p+1)}} \paren{1+ \tilde O(\sigma^{-\frac{1}{2}}T^{-\frac{\min(p,1)}{2(p+1)(p+2)}}) },
    \end{equation*}
    where $\tilde O$ only hides logarithmic factors in $T$.
\end{theorem}

Our analysis also provides high-probability versions of these bounds (see \cref{thm:oblivious_adversary}).
Note that the guarantee for classification from \cref{thm:main_thm} bounds the expected adaptive regret, while in the regression case, \cref{thm:main_regret_regression} bounds the expected oblivious regret. We leave open the question of whether one can achieve guarantees for the adaptive regret in regression.

In the rest of this section, we first construct in \cref{subsec:recursive_classification} the algorithm \textsc{R-Cover} instantiated for classification. This already provides most of the necessary intuitions and for ease of presentation, we defer the construction of the algorithm in the general regression case to \cref{subsec:description_algo_regression}. \textsc{R-Cover} requires a specific variant for a learning with expert advice algorithm defined in \cref{subsec:learning_with_experts}.

\subsection{Recursive construction of \textsc{R-Cover} for classification}
\label{subsec:recursive_classification}

In its simplest form, \textsc{R-Cover} divides the horizon $[T]$ in $K$ equal-length epochs and uses a learning with expert advice algorithm on each epoch on a subset of functions from $\Fcal$ that are representative from the data observed in previous epochs. 
While we can show that this achieves a sublinear regret (see \cref{subsec:simpler_algo} and \cref{thm:regret_simple_algo} for a detailed discussion), to achieve a $\approx\sqrt T$ regret, we need to use a recursive construction, which is parameterized by a depth parameter $P\geq 0$.

\comment{For this simpler version, we can for instance use the classical \emph{Hedge} algorithm \cite{cesa2006prediction} on the projection of $\Fcal$ on the queries observed on previous epochs.}

To ease the recursive construction, in addition to the start time $T_0$, the end time $T_1$, and the depth $P$ of the algorithm we introduce an additional parameter $S\subset \Xcal\times \{0,1\}$ which corresponds to some labeled dataset for previous queries: $S=\{(x_t,\tilde y_t),t\in [T_0]\}$ where $\tilde y_t\in\{0,1\}$ for all $t\in[T_0]$. We denote by $\textsc{R-Cover}_{T_0,T_1}^{(p)}(S)$ the corresponding algorithm. As an important constraint on $S$, the dataset must be \emph{realizable} by the class $\Fcal$. Formally, there must exist $f\in\Fcal$ such that $f(x)=y$ for all $(x,y)\in S$. Intuitively, this dataset incorporates prior information gathered on the problem. 

\comment{The final algorithm will correspond to the depth-$P$ recursive algorithm instantiated with $T_0=0$, $T_1=T$, and an empty dataset $S=\emptyset$.}

\paragraph{Recursive construction.}

\comment{
\begin{algorithm}[ht]

\caption{Recursive construction of $\textsc{R-Cover}_{T_0,T_1}^{(P)}(S)$}\label{alg:recursive_construction}

\LinesNumbered
\everypar={\nl}

\hrule height\algoheightrule\kern3pt\relax
\KwIn{depth $P\geq 0$, times $T_0\leq T_1$ satisfying $T_1-T_0\geq 2^P$, realizable dataset $S\subset \Xcal\times\{0,1\}$}

\vspace{3mm}

\lIf{$P=0$}{
    Fix $f_S\in\Fcal$ realizing dataset $S$ and predict $\hat y_t = f_S(x_t)$ for all $t\in(T_0,T_1]$
}
\Else{
    Fix $f_S\in\Fcal$ realizing dataset $S$ and let $T_{1/2}:=\floor{\frac{T_0+T_1}{2}}$

    \For{$\alpha\in\{0,1/2\}$ \tcp*[h]{running epoch $(T_\alpha,T_{\alpha+1/2}]$}  }{
        After iteration $T_\alpha$, construct all distinct realizable datasets $S_1,\ldots,S_r\subset \Xcal\times\{0,1\}$ obtained by adding labeled points $(x_t,\tilde y_t)_{t\in(T_0,T_\alpha]}$ for queries from previous epoch to $S$

        Perform $\textsc{A-Exp}$ (\cref{alg:exponentially_weighted}) on $(T_\alpha,T_{\alpha+1/2}]$ with experts $\set{\textsc{R-Cover}_{T_\alpha,T_{\alpha+1/2}}^{(P-1)}(S_{r'}),\;r'\in[r] } \cup\{f_S\}$
    }
    
}

\hrule height\algoheightrule\kern3pt\relax
\end{algorithm}
}

\begin{algorithm}[t]

\caption{Recursive construction of $\textsc{R-Cover}_{T_0,T_1}^{(P)}(S)$}\label{alg:recursive_construction}

\LinesNumbered
\everypar={\nl}

\hrule height\algoheightrule\kern3pt\relax
\KwIn{depth $P\geq 0$, times $T_0\leq T_1$ satisfying $T_1-T_0\geq 2^P$, realizable dataset $S\subset \Xcal\times\{0,1\}$}

\vspace{3mm}

Fix $f_S\in\Fcal$ realizing dataset $S$, i.e., $f_S(x)=y$ for all $(x,y)\in S$. Let $T_{1/2}:=\floor{\frac{T_0+T_1}{2}}$

\uIf{$P=0$}{
    \lFor{$t \in (T_0,T_1]$}{
        Predict $\hat y_t = f_S(x_t)$}
}
\Else{
    \For{$t\in(T_0,T_{1/2}]$}{
        Follow predictions of $\textsc{A-Exp}$ (\cref{alg:exponentially_weighted}) run on $(T_0,T_{1/2}]$ (initialized at $t=T_0+1$) and with experts $\set{\textsc{R-Cover}_{T_0,T_{1/2}}^{(P-1)}(S),f_S }$
        
    }

    Construct all distinct realizable datasets $S_1,\ldots,S_r\subset \Xcal\times\{0,1\}$ obtained by adding labeled points $(x_t,\tilde y_t)_{t\in(T_0,T_{1/2}]}$ for queries from the previous epoch to $S$, as defined in \cref{eq:def_construct_datasets}
    
    \For{$t\in(T_{1/2}, T_1]$}{
        Follow predictions of $\textsc{A-Exp}$ (\cref{alg:exponentially_weighted}) run on $(T_{1/2},T_1]$ (initialized at $t=T_{1/2}+1$) and with experts $\set{\textsc{R-Cover}_{T_{1/2},T_1}^{(P-1)}(S_{r'}),\;r'\in[r] } \cup\{f_S\}$
    }
}

\hrule height\algoheightrule\kern3pt\relax
\end{algorithm}

For the base depth $P=0$, given start and end times $T_0<T_1$ and a dataset $S$, the algorithm simply selects any arbitrary function $f_S\in\Fcal$ that agrees on the query set, that is $f_S(x)=y$ for all $(x,y)\in S$, and uses it as prediction at all times in $[T]$. We recall that $S$ intuitively encodes prior information gathered on the prediction problem, hence, the algorithm for $P=0$ simply corresponds to following this prior information irrespective of the information obtained during the learning process on $(T_0,T_1]$. For larger depths, the algorithm will make use of this information proceeding by epochs.

Suppose that we have defined all algorithms for depth $P-1$. Fix $T_0\leq T_1$ with $T_1-T_0\geq 2^P$, and a labeled dataset $S$. We also fix a function $f_S\in\Fcal$ realizing $S$ which will serve as base prediction function. 
First define $T_{1/2}:=\floor{(T_0+T_1)/2}$. This time divides the interval $(T_0,T_1]$ in two epochs $(T_0,T_{1/2}]$ and $(T_{1/2},T_1]$ of roughly equal length. Note that by construction, each epoch has length at least $2^{P-1}$. The algorithm proceeds separately on each epoch and more precisely aims to incorporate the information gathered on the first epoch $(T_0,T_{1/2}]$ when making predictions during the second epoch $(T_{1/2},T_1]$. The procedure is detailed below.
\begin{itemize}
    \item On the first epoch $(T_0,T_{1/2}]$, the algorithm performs a learning with expert advice algorithm using as experts the predictions of the previous-depth algorithm $\textsc{R-Cover}_{T_0,T_{1/2}}^{(P-1)}(S)$ as well as the base expert $f_S$. Note that the depth-$(P-1)$ $\textsc{R-Cover}$ algorithm uses exactly the same dataset $S$ and is run on the same epoch $(T_0,T_{1/2}]$.
    \item On the second epoch $(T_{1/2},T_1]$, the algorithm also performs a learning with expert advice algorithm but uses a larger set of experts, which will correspond to depth-$(P-1)$ $\textsc{R-Cover}$ algorithms run with different datasets. Precisely, at the beginning of the epoch (before the prediction at time $t=T_{1/2}+1$), we consider all possible ways to label the queries $(x_t)_{t\in(T_0,T_{1/2}]}$ such that together with $S$ the increased dataset is still realizable by the function class $\Fcal$. Formally, we enumerate all distinct datasets $S_1,\ldots,S_r$ such that
    \begin{multline}\label{eq:def_construct_datasets}
        \set{S_{r'},r'\in[r]} :=\set{S'=S\cup \set{ (x_t,\tilde y_t),t\in(T_0,T_{1/2}] } : \tilde y_t\in\{0,1\},t\in(T_0,T_{1/2}]}  \\
        \cap \set{ S'\subset \Xcal\times\{0,1\} : \exists f\in\Fcal, \forall (x,y)\in S', f(x)=y  }.
    \end{multline}
    This corresponds to all possible ways to extend the prior dataset $S$ on the observed data during the previous epoch. Intuitively, if the set $S$ was a good prior, in the sense that its labeled points coincide with the predictions of some optimal function $f^\star\in\Fcal$, one of the datasets $S_1,\ldots,S_r$ is still a good prior and now contains additional labels from queries on the previous epoch $(T_0,T_{1/2}]$. Because one does not know a priori which is the best dataset, we perform a learning with expert advice algorithm on the epoch $(T_{1/2},T_1]$ using the expert predictions from $\textsc{R-Cover}_{T_{1/2},T_1}^{(P-1)}(S_{r'})$ for all $r'\in[r]$, as well as the base expert $f_S$.
 \end{itemize}

For our purposes, we need a specific learning with expert advice algorithm $\textsc{A-Exp}$ (see \cref{alg:exponentially_weighted}). We defer its presentation to \cref{subsec:learning_with_experts} for readability.
This concludes the construction of the algorithm for depth $P$, horizon $T$, and dataset $S$, which is summarized in \cref{alg:recursive_construction}.

To simplify the analysis, we can unify the description and notations for the algorithm on both epochs $(T_0,T_{1/2}]$ and $(T_{1/2},T_1]$. We can check that for both epochs $(T_\alpha,T_{\alpha+1/2}]$ for $\alpha\in\{0,1/2\}$, the algorithm first constructs all distinct realizable datasets $S_1,\ldots,S_r$ obtained by adding labels for queries $(x_t)_{t\in (T_0,T_\alpha]}$ to $S$, then runs $\textsc{A-Exp}$ on the set of experts $\textsc{R-Cover}_{T_{\alpha},T_{\alpha+1/2}}^{(P-1)}(S_{r'})$ for all $r'\in[r]$, as well as the base expert $f_S$. For the first epoch we in fact have $r=1$ since there are no labeled queries to add to $S$. 
Importantly, because $\Fcal$ has VC dimension $d$, we always have $r\leq 2T^d+1$ from Sauer-Shelah's \cref{lemma:sauer_lemma}. The additional expert comes from the fact that we also added $f_S$ as expert.

\paragraph{Final algorithm.} 
We pose $P:=\floor{\log_2(T)}$. The final algorithm is then simply $\textsc{R-Cover}_{0,T}^{(P)}(\emptyset)$, that is, we initialize the depth-$P$ algorithm with an empty dataset.

\subsection{Learning with expert advice algorithm}
\label{subsec:learning_with_experts}
Instead of using the standard exponential weights algorithm \cite{cesa2006prediction} for learning with expert advice, we introduce a variant.
We briefly recall the setup of prediction with $K$ experts and fixed horizon $T$ that is relevant to our present discussion. At each iteration $t\in[T]$, the environment chooses losses $\ell_{t,i}$ for each experts $i\in[K]$. The learner then selects an expert $\hat i_t\in[K]$ potentially randomly without knowledge of the losses at time $t$. Then, all losses at time $t$ are revealed to the learner and they incur the loss $\ell_{t,\hat i_t}$ of the selected expert. The goal of the learner is to minimize the regret:
\begin{equation*}
    \text{Reg}(T):=\sum_{t=1}^T \ell_{t,\hat i_t} - \min_{i\in[K]}\sum_{t=1}^T \ell_{t,i}.
\end{equation*}

The classical \emph{exponential weights forecaster} or \emph{Hedge} algorithm (see e.g. \cite{cesa2006prediction}) with parameter $\eta>0$ proceeds as follows. At time $t$, it computes the cumulative regret compared to each expert up to time $t$: $R_{t-1,i}:=\sum_{l=1}^{t-1} \ell_{l,\hat i_l} - \ell_{l,i}$ for all $i\in[K]$. It then randomly samples $\hat i_t\sim p_t$ where the distribution $p_t=(p_{t,i})_{i\in[K]}$ is defined via exponential weights
$p_{t,i}:=e^{\eta R_{t-1,i}}/\sum_{j\in[K]}e^{\eta R_{t-1,j}}.$
We denote by $\Fcal_t=(\ell_{l,i},l\leq t,i\in[K], \hat i_l,l<t)$ the history up to time $t$ included.
Hedge enjoys the following classical bound (see e.g. \cite[Corollary 2.2]{cesa2006prediction}):
\begin{equation*}
   \text{PReg}(T):=\sum_{t=1}^T \Ebb_{\hat i_t}[\ell_{t,\hat i_t}\mid\Fcal_t] - \min_{i\in[K]}\sum_{t=1}^T \ell_{t,i} \leq \frac{\ln K}{\eta } + \frac{T\eta}{2}.
\end{equation*}
We will refer to the quantity on the left-hand side as the pseudo-regret $\text{PReg}(T)$. 
Using the standard choice of parameter $\eta=\sqrt{2\ln K/T}$, and assuming that the losses all have values in $[0,1]$, the previous equation directly gives an expected bound on the regret $\Ebb[\text{Reg}(T)]\lesssim \sqrt{T\ln K}$. For our purposes, we need a refinement of this bound. Using \cite[Theorem 2.1]{cesa2006prediction}, we can obtain
\begin{equation}\label{eq:base_regret_bound}
    \text{PReg}(T) \leq \frac{\ln K}{\eta } + \frac{\eta}{2} \sum_{t=1}^T \sum_{i\in[K]} p_{t,i} r_{t,i}^2,
\end{equation}
where $r_{t,i}:=\ell_{t,\hat i_t}-\ell_{t,i}$ is the instantaneous regret of the forecaster compared to expert $i$. Denote
\begin{equation*}
    \Delta_T:= \sum_{t=1}^T \sum_{i\in[K]} p_{t,i}r_{t,i}^2=\sum_{t=1}^T \Ebb_{\hat i_t}[r_{t,\hat i_t}^2 \mid\Fcal_t].
\end{equation*}

\cref{eq:base_regret_bound} yields a tighter bound than the standard regret bound if one selects $\eta\approx \sqrt{\ln K/\Delta_T}$ instead of the standard choice $\eta\approx\sqrt{\ln K/T}$. Achieving the corresponding bound without a prior knowledge of $\Delta_T$ can be easily performed via the standard doubling trick. Precisely, we use the exponential weights forecaster with initial parameter $\eta_1\approx \sqrt{2\ln K}$ until $\Delta_t\geq 1$, then restart the algorithm with a parameter $\eta_2\approx \eta_1/2$ until $\Delta_t\geq 4$. We continue the process by always restarting the algorithm with a quadrupled threshold for $\Delta$ and a corresponding parameter $\eta>0$ (roughly halved). The precise algorithm is given in \cref{alg:exponentially_weighted}, which is the exponential weights forecaster variant that we use for our algorithm \textsc{R-Cover}.
This variant enjoys the following regret bound, whose proof is given in \cref{sec:proof_learning_expert}.

\begin{algorithm}[ht]

\caption{Adaptive exponential weights forecaster \textsc{A-Exp}}\label{alg:exponentially_weighted}

\LinesNumbered
\everypar={\nl}

\hrule height\algoheightrule\kern3pt\relax
\KwIn{number of experts $K$}

\vspace{3mm}

Initialization: $k=1$, $\Delta_{max,1}=1$, $\eta_1=\sqrt{2\ln K/(\Delta_{max,1}+1)}$, $R_{0,i}=0$ for all $i\in[K]$, $\Delta_1=0$

\For{$t\geq 1$}{
    Let $p_{t,i}= e^{\eta_k R_{t-1,i}}/\sum_{j\in[K]}e^{\eta_k  R_{t-1,j}}$ and sample $\hat i_t\sim p_t$ independently from history

    Observe $\ell_{t,i}$ for $i\in[K]$, let $r_{t,i}=\ell_{t,\hat i_t}-\ell_{t,i}$ and $R_{t,i}=R_{t-1,i}+r_{t,i}$ for $i\in[K]$
    
    Update $\Delta_k \gets \Delta_k + \sum_{i\in[K]}p_{t,i}r_{t,i}^2$

    \If{$\Delta_k>\Delta_{max,k}$}{
        Set $\Delta_{max,k+1}=4\Delta_{max,k}$ and $\eta_{k+1} = \sqrt{2\ln K/(\Delta_{max,k+1}+1)}$

        Reset $R_{t,i}=0$ for all $i\in[K]$, $\Delta_{k+1}=0$, and $k\gets k+1$
    }
}

\hrule height\algoheightrule\kern3pt\relax
\end{algorithm}

\begin{lemma}\label{lemma:regret_exponentially_weighted}
Suppose that all losses lie in $[0,1]$.
Then, the pseudo-regret of the adaptive exponential weights forecaster \textsc{A-Exp} satisfies
    \begin{equation*}
    \textnormal{PReg}(T) \leq 8\sqrt{\max(\Delta_T,1) \ln K }, \quad T\geq 1.
\end{equation*}
Further, for $T\geq 1$ and $\delta\in(0,1)$, with probability at least $1-\delta$ we have
\begin{equation*}
    \textnormal{Reg}(T) \leq 12\sqrt{\max(\Delta_T,1) \ln K } + 2\ln\frac{1}{\delta}.
\end{equation*}
\end{lemma}

\section{Technical overview}
\label{sec:technical_overview}

As discussed above, the classification setting will be mostly sufficient to present our main proof ideas. Hence, in this section we mostly focus on this case.

\subsection{A simple algorithm for a weaker regret guarantee}
\label{subsec:simpler_algo}

To motivate the form of \textsc{R-Cover}, we first consider a significantly simpler algorithm which essentially corresponds to \textsc{R-Cover} with depth $1$.
In this simplest form, \textsc{R-Cover} subdivides the horizon $[T]$ into $K$ equal-length epochs and a learning with expert advice algorithm on each epoch on the projection of the function class $\Fcal$ on query points $x_t$ from prior epochs. For instance, we can use the classical exponential weights forecaster algorithm (e.g. see \cite{cesa2006prediction}). This simplified algorithm which we call \textsc{Cover} is summarized in \cref{alg:simplified_algo}.

\begin{algorithm}[ht]

\caption{Construction of the \textsc{Cover} algorithm} \label{alg:simplified_algo}
\LinesNumbered
\everypar={\nl}

\hrule height\algoheightrule\kern3pt\relax
\KwIn{horizon $T$, number of epochs $K\leq T$}

\vspace{3mm}

Let $T_k = \floor{k\frac{T}{K}}$ for $k\in\{0,\ldots,K\}$.

\For{$k\in[K]$}{
    Construct a minimal-size cover $S_k\subset \Fcal$ such that for any $f\in\Fcal$ there exists $g\in S_k$ with $f(x_s)=g(x_s)$ for $s\in[T_{k-1}]$

    For iterations $t\in(T_{k-1},T_k]$, run any learning with expert advice algorithm (e.g. Hedge) with expert set $S_k$
}

\hrule height\algoheightrule\kern3pt\relax
\end{algorithm}

We can show that with a convenient choice of the number of epochs $K\approx T^{1/3}$, \textsc{Cover} already achieves a $\approx T^{2/3}$ regret guarantee without any prior knowledge on the distribution $\mu$. Given the simplicity of \textsc{Cover}, this result may be of independent interest.

\begin{theorem}\label{thm:regret_simple_algo}
    Fix $T\geq 1$. Let $\Fcal:\Xcal\to\{0, 1\}$ be a function class with VC dimension $d$. Suppose that $(x_t)_{t\geq 1}$ is a $\sigma$-smooth sequence on $\Xcal$ with respect to some unknown base measure $\mu$. Then, $\textsc{Cover}$ run with parameter $K=\lfloor\ln  T \cdot (T/d)^{1/3} \sigma^{-2/3} \rfloor$ makes predictions $\hat y_t$ such that
    \begin{equation*}
        \Ebb\sqb{\sum_{t=1}^T \ell_t(\hat y_t) - \inf_{f\in\Fcal} \sum_{t=1}^T \ell_t(f(x_t)) } \leq  C 
        \ln^2 T  \paren{\frac{d T^2}{\sigma}}^{1/3}.
    \end{equation*}
    for some universal constant $C>0$.
\end{theorem}

We formally prove this result in \cref{sec:simpler_proofs}. In this section, our goal is mostly to give key intuitions about the underlying strategy for the full algorithm \textsc{R-Cover}. To give some insights into why \textsc{Cover} already achieves sublinear regret, note that if the queries prior to some epoch $(T_{k-1},T_k]$ are ``representative'' of the queries during this epoch, then the cover $S_k$ constructed at the beginning of the epoch (line 3 of \cref{alg:simplified_algo}) is a good representative set of relevant functions. Naturally, this holds if the underlying process $(x_t)_{t\in[T]}$ is i.i.d.---that is $\sigma=0$. The crux of our analysis is to show that when the adversary is $\sigma$-smooth this still holds in an amortized sense:
\comment{Note that it is not true that the queries $(x_t)_{t\leq T_{k-1}}$ observed prior to some epoch $(T_{k-1},T_k]$ are always representative of the queries during that epoch. Indeed, a $\sigma$-smooth adversary can for instance decide to have the sequence of distributions $(\mu_t)_{t\in[T]}$ adaptively switch from one distribution to a completely unrelated one up to $\floor{1/\sigma}$ times. However, we show that the number of epochs for which prior queries $(x_t)_{t\leq T_{k-1}}$ are not representative of the queries on the epoch $(T_{k-1},T_k]$ is bounded.}
we show that the number of epochs for which prior queries $(x_t)_{t\leq T_{k-1}}$ are not representative of the queries on the epoch $(T_{k-1},T_k]$ is bounded.

To quantify the notion of ``representativeness'', we introduce the following quantity, which is essentially the maximum $\ell_1$ discrepancy between queries observed until some time $t_0 <t$ and the query made at time $t$ on the function class $\Fcal$. For any $0\leq t_0 <t\leq T$, we define
\begin{equation}\label{eq:def_gamma_t_t_0}
    \gamma_{t_0}(t) := \sup_{\substack{f,g\in\Fcal \text{ s.t.} \\ f(x_s)=g(x_s),\, s\in[t_0]}} \Pbb\paren{f(x_t) \neq g(x_t) \mid \Hcal_{t-1}}  =\sup_{\substack{f,g\in\Fcal \text{ s.t.} \\ f(x_s)=g(x_s),\, s\in[t_0]}} \Pbb_{x\sim\mu_t} \paren{f(x) \neq g(x)}, 
\end{equation}
where $\Hcal_{t-1}$ denotes all history available until the end of iteration $t-1$.
One of our main contributions for the analysis of smoothed adversaries is the following result which bounds the number of epochs on which prior history is not representative. 

\comment{
Intuitively, if the queries prior to $t_0$ were representative of the query at time $t$, then the empirical projection of $\Fcal$ onto the query set $(x_t)_{t\leq t_0}$ should reasonably cover $x_t\sim\mu_t$ and as a result $\gamma_{t_0}(t)$ would be smaller.
}

\begin{proposition}\label{prop:simplified}
    Let $T\geq 2$ and $\Fcal:\Xcal\to\{0, 1\}$ be a function class with VC dimension $d$. Let $0=T_0<T_1<\ldots <T_K=T$ define epochs.
    Fix any parameters $q,\delta\in (0,1]$ and denote $w(T,\delta):=d\ln\paren{\frac{T}{\sigma}\ln\frac{1}{\delta}} + \ln\frac{T}{\delta}+2$. Then, with probability at least $1-\delta$,
    \begin{equation*}
        \abs{ \set{k\in[K]: \sum_{t=T_{k-1}+1}^{T_k} \gamma_{T_{k-1}}(t) \cdot \1[\gamma_{T_{k-1}}(t) \geq q] \geq w(T,\delta)  }} \leq C \frac{\ln^2 T}{q\sigma},
    \end{equation*}
    for some universal constant $C\geq 1$. For a bound in expectation we can take $w(T):=d\ln\frac{T}{\sigma} +2$.
\end{proposition}
The proof of \cref{prop:simplified} uses some key results from \cite{block2024performance}. At the high-level, we focus on a sub-sequence of $(x_t)_{t\in[T]}$ by only keeping times which have a significant contribution $\gamma_{T_{k-1}}(t)\geq q$. Following arguments from \cite{block2024performance}, we next apply a bound (\cref{lemma:decoupling_stronger}) to decouple the total sum of the terms $\gamma_{T_{k-1}}(t)$ via a tangent sequence to $(x_t)_{t\in[T]}$ which is then bounded by the Wills functional of the function class $\Fcal$ (\cref{lemma:block_whp}). We defer a detailed sketch of proof to \cref{sec:simpler_proofs}.

\cref{prop:simplified} shows that apart from $\tilde \Ocal(1/(q\sigma))$ epochs, we pay at most a price $w(T,\delta)$ during each epoch $(T_{k-1},T_k]$ for the times $t\in (T_{k-1},T_k]$ when the cover constructed from queries prior to this epoch was not representative of query $x_t$ by some threshold $q$. Here, we view $w(T,\delta)$ as a reasonable price to pay on each epoch. Hence, intuitively, apart from $\tilde \Ocal(1/(q\sigma))$ epochs, the cover constructed from queries on prior epochs is always representative up to threshold $q$.

Up to the logarithmic factors, \cref{prop:simplified} is tight in the following sense. For any choice of the online mechanism to construct epochs and threshold $q$, a $\sigma$-smooth adversary can ensure that for $\Ocal(1/(q\sigma))$ epochs $(T_{k-1},T_k]$, queries on prior epochs are not representative up to threshold $q$ from all times in $(T_{k-1},T_k]$. We detail below the scenarios for which \cref{prop:simplified} is tight. 
\comment{We believe that these essentially captures all possible attack behaviors of a smooth adversary.}

\comment{

\begin{proposition}\label{prop:tightness_of_adversary_construction}
    There exists a function class $\Fcal:\Xcal\to\{0, 1\}$ with VC dimension $1$ such that the following holds. For any $T\geq 1$, $\sigma,q\in(0,1]$ and any online mechanism to construct epochs $(T_{k-1},T_k]$ for $k\in[K]$ as defined in \cref{prop:simplified}. Then,
    \begin{equation*}
        \abs{\set{k\in[K]: \forall t\in(T_{k-1},T_k], \gamma_{T_{k-1}}(t) \geq q  }} \geq \min\paren{ \frac{c}{q\sigma} , K},\quad (a.s.),
    \end{equation*}
    for some universal constant $c>0$.
\end{proposition}

While this is not quite needed to prove our main results \cref{thm:main_thm,thm:main_regret_regression}, the proof of \cref{prop:tightness_of_adversary_construction} is quite instructive hence we give below the main ideas. We believe that it essentially captures all possible attack behaviors of a smooth adversary.

\paragraph{Sketch of proof.} 

}

Because of the $\sigma$-smoothness constraint, the adversary cannot query the algorithm on completely different regions of the space $\Xcal$ at each epoch. One possible strategy for the adversary
\comment{, which we discussed as motivation above, 
}is to switch distributions $\floor{1/\sigma}$ times during the learning process, possibly onto a completely new region of the space. This corresponds to $q=1$ in \cref{prop:simplified}: at the start of $\floor{1/\sigma}$ epochs $(T_{k-1},T_k]$, the adversary switches query distributions $\mu_t$ and selects a distribution with support on a new region for which prior queries are irrelevant. This results in $\gamma_{T_{k-1}}(t)=1$ for all $t\in(T_{k-1},T_k]$. 

A more refined strategy for the adversary is to select a parameter $q$ and at the start of a new epoch $(T_{k-1},T_k]$, switch the query distribution as follows. They construct a new mixture distribution $\mu_k:= q\nu_k + (1-q)\mu_0$ where with probability $q$ the learner is queried on a new distribution $\nu_k$ with say completely new support compared to the history, and with probability $1-q$ the learner is queried on a base measure $\mu_0$ that is very similar to previous queries. This results in $\gamma_{T_{k-1}}(t)\geq q$ for all $t\in(T_{k-1},T_k]$. On one hand, during the epoch, the adversary could only test the learner on a fraction $q$ of ``truly adversarial'' queries sampled from $\mu_k$. On the other hand, the smoothness constraint is now easier to satisfy and we can check that the adversary can afford to corrupt $\approx 1/(q\sigma)$ epochs. This precisely corresponds to the bound from \cref{prop:simplified} up to logarithmic factors. As it turns out, this mixture strategy is stronger for the adversary and choosing $q\approx 1/\sqrt T$ is the strategy that yields the lower bound from \cref{thm:lower_bound}. 

\comment{
In all cases, these strategies can be easily instantiated with the simple class of thresholds on $[0,1]$ that is $\Fcal=\{x\in[0,1]\mapsto \1[x\geq x_0]: x_0\in[0,1]\}$, which has VC dimension $1$. This is what we use in the proof of \cref{prop:tightness_of_adversary_construction} for simplicity.
}

\begin{remark} 
The statement from \cref{prop:simplified} is written specifically for classification, for which analyzing the $\ell_1$ diameter as defined in $\gamma_{t_0}(t)$ in \cref{eq:def_gamma_t_t_0} is amenable. The proof of \cref{prop:simplified} requires controlling the complexity of the class $\{\1[f\neq g]:f,g\in\Fcal\}$ which has VC dimension bounded by $2d$ if $\Fcal$ has VC dimension $d$. While the VC dimension behaves nicely with this self-difference operation, this is not the case for the fat-shattering dimension which is known to behave somewhat wildly with the addition \cite{asor2014additive}.\footnote{\cite{asor2014additive} notes that the function class $\Fcal$ of increasing functions on $[0,1]$ always has fat-shattering dimension one at any scale, while $\Fcal-\Fcal=\{f-g:f,g\in\Fcal\}$ has infinite shattering dimension at all scales.} For the regression setting, we need to localize this difference class around an oblivious benchmark function $f^\star$. The localized analog of $\gamma_{t_0}(t)$ that we use in our proofs is defined in \cref{eq:definition_gama}. The corresponding generalization of \cref{prop:simplified} is \cref{lemma:main_bound_modified}. In this general regression setting, the term in $w(T,\delta)$ from \cref{prop:simplified} depending on the VC dimension $d$ is replaced by the Wills functional of $\Fcal$.
\end{remark}

With the main tool \cref{prop:simplified} at hand, we can easily prove a simpler version of \cref{thm:regret_simple_algo} for the expected oblivious regret. Fix some benchmark function $f^\star\in\Fcal$. On each epoch $k\in[K]$, \textsc{Cover} runs a learning with expert advice algorithm on the cover $S_k$, which has size $\Ocal(T^d)$ by Sauer-Shelah's \cref{lemma:sauer_lemma}. Hence, using classical regret bounds (e.g. \cite[Corollary 2.2]{cesa2006prediction}), the total expected regret incurred by these algorithms is bounded by
\begin{equation}\label{eq:first_regret_term}
    C\sum_{k\in[K]} \sqrt{ (T_k-T_{k-1}) \cdot d\ln T} \lesssim \sqrt{KdT\ln T},
\end{equation}
for some constant $C\geq 1$, where we used Jensen's inequality in the last inequality. Next, for each epoch $k\in[K]$, denote by $f_k\in S_k$ the function in the cover that had the correct labeling for $f^\star$, that is, $f_k(x_t) = f^\star(x_t)$ for all $t\in[T_{k-1}]$.
Because $f_k$ is one of the experts considered during epoch $k$, it suffices to bound the remaining term
\comment{
\begin{equation*}
    \sum_{k\in[K]} \sum_{t=T_{k-1}+1}^{T_k} \ell_t(f_k(x_t)) - \ell_t(f^\star(x_t)) \leq \sum_{k\in[K]} \sum_{t=T_{k-1}+1}^{T_k} \1[f_k(x_t) \neq f^\star(x_t)].
\end{equation*}
Taking the expectation of each term for $x_t$ conditionally on the history $\Hcal_{t-1}$, we obtain
}
\begin{equation*}
    \Ebb\sqb{\sum_{k\in[K]} \sum_{t=T_{k-1}+1}^{T_k} \ell_t(f_k(x_t)) - \ell_t(f^\star(x_t))} \leq \Ebb\sqb{\sum_{k\in[K]} \sum_{t=T_{k-1}+1}^{T_k}  \gamma_{T_{k-1}}(t)}, 
\end{equation*}
since the functions $f_k$ and $f^\star$ agreed on all queries of previous epochs. We can then use \cref{prop:simplified} which bounds the sum for each epoch $k\in[K]$. Applying \cref{prop:simplified} for $q\geq q_0 :=\ln^2 (T)/(K\sigma)$ bounds the number of epochs for which this sum deviates significantly. At the high level, it implies that the quantities $\gamma_{T_{k-1}}(t)$ are roughly of order $q_0$ in average. Precisely, letting $\Delta T:=\max_{k\in[K]} T_k-T_{k-1}=\Ocal(T/k)$ and $l_0:=\floor{\log_2(1/q_0)} = \Ocal(\log T)$, \cref{prop:simplified} implies that with probability at least $1-(l_0+1)\delta$,
\begin{align}\label{eq:second_regret_term_expanded}
    \sum_{k\in[K]} \sum_{t=T_{k-1}+1}^{T_k}  \gamma_{T_{k-1}}(t)
    &\leq q_0 T + \sum_{l=0}^{l_0} \sum_{k\in[K]} \overbrace{\sum_{t=T_{k-1}+1}^{T_k}  \gamma_{T_{k-1}}(t) \cdot \1[\gamma_{T_{k-1}}(t)\in[2^l q_0, 2^{l+1}q_0]]}^{\Gamma_{k,l}} \notag\\
    &\leq q_0 T + (l_0+1)K\cdot w(T,\delta)  + \sum_{l=0}^{l_0}  2^{l+1}q_0 \Delta T\cdot  \abs{\set{k\in[K]: \Gamma_{k,l} \geq w(T,\delta) }}\notag \\
    &\leq q_0 T + (l_0+1)K\cdot w(T,\delta)  + \sum_{l=0}^{l_0}  2^{l+1}q_0 \Delta T \cdot C\frac{\ln^2 T}{2^l q_0 \sigma} \lesssim \frac{\ln^3 T}{K\sigma}\cdot T.
\end{align}
Putting the two regret terms from \cref{eq:first_regret_term,eq:second_regret_term_expanded} together and optimizing over the choice of $K$ gives the same bound as \cref{thm:regret_simple_algo} for the expected oblivious regret of \textsc{Cover}. We explain how this oblivious regret guarantee can be turned into an adaptive regret guarantee in \cref{subsec:proof_sketch_main_thm}.

\subsection{Achieving the optimal regret using recursive covers}
\label{subsec:intuitions_r_cover}

The main obstacle for \textsc{Cover} for achieving the optimal regret dependency $\sqrt T$ in the horizon is that it needs to balance between two competing regret terms: (1) the regret incurred by learning with expert algorithms, which usually increases with the number of epochs; and (2) the discretization error obtained by approximating the optimal function using a net constructed on prior epochs, which decreases with the number of epochs. 

We use a localization strategy to increase the number of effective epochs on which a cover is recomputed. To not incur a large regret term due to the learning with expert algorithms, we introduce the adaptive variant from the classical \emph{Hedge} algorithm, \textsc{A-Exp}, which has a regret bound depending on the actual difficulty of the learning with expert problem instead of a worst-case bound (see \cref{subsec:learning_with_experts}). Going back to an epoch $(T_{k-1},T_k]$ of \textsc{Cover}, \cref{prop:simplified} essentially implies that during most epochs $k\in[K]$ one can bound
\begin{equation}\label{eq:optimistic_bound}
    \sum_{t=T_{k-1}}^{T_k} \gamma_{T_{k-1}}(t) \lesssim q_0(T_{k-1}-T_k)
\end{equation}
where $q_0=\ln^2 (T)/(K\sigma)$. As a result, if we restrict our search space on epoch $k$ to some functions that shared the same values on previous epoch queries $(x_t)_{t\leq T_{k-1}}$, we expect that these would only disagree (have different predictions) for a fraction $\approx q_0$ of the times in epoch $k$. Using the regret guarantee from \textsc{A-Exp} from \cref{lemma:regret_exponentially_weighted} we can then show that on epochs $k\in[K]$ for which \cref{eq:optimistic_bound} holds, performing \textsc{A-Exp} on a set $S\in\Fcal$ of functions that agreed on previous epochs incurs a learning with expert regret
\begin{equation*}
    \sum_{t=T_{k-1}+1}^{T_k} \ell_t(\hat y_t)-\min_{f\in S}\sum_{t=T_{k-1}+1}^{T_k} \ell_t(f(x_t)) \lesssim \sqrt{q_0(T_{k-1}-T_k)\ln|S|},
\end{equation*}
with reasonable probability $1-\delta$, instead of the worst-case bound $\sqrt{(T_{k-1}-T_k)\ln |S|}$ for the Hedge algorithm. Here, $\hat y_t$ denotes the predictions of the learning with expert advice algorithm, and we omitted lower-order terms which may depend on the probability failure $\delta$.

This regret improvement for the regret of \textsc{A-Exp} leads us to the following depth-$2$ algorithm: on each epoch $(T_{k-1},T_k]$ we can run any learning with expert advice algorithm (say Hedge) using as experts the predictions of all \textsc{Cover} algorithms that are run with horizon $T_k-T_{k-1}$, use a fixed number of epochs, use \textsc{A-Exp} as expert advice algorithm (line 4 of \cref{alg:simplified_algo}) and restrict their search space to functions in $\Fcal$ that agreed on previous epoch queries $(x_t)_{t\leq T_{k-1}}$. By Sauer-Shelah's \cref{lemma:sauer_lemma}, there are at most $2T^d$ such experts. Optimizing the choice of number of epochs for each of the two layers yields an improved dependency in $T^\alpha$ for the final regret bound compared to the $1$-depth \textsc{Cover} algorithm in \cref{thm:regret_simple_algo}, for some $\alpha\in(1/2,2/3)$. 

To achieve the optimal regret, we run this strategy recursively over $\floor{\log_2(T)/2}$ depths, which is \textsc{R-Cover}. This strategy is akin to some form of chaining at the algorithmic level. The smallest sub-epochs on the last layer have length of order $\sqrt T$. Note that the labeled dataset $S$ that is used as parameter in the recursive construction of \textsc{R-Cover} in \cref{alg:recursive_construction} now corresponds to the possible labelings of queries in prior epochs. In practice, the optimal depth to achieve the best regret regret bound depends depends on the smoothness parameter $\sigma$. To avoid requiring this information when implementing \textsc{R-Cover}, at each depth, in addition to the experts corresponding to the predictions of the next layer algorithm, we also add an expert that uses a single function as prediction ($f_S$ in lines 6 and 10 of \cref{alg:recursive_construction}). This hedges the final algorithm for all choices of depths at once. Within the proof, we may then focus on the algorithm up to a fixed $\sigma$-dependent depth.

\subsection{Proof sketch for \cref{thm:main_thm}}
\label{subsec:proof_sketch_main_thm}

\comment{
Now that we have introduced the main conceptual ingredients of the proof, we give a brief sketch of the regret bound. 
}
We start by focusing on the oblivious regret compared to some fixed benchmark function $f^\star$. \textsc{R-Cover} is composed of $P$ layers. Each layer $p\in[P]$ corresponds to epochs $(T_{k-1}^{(p)},T_k^{(p)}]$ for $k\in[N_p]$. For instance, initially there is a single epoch for $p=P$ and at the last layer $p=0$ there are $\approx \sqrt T$ epochs. For each depth $p\in[P]$ and epoch $k\in[N_p]$ we consider the depth-$p$ \textsc{R-Cover} algorithm that was instantiated with the ``correct'' labeling according to $f^\star$. That is, we focus on the algorithm that used the dataset
$S_k^{(p)} = \set{(x_t,f^\star(x_t)),t\in[T_{k-1}^{(p)}]}$.

\paragraph{Regret decomposition.} The point of the recursive procedure is that it allows to localize the error by focusing only on the runs of \textsc{R-Cover} that used these correct labeled datasets. The first step of the proof (\cref{subsec:regret_decomposition}) is to show that we can decompose the regret of the algorithm compared to $f^\star$ in the following way, where $\hat y_t$ denotes the predictions of the final algorithm. With probability $1-\delta$,
\begin{multline}\label{eq:regret_decomposition_simplified}
    \sum_{t=1}^T \ell_t(\hat y_t) - \ell_t(f^\star(x_t)) \\
    \lesssim \sum_{p=p_0}^{P} \sum_{k\in[N_p]} \underbrace{\sqrt{\max\paren{\Delta_k^{(p)},2}  d\ln T}}_{Reg_k^{(p)}} 
    + \sum_{k\in[N_{p_0}]} \underbrace{\sum_{t=T_{k-1}^{(p_0)}}^{T_k^{(p_0)}} \ell_t\paren{f_{k,S}^{(p_0)}(x_t)} - \ell_t(f^\star(x_t)) }_{\Lambda_k^{(p_0)}} + N_{p_0} \ln\frac{1}{\delta}.
\end{multline}
The first term of \cref{eq:regret_decomposition_simplified} is the regret accumulated along the localization trajectory for running the learning with expert advice algorithm \textsc{A-Exp}. Up to minor details, here $Reg_k^{(p)}$ corresponds to the bound on the regret incurred by \textsc{A-Exp} for the depth-$p$ algorithm run during epoch $k\in[N_p]$, which is guaranteed by \cref{lemma:regret_exponentially_weighted}. The quantity $\Delta_k^{(p)}$ is the same as that which appears in \cref{lemma:regret_exponentially_weighted} and measures the difficulty of the learning with expert problem on epoch $k$ at depth $p$. Technically, the bound from \cref{lemma:regret_exponentially_weighted} also includes a failure probability term which accumulated over the complete trajectory corresponds to the term $N_{p_0}\ln\frac{1}{\delta}$. This can be viewed as a lower order term.

The second term of \cref{eq:regret_decomposition_simplified} intuitively corresponds to the excess error of a learner that at the beginning of each depth-$p_0$ epoch $k\in[N_{p_0}]$ has access to an oracle which reveals the values of the optimal function $f^\star$ on all prior epoch queries $(x_t)_{t\leq T_{k-1}^{(p_0)}}$. Here, we use the notation $f_{k,S}^{(p_0)}$ to denote the base function $f_S$ that was constructed at line 1 of \cref{alg:recursive_construction} during the run of the depth-$p_0$ algorithm \textsc{R-Cover} on epoch $k$ using the correct labeling $S_k^{(p_0)}$. The quantity $\Lambda_k^{(p_0)}$ corresponds to the excess error of this base function $f_{k,S}^{(p_0)}$ compared to $f^\star$ on the epoch $k$.

\paragraph{Bounding each regret term via smoothness.} We next bound each regret term and more precisely bound the terms $\Delta_k^{(p)}$ and $\Lambda_k^{(p)}$. Using the same arguments as described in \cref{subsec:simpler_algo} when bounding the excess error of \textsc{Cover} on each epoch, we can show that the quantity
\begin{equation*}
    \Gamma_k^{(p)} := \sum_{t\in (T_{k-1}^{(p)},T_k^{(p)}]}\gamma_{T_{k-1}^{(p)}}(t)
\end{equation*}
bounds both $\Delta_k^{(p)}$ and $\Lambda_k^{(p)}$ up to constant factors. Within the general proof for regression, these need to be bounded by different quantities $\Gamma_k^{(p,2)}$ and $\Gamma_k^{(p,1)}$ respectively in which $\gamma(t)$ is replaced by the $\ell_2$ and $\ell_1$ deviations from $f^\star$ (see \cref{lemma:bounding_deltas} for the precise bound). 
We then bound the terms $\Gamma_k^{(p)}$ using \cref{prop:simplified}, which is the only step so far that required the smoothness assumption. For regression, the generalization of this result is \cref{lemma:main_bound_modified}. In this full form, the guarantee obtained on $\Gamma_k^{(p,1)}$ and $\Gamma_k^{(p,2)}$ depends on the scale $\epsilon$ of the cover constructed at the beginning of each epoch (in classification we used $\epsilon=0$ in \cref{alg:recursive_construction}). Naturally, the guarantee degrades as $\epsilon$ grows.
\comment{In classification, because values lie in $\{0,1\}$ both norms are identical hence we can use a single quantity $\Gamma_k^{(p)}$.}

Putting everything together gives a high-probability bound on the regret of the algorithm compared to $f^\star$ of the same order as the desired adaptive bound from \cref{thm:main_thm} (see \cref{prop:oblivious_vc_class}). The corresponding bound in the general regression case is \cref{thm:oblivious_adversary}.

\paragraph{From oblivious to adaptive regret guarantees}

The last step of the proof of \cref{thm:main_thm} is to obtain adaptive regret bounds from high-probability oblivious regret bounds. This uses tools from prior works on smoothed online learning, and in particular \cite{haghtalab2024smoothed}. 

We first construct a cover of the function class $\Fcal$ with respect to the base measure $\mu$ and aim to have low regret compared to functions in this cover. Precisely, we construct a subset of $\Fcal$ such that for all $f\in\Fcal$ there exists $h\in\Hcal$ with
$\Pbb_{x\sim\mu}(f(x)\neq h(x))\leq \epsilon$.
Since $\Fcal$ has VC dimension $d$, we can ensure $\ln |\Hcal| \leq 2d\ln(e^2/\epsilon)$ (see \cite{haussler1995sphere} or \cite[Lemma 13.6]{boucheron2013concentration}). Using the union bound over the high-probability oblivious regret bounds from the previous section, we can ensure that \textsc{R-Cover} has low regret compared to all functions within the cover. We note that contrary to \cite{haghtalab2024smoothed}, this covering construction is only for proof purposes and is not performed within the algorithm \textsc{R-Cover}. In fact, since $\mu$ is unknown in our setting, computing such a cover is impossible, as exemplified by the lower bound \cref{thm:lower_bound}. 

It then only remains to show that $\Hcal$ is also a good cover on the queries made by the smooth adversary $(x_t)_{t\in[T]}$. This would be immediate if these queries are i.i.d.\ sampled from $\mu$ by VC uniform convergence. To reduce to the i.i.d.\ case, \cite{haghtalab2024smoothed} show the following coupling lemma.

\begin{lemma}[\cite{haghtalab2024smoothed,block2022smoothed}]
\label{lemma:coupling}
    Let $(X_t)_{t\in[T]}$ be $\sigma$-smooth with respect to $\mu$. Then for all $k\geq 1$, there is a coupling of $(X_t)_{t\in[T]}$ with random variables $\{Z_{t,j},t\in[T],j\in[k]\}$ such that the $Z_{t,j}\overset{iid}{\sim}\mu$ and on an event $\Ecal_k$ of probability at least $1-Te^{-\sigma k}$, we have $X_t\in\{Z_{t,j},j\in[k]\}$ for all $t\in[T]$. 
\end{lemma}

In particular, it suffices for the cover to perform well on all queries $Z_{t,j}$ for $t\in[T]$ and $j\in[k]$ for some $k\approx \sigma^{-1}\ln T$, which are i.i.d. This gives the desired adaptive regret bound by using VC uniform convergence bounds on the i.i.d.\ variables $Z_{t,j}$.

\section{Regret analysis}
\label{sec:regret_bound}

In this section, we first describe our full algorithm \textsc{R-Cover} for regression in \cref{subsec:description_algo_regression} then prove our main results \cref{thm:main_thm,thm:main_regret_regression} in the rest of the section.

\comment{
The first step of the proof is to derive high-probability regret bounds in the simpler case of \emph{oblivious} benchmark functions. Precisely, we first fix a function $f^\star\in\Fcal$ and aim to bound the regret of \textsc{R-Cover} compared to $f$ (see \cref{thm:oblivious_adversary}). To do so, in \cref{subsec:regret_decomposition} we first decompose the total regret along each epoch. At the high level, we only incur the regret of algorithms $\textsc{R-Cover}_{T_0,T_1}^{(p)}(S)$ that used the ``correct'' labeling $S$ corresponding to the benchmark prediction function $f^\star$, that is $S=\{(x_t,f^\star(x_t),t\leq T_0\}$. We next bound each of these regret terms in  \cref{subsec:bounding_each_term} and conclude the oblivious regret analysis, deferring the proof of a main lemma in \cref{subsec:proof_main_lemma}. As a remark, the initial regret decomposition does not require smoothness of the data, this assumption is only used in \cref{subsec:proof_main_lemma}.

Last, in \cref{subsec:oblivious_to_adaptive} we use the high-probability oblivious regret bounds to obtain regret bounds for adaptive benchmark functions in the classification case. That is, we bound the regret compared to the best function $f^\star\in\Fcal$ in hindsight (depending on the learning sequence $(x_t,y_t)_{t\in[T]}$. This uses relatively standard arguments from the literature on smooth adversaries.
}

\subsection{General recursive procedure for regression}
\label{subsec:description_algo_regression}

In this section, we generalize the algorithm given for classification in \cref{alg:recursive_construction} to handle general regression function classes $\Fcal$. Note that at the beginning of each new epoch, \textsc{R-Cover} effectively computes a $0$-cover of the previously observed dataset. In the regression setting, we instead compute an $\epsilon$-cover for some $\epsilon>0$ of the functions within the class $\Fcal$ centered around a reference function $f_0\in\Fcal$. This effectively restricts the search space of the algorithm to the neighborhood of $f_0$, and replaces the labeled dataset $S$ from \cref{alg:recursive_construction} in the classification case (these will be equivalent in this case). Instead of using a single reference function $f_0$, we use a sequence of reference functions which will correspond to reference functions from previous depths. This is used to ensure that the search space within $\Fcal$ of sub-algorithms (akin to lines 6 and 10 of \cref{alg:recursive_construction}) are consistent with the search space of algorithm calls from previous depths. These reference functions $f_i:\Xcal\to[0,1]$ are stored together with the start time of their corresponding algorithm call $t_i\in[T]$, within a set $S$.

To summarize, the recursive algorithms uses as parameters a start time $T_0$, an end time $T_1$, the depth $P$, a finite set of reference functions $S = \set{ (f_i,t_i),i}$, as well as the scale parameter $\epsilon$. By abuse of notation, we still refer to the corresponding algorithm as $\textsc{R-Cover}_{T_0,T_1}^{(P,\epsilon)}(S)$.
The algorithm only aims to achieve low regret compared to functions within $\Fcal$ that had similar predictions to the reference functions within $S$ on the history. Precisely, for any $f_0:\Xcal\to[0,1]$ and $0\leq T_0\leq T$, we define the set
\begin{equation}\label{eq:def_B_f_0_ball}
    B_{f_0}(\Fcal;\epsilon,T_0) := \set{f\in\Fcal: \max_{t\in[T_0]} |f(x_t)-f_0(x_t)| \leq \epsilon }.
\end{equation}

For the base depth $P=0$, the algorithm simply follows the predictions of any function within 
\begin{equation}\label{eq:def_Fcal_S}
    \Fcal(S):= \bigcap_{(f_0,t_0)\in S} B_{f_0}(\Fcal;\epsilon,t_0),
\end{equation}
which corresponds to the search space of the algorithm. For $P\geq 1$, the algorithm defines two sub-epochs using $T_{1/2}$ exactly as in \cref{alg:recursive_construction}. At the beginning of each epoch at time $T_\alpha$ for $\alpha\in\{0,1/2\}$, the algorithm constructs a minimum covering $\epsilon$-cover of the search space on the previously queried points $(x_t)_{t\in[T_\alpha]}$ as defined in \cref{def:covering_numbers}. That is, we construct a set $\Ccal \subset\Fcal(S)$ such that for all $f\in  \Fcal(S)$ there exists $g\in \Ccal$ such that
\begin{equation*}
    \max_{t\in[T_\alpha]} |f(x_t)-g(x_t)| \leq \epsilon,
\end{equation*}
and that has minimal cardinality. The algorithm then perform the learning with expert advice algorithm \textsc{A-Exp} using the expert predictions from $\textsc{R-Cover}_{T_\alpha,T_{\alpha+1/2}}^{(P-1,\epsilon)}(S\cup\{(f,T_\alpha)\})$ for all $f\in \Ccal$, as well an expert corresponding to any fixed function $ f_S\in\Fcal(S)$. The recursive algorithm is summarized in \cref{alg:recursive_construction_regression}.\footnote{Compared to the description of \cref{alg:recursive_construction} we merged the description of the algorithm on each epoch $(T_0,T_{1/2}]$ and $(T_{1/2},T_1]$ for conciseness.}

\begin{algorithm}[ht]

\caption{Recursive construction of $\textsc{R-Cover}_{T_0,T_1}^{(P,\epsilon)}(S)$}\label{alg:recursive_construction_regression}

\LinesNumbered
\everypar={\nl}

\hrule height\algoheightrule\kern3pt\relax
\KwIn{depth $P\geq 0$, start and end times $T_0\leq T_1$ satisfying $T_1-T_0\geq 2^P$, set of reference functions $S$, scale $\epsilon$}

\vspace{3mm}

\eIf{$P=0$}{
     Fix any $f_S\in \Fcal(S)$ (see \cref{eq:def_Fcal_S}) and predict $\hat y_t = f_S (x_t)$ for all $t\in(T_0,T_1]$
}{
    Fix any $f_S\in \Fcal(S)$ and let $T_{1/2}:=\floor{\frac{T_0+T_1}{2}}$

    \For{$\alpha\in\{0,1/2\}$ (epoch $(T_\alpha,T_{\alpha+1/2}])$}{
        After iteration $T_\alpha$, construct a $\epsilon$-cover $\Ccal$ of $\Fcal(S)$ on the queries $(x_t)_{t\in[T_\alpha]}$

        Perform $\textsc{A-Exp}$ (see \cref{alg:exponentially_weighted}) on $(T_\alpha,T_{\alpha+1/2}]$ with experts $\set{\textsc{R-Cover}_{T_\alpha,T_{\alpha+1/2}}^{(P-1,\epsilon)}\paren{S\cup\{(f,T_\alpha)\} },\; f\in\Ccal } \cup\{ f_S \}$
    }
    
}

\hrule height\algoheightrule\kern3pt\relax
\end{algorithm}

Similar to the classification case, we use the depth $P=\floor{\log_2(T)}$ and run $\textsc{R-Cover}_{0,T}^{(P)}(\emptyset)$ as our final algorithm. Note that the search space for the final algorithm is the complete function class $\Fcal$.
We can also still give a bound on the number of experts considered at each step. It is at most $\Ncal(\Fcal(S),\epsilon,T)+1$ where $\Ncal(\Fcal(S),\epsilon,T)$ denotes the size of the minimal $\epsilon$-covering of $\Fcal(S)$ on the queries $(x_t)_{t\in[T]}$. Using \cref{thm:fat_shattering_bound} this can be further bounded as follows,
\begin{equation}\label{eq:bound_covering_numbers_abuse}
    \ln \Ncal(\Fcal(S);\epsilon,T) \lesssim \mathsf{fat}_\Fcal(c\alpha\epsilon) \ln^{1+\alpha}\paren{\frac{CT}{\epsilon}}.
\end{equation}
In the last inequality, we used the fact that $\mathsf{fat}_{\Fcal(S)}(r) \leq \mathsf{fat}_\Fcal(r)$ for all $r\geq 0$ since $\Fcal(S)\subset \Fcal$.
Given that the covering numbers of all function classes $\Fcal(S)$ are upper bounded by this quantity, in the rest of the paper, we may safely omit the dependency in $S$ to lighten the notations.

\subsection{Regret decomposition}
\label{subsec:regret_decomposition}

Before decomposing the regret of the final algorithm, we define a few notations. Fix a depth $p\in \{0,\ldots,P\}$. Note that the final algorithm $\textsc{R-Cover}_{0,T}^{(P,\epsilon)}(\emptyset)$ calls depth-$p$ algorithms $\textsc{R-Cover}^{(p,\epsilon)}_{T_0,T_1}$ on fixed depth-$p$ epochs $(T_0,T_1]$. Precisely, there are $N_p:=2^{P-p}$ such depth-$p$ epochs and we define $T_0^{(p)}=0 < T_1^{(p)} <\ldots < T_{N_p}^{(p)}=T$ the start and end times of these epochs. We then use the notation $E_k^{(p)}:=(T_{k-1}^{(p)},T_k^{(p)}]$ for the $k$-th depth-$p$ epoch. For instance, by construction one has $T_{N_p/2}^{(p)} = \floor{T/2}$ for all $p<P$(see line 4 of \cref{alg:recursive_construction_regression}). More generally, these epochs all have roughly the same length. In fact, we note that
\begin{equation}\label{eq:def_length_epoch}
    T_k^{(p)} - T_{k-1}^{(p)} \in\set{\floor{\frac{T}{N_p}}, \floor{\frac{T}{N_p}} +1  },\quad k\in[N_p].
\end{equation}

Next, we fix a function $f^\star \in\Fcal$ that will serve as benchmark for the algorithm's predictions. Importantly, we suppose for now that $f^\star$ is fixed and non-adaptive: it does not depend on the realizations of $(x_t,y_t)_{t\in[T]}$. We will later extend the regret bound to potentially adaptive benchmark functions $f^\star\in\Fcal$ in the classification setting.

We next construct by induction some benchmark functions $f_k^{(p)}$ together with reference function sets $S_k^{(p)}$ for all depths $p\in\{0,\ldots,P\}$ and epochs $k\in[N_p]$. At the high-level, we follow the ``trajectory'' of the function $f^\star$ within the covers constructed within the recursive algorithms starting with the final depth-$P$ algorithm $\textsc{R-Cover}_{0,T}^{(P,\epsilon)}(\emptyset)$.

We start at depth $p=P$, for which there is a single epoch $k=1$. We then simply pose $f_1^{(P)}\in\Fcal$ arbitrarily, and let $S_1^{(P)}:=\emptyset$, which is the reference function set used for the final algorithm. In particular we have $f^\star\in \Fcal(S_1^{(P)})=\Fcal$ (see the definition of $\Fcal(S)$ in \cref{eq:def_Fcal_S}).
Now suppose that we have constructed the reference functions $f_k^{(p)}$ and the set $S_k^{(p)}$ for some $p\in[P]$ and all $k\in[N_{p}]$ such that
\begin{equation*}
    f^\star\in \Fcal(S_k^{(p)}),\quad k\in[N_p].
\end{equation*}
We now focus on a given epoch $E_k^{(p)}$, which is composed of two sub-epochs $E_{2k-1}^{(p-1)}$ and $E_{2k}^{(p-1)}$. Fix any $l\in[2]$. At the beginning of epoch $E_{2(k-1)+l}^{(p-1)}$, the algorithm $\textsc{R-Cover}_{T_{k-1}^{(p)},T_k^{(p)}}^{(p,\epsilon)}(S_k^{(p)})$ first constructs a (strict) $\epsilon$-cover of $\Fcal(S_k^{(p)})$ for queries $x_t$ for $t\leq T_{2(k-1)+l-1}^{(p-1)}$, which we denote $\Hcal_{2(k-1)+l}^{(p-1)}$. By construction, we have $\Hcal_{2(k-1)+l}^{(p-1)}\subset \Fcal(S_k^{(p)})$ and $\Fcal(S_k^{(p)})$ contains $f^\star$ by induction hypothesis. Hence, we can select $f_{2(k-1)+l}^{(p-1)}\in \Hcal_{2(k-1)+l}^{(p-1)}$ such that
\begin{equation}\label{eq:consruct_function_next_level}
    \max_{t\in[T_{2(k-1)+l-1}^{(p-1)}]} \abs{f^\star(x_t) - f_{2(k-1)+l}^{(p-1)}(x_t) } \leq \epsilon.
\end{equation}
Additionally, we construct the increased reference set
\begin{equation*}
    S_{2(k-1)+l}^{(p-1)} := S_k^{(p)} \cup \set{ \paren{ f_{2(k-1)+l}^{(p-1)} , T_{2(k-1)+l-1}^{(p-1)} }}.
\end{equation*}
This ends the construction of the reference functions at depth $p-1$. Note that \cref{eq:consruct_function_next_level} exactly implies that the induction hypothesis holds at depth $p-1$.

Each reference function set $S_k^{(p)}$ for $p\in\{0,\ldots,P\}$ and $k\in[N_p]$ corresponds to a run of the depth-$p$ algorithm. Intuitively, this is the depth-$p$ algorithm that uses the ``correct'' reference function set on epoch $E_k^{(p)}$ in the sense that this is the run that always kept $f^\star$ within its search space. To simplify the notations, we will refer to this depth-$p$ algorithm as $\textsc{R-Cover}_k^{(p)}$ (instead of using the full notation $\textsc{R-Cover}_{T_{k-1}^{(p)},T_k^{(p)}}^{(p,\epsilon)}(S_k^{(p)})$). For convenience, we denote by $f_{k,S}^{(p)}\in\Fcal(S_k^{(p)})$ the function that $\textsc{R-Cover}_k^{(p)}$ fixed at the beginning of its run (see lines 2 and 4 of \cref{alg:recursive_construction_regression}). We will also use the notation $\textsc{R-Cover}^{(p)}_k(t)$ to denote the prediction of this algorithm at some time $t\in E_k^{(p)}$. Finally, we denote by $\hat y_t$ the predictions of the final algorithm; note that these are the same as $\textsc{R-Cover}^{(0)}_1(t)$.

Let us now focus on a single depth-$p$ epoch $k\in[N_p]$ for $p>0$. This is composed of $2$ sub-epochs $E_{2(k-1)+l}^{(p-1)}$ for $l\in[2]$.
On each sub-epoch $l\in[2]$, the algorithm performs the exponential weights algorithm using experts that we denote $\Acal_{l,r'}^{(p-1)}$ for $r'\in[r_l^{(p-1)}]$. We also denote by $\Acal_{l,r'}^{(p-1)}(t)$ their prediction for some time during the corresponding epoch $E^{(p-1)}_{2(k-1)+l}$. Last, we use the following notation to denote the magnitude of the expert problem at epoch $l$, where $\hat p_t$ denotes the distribution over the experts $\Acal_{l,r'}^{(p-1)}(t)$ that was used by \textsc{A-Exp} at time $t$:
\begin{align*}
    \tilde \Delta_{k,l}^{(p)} &:= \sum_{t\in E_{2(k-1)+l}^{(p-1)}} \Ebb_{r\sim \hat p_t}\sqb{(\ell_t(\hat y_t ) - \ell_t(\Acal_{l,r}(t) ))^2}\\
    &=\sum_{t\in E_{2(k-1)+l}^{(p-1)}}   \frac{\sum_{r\in[r_l^{(p-1)}]}e^{\eta_t R_{r',t-1}} (\ell_t(\hat y_t ) - \ell_t(\Acal_{l,r'}(t) ))^2}{\sum_{r'\in[r_l^{(p-1)}]}e^{\eta_t R_{r',t-1}}},
\end{align*}
where by abuse of notation we kept $R_{r',t}$ to denote the cumulative regret compared to algorithm $r'$ up to time $t$ during epoch $E_{2(k-1)+l}^{(p-1)}$. For convenience, let us denote by $Reg_{k,l}^{(p)}$ the regret incurred by the exponential weights algorithm $\textsc{A-Exp}$ on each epoch $E_{2(k-1)+l}^{(p-1)}$ for $l\in[2]$ (see line 7 of \cref{alg:recursive_construction_regression}). Then,
\begin{multline}\label{eq:decomposition_one_step}
    \sum_{t\in E_k^{(p)}} \ell_t(\textsc{R-Cover}_k^{(p)}(t) ) =\sum_{l\in[2]} \sum_{t\in E_{2(k-1)+l}^{(p-1)}}  \ell_t(\textsc{R-Cover}_k^{(p)}(t) )\\
    \leq \sum_{l\in[2]} \min\set{ \sum_{t\in E_{2(k-1)+l}^{(p-1)}} \ell_t\paren{f_{k,S}^{(p)}(x_t) } , \sum_{t\in E_{2(k-1)+l}^{(p-1)}}  \ell_t\paren{\textsc{R-Cover}_{2(k-1)+l}^{(p-1)}(t) } } + \sum_{l\in[2]} Reg_{k,l}^{(p)}.    
\end{multline}
In the last inequality we used the fact that the algorithm $\textsc{R-Cover}^{(p-1)}_{2(k-1)+l}$ that uses the reference set $S_{2(k-1)+l}^{(p-1)}$ is one of the experts $\Acal_{l,r'}$ for $r'\in[r_l^{(p-1)}]$, as well as the expert that uses $f_{k,S}^{(p)}$ as reference function (see line 7 or \cref{alg:recursive_construction_regression}). We next use \cref{lemma:regret_exponentially_weighted} to bound the regret terms $Reg_{k,l}^{(p)}$. 
First, recall that we always have $r_l^{(p-1)}\leq \Ncal(\Fcal;\epsilon,T)+1$ where by abuse of notation $\Ncal(\Fcal;\epsilon,T)$ is the $\epsilon$-covering number of $\Fcal(S_k^{(p)})$ on $(x_t)_{t\in[T]}$ (this abuse of notation is mild from the discussion around \cref{eq:bound_covering_numbers_abuse}). Taking the union bound, we obtain that with probability at least $1-\delta$,
\begin{align*}
    \sum_{l\in[2]} Reg_{k,l}^{(p)} &\leq \sum_{l\in[2]} 12\sqrt{ \max\paren{\tilde\Delta_{k,l}^{(p-1)},1} \ln\paren{\Ncal(\Fcal;\epsilon,T)+1 } } + 4\ln\frac{2}{\delta}.
\end{align*}
Instead of working with the quantities $\tilde \Delta_{k,l}^{(p)}$, we instead define
\begin{equation*}
    \Delta_k^{(p)}:= \Delta_{k,1}^{(p)} + \Delta_{k,2}^{(p)},\quad p\in[P],k\in[N_p].
\end{equation*}
Applying Jensen's inequality gives
\begin{equation}\label{eq:regret_a_exp_1_step}
    \sum_{l\in[2]} Reg_{k,l}^{(p)}   \leq  12\sqrt{2\max\paren{\Delta_k^{(p)},2} \ln\paren{\Ncal(\Fcal;\epsilon,T)+1 } }  + 4\ln\frac{2}{\delta}.
\end{equation}

We are now ready to decompose the regret of the algorithm along the learning trajectory using the previous bound recursively. We start from the level $P$ and go down to some fixed depth $p_0\in\{0,\ldots,P\}$. Using \cref{eq:decomposition_one_step} gives
\begin{equation*}
    \sum_{t\in E_1^{(P)}} \ell_t(\hat y_t )\leq \sum_{p=\max(p_0,1)}^P \sum_{k\in[N_p], l\in[2]} Reg_{k,l}^{(p)}
    +\sum_{k\in[N_{p_0}]}\sum_{t\in E_k^{(p_0)}} \ell_t\paren{ f_{k,S}^{(p_0)}(x_t) } 
\end{equation*}
Next, using \cref{eq:regret_a_exp_1_step}, with probability at least $1-\delta\sum_{p=p_0}^{P}N_p \leq  1-2N_{p_0}\delta T \leq 1-2\delta T$ (recall that $T\geq 2^P$) we have for any choice of $p_0\in\{0,\ldots,P\}$,
\begin{multline}\label{eq:decomposed_regret}
     \sum_{t=1}^T \ell_t(\hat y_t ) - \ell_t(f^\star(x_t) ) 
    \leq 12\sum_{p=\max(p_0,1)}^{P} \sum_{k\in[N_p]} \sqrt{2\max\paren{\Delta_k^{(p)},2} \ln\paren{\Ncal(\Fcal;\epsilon,T)+1 }}\\
      + 8N_{p_0}\ln\frac{2}{\delta} +\sum_{k\in[N_{p_0}]}\sum_{t\in E_k^{(p_0)}} \ell_t\paren{f_{k,S}^{(p_0)}(x_t) } - \ell_t(f^\star(x_t) ). 
\end{multline}

Up to the last layer for $p=p_0$, the previous inequality shows that the regret of the algorithm essentially only corresponds to the regret accumulated by the learning with expert algorithms along the trajectory for the benchmark function $f^\star$.
For convenience, we introduce for all $p\in\{0,\ldots,P\}$ and $k\in[N_p]$ the quantity
\begin{equation*}
    \Lambda_k^{(p)}:= \sum_{t\in E_k^{(p)}} \ell_t\paren{f_{k,S}^{(p)}(x_t) } - \ell_t(f^\star(x_t) ).
\end{equation*}
We used a similar notation for $\Delta_k^{(p)}$ and $\Lambda_k^{(p)}$ for $p\in[P]$ because these terms will be bounded with the same techniques.

\subsection{Bounding the regret term for each depth}
\label{subsec:bounding_each_term}

We next bound each term of the right-hand side of \cref{eq:decomposed_regret} separately for each layer $p\in\{p_0,\ldots,P\}$. That is, we need to bound the error terms $\Delta_k^{(p)}$ and $\Lambda_k^{(p)}$ for $k\in[N_p]$.
Fix $p\in\{0,\ldots,P\}$ and $k\in[N_p]$ and let
\begin{equation}\label{eq:def_Pcal_close_functions}
    \Pcal_k^{(p)}:=\set{ f\in\Fcal: \max_{t\in[T_{k-1}^{(p)}]} \abs{f(x_t)-f^\star(x_t)} \leq 2\epsilon}.
\end{equation}
We then define for any $r\geq 1$,
\begin{align}
     \Gamma_k^{(p,r)} &:= \sum_{t\in E_k^{(p)}} \gamma^{(p,r)}(t) \quad\text{where}\quad \gamma^{(p,r)}(t) := \sup_{f\in\Pcal_k^{(p)}}\Ebb\sqb{ |f(x_t)-f^\star(x_t)|^r\mid\Hcal_{t-1}}, \quad t\in E_k^{(p)},\label{eq:definition_gama}
\end{align}
Intuitively, $\Gamma_k^{(p,r)}$ quantifies the $\ell_r$ discrepancy between the queries on epoch $E_k^{(p)}$ and queries prior to this epoch. This measures the level of non-stationarity of the smooth process $(x_t)_{t\in[T]}$ on each epoch. The following results shows that it suffices to bound $\Gamma_k^{(p,r)}$ to bound $\Delta_k^{(p)}$ and $\Lambda_k^{(p)}$.

\begin{lemma}\label{lemma:bounding_deltas}
    With probability at least $1-\delta$,
    \begin{equation*}
        \Delta_k^{(p)} \leq 5\Gamma_k^{(p,2)} + 16\ln\frac{T}{\delta},\quad p\in[P],k\in[N_p].
    \end{equation*}
    Similarly, for any $p\in \{0,\ldots,P\}$, with probability at least $1-\delta$,
    \begin{equation*}
        \Lambda_k^{(p)} \leq 2\Gamma_k^{(p,1)} + 3\ln\frac{T}{\delta},\quad k\in[N_p].
    \end{equation*}
\end{lemma}

\begin{proof}
    Fix $p\in[P]$ and $k\in[N_p]$. During its run on epoch $E_k^{(p)}$, the learning with expert prediction algorithm $\textsc{A-Exp}$ uses predictions from depth-$(p-1)$ algorithms. In practice, all considered sub-algorithms---that is, for epochs $E_{k'}^{p'}$ with $p'<p$ and such that $E_{k'}^{(p')}\subset E_k^{(p)}$---are \emph{proper} in the sense that they proceed by first selecting some predictor function $\hat f_t\in\Fcal$ then implementing its prediction $\hat f_t(x_t)$. The choice of the function $\hat f_t$ is randomized, but is made before observing the value $x_t$. As an important remark, all these potentially-selected functions belong to $\Fcal(S_k^{(p)})$ since for sub-algorithms we append reference functions $(f_i,t_i)$ to the reference set $S_k^{(p)}$. Next, note that by construction, we have $(f_k^{(p)},T_{k-1}^{(p)})\in S_k^{(p)}$ (see the recursion line 7 of \cref{alg:recursive_construction_regression}). In particular, all these functions belong to $\Fcal(S_k^{(p)}) \subset B_{f_k^{(p)}}(\Fcal;\epsilon,T_{k-1}^{(p)}):= B_k^{(p)}$, where we introduced the last notation for simplicity. For $p=P$, we simply have $B_1^{(P)}=\Fcal$.

    We use the same notations as in \cref{subsec:regret_decomposition}: for all $l\in[2]$, during epoch $E_{2(p-1)+l}^{(p-1)}$ the algorithm $\textsc{R-Cover}_k^{(p)}$ performs the exponential weights algorithm using as experts the predictions of the lower-level algorithms, which we denote $\Acal_{l,r'}$ for $r'\in[r_l^{(p-1)}]$.
    As a summary of the previous discussion, for any $t\in E_{2(k-1)+l}^{(p-1)}$, we can define a (random) function $f_{l,r',t}\in B_k^{(p)}$ that the algorithm $\Acal_{l,r'}$ has committed to use for its prediction at time $t$. We also note that $B_k^{(p)}$ only depends on the history up to time $T_{k-1}^{(p)}$. Altogether, $x_t\mid \sigma(\Hcal_{t-1};f_{l,r',t},r'\in[r_l^{(p-1)}],B_k^{(p)})$ still has the same distribution as $x_t\mid \Hcal_{t-1}$. On top of these predictions, $\textsc{R-Cover}_k^{(p)}$ performs the exponential weights algorithm: for iteration $t\in E_{2(k-1)+l}^{(p-1)}$ it first samples $\hat r_t\sim q_{2(k-1)+l}^{(p)}(t)$ for some $\Hcal_{t-1}$-measurable distribution $q_{2(k-1)+l}^{(p)}(t)$ on $[r_l^{(p-1)}]$ then commits to using the prediction of $\Acal_{l,\hat r_t}$, that is using the function $f_{l,\hat r_t,t}\in B_k^{(p)}$. Now construct a tangent sequence $(\hat r_t')_{t\in E_k^{(p)}}$. That is, conditionally on $\Hcal_{t-1}$ we sample $\hat r_t'$ independently from $r_t$ with the same distribution $q_{2(k-1)+l}^{(p)}(t)$. We have
    \begin{align*}
        &\Delta_k^{(p)} = \sum_{l\in[2]} \tilde \Delta_{k,l}^{(p)}\\
        &=\sum_{l\in[2]} \sum_{t\in E_{C(k-1)+l}^{(p-1)}}  \Ebb_{\hat r_t'}\sqb{\paren{\ell_t(f_{l,\hat r_t,t}(x_t) ) - \ell_t(f_{l,\hat r_t',t}(x_t) )}^2 \mid \Hcal_t,\, \hat r_t,\,f_{r,l',t},l'\in[r_l^{(p-1)}]}\\
        &\leq\sum_{l\in[2]} \sum_{t\in E_{C(k-1)+l}^{(p-1)}} \Ebb_{\hat r_t'}\sqb{\paren{ f_{l,\hat r_t,t}(x_t)  -  f_{l,\hat r_t',t}(x_t) }^2 \mid \Hcal_t,\, \hat r_t,\,f_{r,l',t},l'\in[r_l^{(p-1)}]}\\
        &\leq 2\sum_{l\in[2]} \sum_{t\in E_{C(k-1)+l}^{(p-1)}} \underbrace{\Ebb_{\hat r_t'}\sqb{\paren{ f_{l,\hat r_t,t}(x_t)  -  f^\star(x_t) }^2 + \paren{ f_{l,\hat r_t',t}(x_t)  -  f^\star(x_t) }^2 \mid \Hcal_t,\, \hat r_t,\,f_{r,l',t},l'\in[r_l^{(p-1)}]}}_{X_t^{(p)}},
    \end{align*}
    where we used the identity $(a+b)^2\leq 2(a^2+b^2)$ for $a,b\geq 1$. Next, note that
    \begin{align*}
        Y_t^{(p)} &:=\Ebb_{x_t,\hat r_t} \sqb{X_t^{(p)} \mid \Hcal_{t-1},\,f_{r,l',t},l'\in[r_l^{(p-1)}]} \\
        &=\Ebb_{\hat r_t,\hat r_t'}\left[ \Ebb_{x_t\mid\Hcal_{t-1}}\sqb{\paren{ f_{l,\hat r_t,t}(x_t)  -  f^\star(x_t) }^2 + \paren{ f_{l,\hat r_t',t}(x_t)  -  f^\star(x_t) }^2 \mid \Hcal_{t-1},\, f_{l,\hat r_t,t},\,f_{l,\hat r_t',t}} \right.\\
        &\qquad\qquad\qquad\qquad\qquad\qquad\qquad\qquad\qquad\qquad\qquad\qquad \left.\mid \Hcal_{t-1}, \,f_{r,l',t},l'\in[r_l^{(p-1)}] \right]\\
        &\leq 2\sup_{f\in B_k^{(p)}} \Ebb_{x_t\mid\Hcal_{t-1}}\sqb{ \paren{f(x_t) - f^\star(x_t)}^2\mid\Hcal_{t-1}}\\
        &\leq 2\gamma^{(p,2)}(t).
    \end{align*}
    In the last inequality, we used the definition of $B_k^{(p)} = B_{f_k^{(p)}}(\Fcal;\epsilon,T_{k-1}^{(p)})$ together with the fact that by construction $f^\star\in\Fcal(S_k^{(p)})\subset B_k^{(p)}$. Hence, the triangle inequality implies that $B_k^{(p)} \subset \Pcal_k^{(p)}$.

    The previous equation shows that $\Delta_k^{(p)}-4\Gamma_k^{(p,2)} \leq 2\sum_{t\in E_p^{(k)}} (X_t^{(p)}-2\gamma^{(p,2)}(t))$ where the right-hand side is a sum of super-martingale differences. Further, these differences are bounded in absolute value by $|X_t^{(p)}-2\gamma^{(p,2)}(t)|\leq 4$.
    To bound $\Delta_k^{(p)}$ in terms of $\Gamma_k^{(p,2)}$, we can directly use Azuma-Hoeffding's inequality which would give an extra term of the form $\sqrt{(T_k^{(p)}-T_{k-1}^{(p)})\ln\frac{1}{\delta}}$ for a bound with probability $1-\delta$. Because we will consider cases for which $\Gamma_k^{(p)}$ is significantly smaller than $\sqrt{T_k^{(p)}-T_{k-1}^{(p)}}$, we instead use Freedman's inequality stated in \cref{lemma:freedman_inequality}
    to the sum $\sum_{t\in E_p^{(k)}} X_t^{(p)}-Y_t^{(p)}$ using the filtration $\Fcal_t=\sigma(X_{t'}^{(p)},t'\leq t; Y_{t'}^{(p)},t'\leq t+1)$ (note that $\Ebb[X_t^{(p)}\mid \Fcal_{t-1}]=Y_t^{(p)}$). To do so, we compute
    \begin{align*}
        \sum_{t\in E_k^{(p)}} \Ebb\sqb{(X_t^{(p)}-Y_t^{(p)})^2 \mid \Fcal_{t-1}} &\leq \sum_{t\in E_k^{(p)}} \Ebb\sqb{(X_t^{(p)})^2 \mid \Fcal_{t-1}}\\
        &\overset{(i)}{\leq} 2 \sum_{t\in E_k^{(p)}} \Ebb\sqb{X_t^{(p)} \mid \Fcal_{t-1}}  = 2 \sum_{t\in E_k^{(p)}}Y_t^{(p)} \\
        &\leq 4\sum_{t\in E_k^{(p)}}\gamma^{(p,2)}(t) =4\Gamma_k^{(p,2)}.
    \end{align*} 
    In $(i)$ we used the fact that $|X_t^{(p)}|\leq 2$ since functions in $\Fcal$ have value in $[0,1]$. Last, we always have $|X_t^{(p)} - Y_t^{(p)}|\leq 4$.
    Then, \cref{lemma:freedman_inequality} with the union bound over all $p\in[P]$ and $k\in[N_p]$ implies that with probability at least $1-\delta$, using $\eta=1/8$,
    \begin{equation}\label{eq:final_concentration_Delta_Gamma}
        \Delta_k^{(p)} \leq 2\sum_{t\in E_k^{(p)}} (X_t^{(p)}-Y_t^{(p)}) + 4\Gamma_k^{(p,2)} \leq 5\Gamma_k^{(p,2)} + 16\ln\frac{T}{\delta},\quad p\in[P],k\in[N_p].
    \end{equation}
    Here we used the fact that $\sum_{p=1}^P N_p = \sum_{p=1}^P 2^{P-p}\leq T$.

    We next bound the terms $\Lambda_k^{(p)}$ in a similar fashion for a fixed $p\in\{0,\ldots,P\}$. First, note that by construction, for any $k\in[N_p]$, we still have
    \begin{equation*}
        f_{k,S}^{(p)},f^\star \in \Fcal(S_k^{(p)}) \subset B_k^{(p)}:=B_{f_k^{(p)}}(\Fcal;\epsilon,T_{k-1}^{(p)}).
    \end{equation*}
    As a result, as discussed above $f_{k,S}^{(p)}\in \Pcal_k^{(p)}$. Further, $f_{k,S}^{(p)}$ is fixed at the beginning of epoch $E_k^{(p)}$ hence can be made $\Hcal_{T_{k-1}^{(p)}}$-measurable without loss of generality. Then, for any $k\in[N_p]$ using the fact that the losses are $1$-Lipschitz,
    \begin{equation*}
        \Lambda_k^{(p)} \leq  \sum_{t\in E_k^{(p)}} \underbrace{ \abs{  f_{k,S}^{(p)}(x_t)  -  f^\star(x_t)} }_{\tilde X_t^{(p)}}
    \end{equation*}
    and similarly as before,
    \begin{align*}
        \tilde Y_t^{(p)} &:= \Ebb_{x_t,\hat r_t} \sqb{X_t^{(p)} \mid \Hcal_{t-1},\,f_{k,S}^{(p)}}\\
        &=\Ebb_{\hat r_t,\hat r_t'}\sqb{ \Ebb_{x_t\mid\Hcal_{t-1}}\sqb{\abs{f_{k,S}^{(p)}(x_t) - f^\star(x_t) }\mid \Hcal_{t-1},\, f_{k,S}^{(p)}} \mid \Hcal_{t-1},\,f_{k,S}^{(p)}} \leq \gamma^{(p,1)}(t).
    \end{align*}
    We can also bound $|\tilde X_t^{(p)}-\tilde Y_t^{(p)}| \leq 2$. Again, we use Freedman's inequality  to $\sum_{t\in E_k^{(p)}} \tilde X_t^{(p)}-\tilde Y_t^{(p)}$ noting that
    \begin{align*}
        \sum_{t\in E_p^{(k)}} \Ebb\sqb{(\tilde X_t^{(p)}-\tilde Y_t^{(p)})^2 \mid \Fcal_{t-1}} \leq \sum_{t\in E_p^{(k)}} \Ebb\sqb{(\tilde X_t^{(p)})^2 \mid \Fcal_{t-1}} \leq 2 \sum_{t\in E_p^{(k)}} \Ebb\sqb{\tilde X_t^{(p)} \mid \Fcal_{t-1}}  \leq 2\Gamma_k^{(p,1)}.
    \end{align*} 
    Similarly as before, \cref{lemma:freedman_inequality} with $\eta=1/3$ with the union bound on all $k\in[N_p]$ then implies that with probability at least $1-\delta$,
    \begin{equation*}
        \Lambda_k^{(p)} \leq \sum_{t=1}^T(\tilde X_t^{(p)}-\tilde Y_t^{(p)}) + \Gamma_k^{(p,1)} \leq 2\Gamma_k^{(p,1)} + 3\ln\frac{T}{\delta},\quad k\in[N_p].
    \end{equation*}
    Here we used $N_p\leq T$. This ends the proof of the lemma.
\end{proof}

We denote by $L_p:=\floor{T/N_p}+1$ the maximum length of each depth-$p$ epoch. We recall that from \cref{eq:def_length_epoch} the depth-$p$ epochs all have length $L_p$ of $L_p-1$. Note that by construction because $p\geq p_0\geq 1$, we have $L_p\geq 2$ and hence $L_p-1\geq 1$.
In the worst case, each term $\Gamma_k^{(p)}$ for $k\in[N_p]$ could be as large as $L_p$. We show, however, that smoothness ensures that such epochs are very few.

\begin{lemma}\label{lemma:main_bound_modified}
    Fix $r\geq 1$, $p\in\{0,\ldots,P\}$ and suppose that $(x_t)_{t\in[T]}$ is a $\sigma$-smooth stochastic process with respect to some measure $\mu$, where $T\geq 2$. For any parameters $w\geq 2$ and $q\in(0,1]$ satisfying
    \begin{equation}\label{eq:assumption_q}
        q\geq 12\sqrt{(2\epsilon)^r \frac{2\ln (eT)}{\sigma}},
    \end{equation}
    with probability at least $1-\delta$,
    \begin{multline*}
        \abs{\set{k\in[N_p]:\sum_{t\in E_k^{(p)}} \gamma^{(p,r)}(t)\1[\gamma^{(p,r)}(t) \geq q] \geq w}} \\
        \leq \frac{c_0 r \ln^2  T}{q \sigma w} \paren{\ln\Ebb_\mu\sqb{W_{8\frac{T}{\sigma}\ln (\frac{T}{2\delta})} \paren{ \Fcal}} + \ln\frac{T}{\delta}+w},
    \end{multline*}
    for some universal constant $c_0>0$. In particular, if
    \begin{equation}\label{eq:stronger_assumption_q}
        q \geq \max\paren{24\sqrt{(2\epsilon)^r\frac{2\ln (eT)}{\sigma}} , \frac{2w}{L_p-1}},
    \end{equation}
    then with probability at least $1-\delta$,
    \begin{equation*}
        \abs{\set{k\in[N_p]: \Gamma_k^{(p,r)} \geq q(T_k^{(p)}-T_{k-1}^{(p)})}} \leq  \frac{2c_0 r \ln^2 T}{q \sigma w} \paren{\ln\Ebb_\mu\sqb{W_{8\frac{T}{\sigma}\ln (\frac{T}{2\delta})} \paren{\Fcal}} + \ln\frac{T}{\delta}+w}.
    \end{equation*}
\end{lemma}

We defer the proof of this result to \cref{subsec:proof_main_lemma}.
We now select the parameter
\begin{equation*}
    w=w(T,\delta):=\ln\Ebb_\mu\sqb{W_{8T\ln (\frac{T}{2\delta})/\sigma} \paren{\Fcal}} + 10\ln\frac{T}{\delta} +2
\end{equation*}
which satisfies $w\geq 2$. We combine \cref{lemma:bounding_deltas,lemma:main_bound_modified} both for the probability tolerance $\delta$, which implies that for any $w\geq 2$ and $q\in(0,1]$ satisfying \cref{eq:stronger_assumption_q} for $r=2$, with probability at least $1-2\delta$,
\begin{align}
&\abs{\set{k\in[N_p]: \Delta_k^{(p)} \geq 6q(T_k^{(p)}-T_{k-1}^{(p)})  }} \notag \\
    &\overset{(i)}{\leq} \abs{\set{k\in[N_p]: \Delta_k^{(p)} \geq 5q(T_k^{(p)}-T_{k-1}^{(p)}) + 20\ln\frac{T}{\delta}}} \notag \\
    &\overset{(ii)}{\leq} \abs{\set{k\in[N_p]: \Gamma_k^{(p,2)} \geq q(T_k^{(p)}-T_{k-1}^{(p)}) }} \notag \\
    &\overset{(iii)}{\leq}  \frac{c_0\ln^2  T}{q \sigma w(T,\delta)} \paren{\ln\Ebb_\mu\sqb{W_{8T\ln (\frac{T}{2\delta})/\sigma} \paren{\Fcal}} + \ln\frac{T}{\delta}+w(T,\delta)} \leq \frac{2c_0\ln^2  T}{q \sigma},\label{eq:bound_for_Delta}
\end{align}
for some constant $c_0\geq 1$. In $(i)$ we used the fact that $q\geq \frac{2w}{L_p-1} \geq \frac{20\ln\frac{T}{\delta}}{T_k^{(p)}-T_{k-1}^{(p)}}$. In $(ii)$ we used \cref{lemma:bounding_deltas} and in $(iii)$ we used \cref{lemma:main_bound_modified}.
Similarly, for any $w\geq 2$ and $q\in(0,1]$ satisfying \cref{eq:stronger_assumption_q} for $r=1$, with probability at least $1-2\delta$,
\begin{equation}\label{eq:bound_for_Lambda}
    \abs{\set{k\in[N_p]: \Lambda_k^{(p)} \geq 3q(T_k^{(p)}-T_{k-1}^{(p)}) }} 
    \leq  \frac{2c_0\ln^2 (2 T)}{q \sigma } .
\end{equation}

\paragraph{Bounding the regret term involving $\Delta_k^{(p)}$.}Using \cref{eq:bound_for_Delta,eq:bound_for_Lambda}, we can now bound the regret terms from the decomposition in \cref{eq:decomposed_regret}. We start with the term involving the quantities $\Delta_k^{(p)}$. For $p\in\{p_0,\ldots,P\}$ we let
\begin{equation*}
    q_0^{(p)} := 300 \cdot \max\paren{2\epsilon\sqrt{\frac{\ln T}{\sigma}}, \frac{w(T,\delta)}{L_p-1},\frac{ c_0\ln^2 T}{\sigma N_p}}.
\end{equation*}
If $q_0^{(p)} \geq 1$, we can simply bound
\begin{equation}\label{eq:trivial_bound_Delta}
    \sum_{k\in[N_p]}\sqrt{\max\paren{\Delta_k^{(p)},2}} \overset{(i)}{\leq} \sum_{k\in[N_p]}2\sqrt{T_k^{(p)}-T_{k-1}^{(p)}}  \overset{(ii)}{\leq} 2\sqrt{T N_p} \leq 2\sqrt{q_0^{(p)} T N_p}.
\end{equation}
In $(i)$ we used the fact that $\Delta_k^{(p)}$ is a sum of terms bounded by $1$ since the loss is $1$-Lipschitz and the functions $f\in\Fcal$ have value within $[0,1]$. In $(ii)$ we used Jensen's inequality.

Otherwise, if $q_0^{(p)} \leq 1$, we introduce the parameters $q_s^{(p)} = 4^s q_0^{(p)}$ for $s\geq 0$ and let $s_p$ be the last index for which $q_s^{(p)} \leq 4$. We then define the sets
\begin{equation*}
    \Tcal^{(p)}(s):= \set{ k\in[N_p]: q_s^{(p)} (T_k^{(p)}-T_{k-1}^{(p)}) \leq \Delta_k^{(p)} < q_{s+1}^{(p)}  (T_k^{(p)}-T_{k-1}^{(p)}) }, \quad s\in \{0,\ldots,s_p\}.
\end{equation*}
By construction of $s_p$, any epoch $k\in[N_p]$ either belongs to one of the sets above or satisfies $\Delta_k^{(p)} \leq q_0^{(p)}(T_k^{(p)} - T_{k-1}^{(p)})$.
Also, note that there exists a constant $c>0$ such that $s_p \leq c\ln T$ since $L_p,N_p\leq T$. 
We also recall that $P\leq \log_2(T)$. 
Hence, up to changing the constant $c>0$, \cref{eq:bound_for_Delta} implies that on some event $\Ecal_\delta$ with probability at least $1-c\delta  \ln^2 T$, for all $p\in\{p_0,\ldots,P\}$ such that $q_0^{(p)}\leq 1$,
\begin{equation*}
    \abs{ \Tcal^{(p)}(s) } \leq \frac{12c_0\ln^2 T}{q_s^{(p)} \sigma} ,\quad s\in\{0,\ldots,s_p\}.
\end{equation*}

Hence, on $\Ecal$, for any $p\in\{p_0,\ldots,P\}$ such that $q_0^{(p)}\leq 1$, we have
\begin{align}
    \sum_{k\in[N_p]} \sqrt{\max\paren{\Delta_k^{(p)},2}} &\overset{(i)}{\leq} \sum_{k\in[N_p]} \sqrt{q_0^{(p)}(T_k^{(p)}-T_{k-1}^{(p)})}  + \sum_{s=0}^{s_p} 
 |\Tcal^{(p)}(s)| \cdot \sqrt{q_{s+1}^{(p)}L_p} \notag\\
  &\overset{(ii)}{\leq} \sqrt{q_0^{(p)} TN_p }  + \frac{12c_0\ln^2 T }{\sigma} \sum_{s=0}^{s_p} \frac{\sqrt{q_{s+1}^{(p)}L_p} }{ q_s^{(p)}}  \notag\\
    &= \sqrt{q_0^{(p)} TN_p }  + \frac{24c_0\ln^2 T}{\sigma}\sum_{s=0}^{s_p} \sqrt{\frac{L_p}{q_s^{(p)}} } \notag\\
    &\overset{(iii)}{\leq} \sqrt{q_0^{(p)} TN_p }  + \frac{48c_0\ln^2 T}{\sigma} \sqrt{\frac{2T}{q_0^{(p)}N_p} } \notag\\
    &\overset{(iv)}{\leq} 2\sqrt{q_0^{(p)} TN_p } . \label{eq:bound_expected_sum_level_p}
\end{align}
In $(i)$ we used the fact that $q_0^{(p)}(T_k^{(p)}-T_{k-1}^{(p)}) \geq 2w \geq 2$ to delete the maximum with $2$, and in $(ii)$ we used Jensen's inequality. In $(iii)$ we use the fact that $L_pN_p \leq T+N_p\leq 2T$ and in $(iv)$ we used the fact that $q_0^{(p)} \geq 48\sqrt 2\frac{c_0\ln^2 T}{\sigma N_p}$.

As a result, on the event $\Ecal_\delta$, we can combine \cref{eq:trivial_bound_Delta,eq:bound_expected_sum_level_p} to obtain
\begin{align}
    \sum_{p=p_0}^P \sum_{k\in[N_p]} \sqrt{\max\paren{\Delta_k^{(p)},2}} &\lesssim \sqrt{\epsilon}\paren{\frac{\ln T}{\sigma}}^{1/4} \sum_{p=p_0}^P \sqrt{TN_p} + \sqrt{w(T,\delta)}\sum_{p=p_0}^P \sqrt{\frac{TN_p}{L_p-1}}\notag  \\
    &\qquad\qquad+ (P-p_0+1) \sqrt{\frac{\ln^2  T}{\sigma} \cdot T} \notag \\
    &\lesssim \sqrt {\epsilon TN_{p_0}} + N_{p_0} \sqrt{w(T,\delta)} + \ln^2  T \sqrt{\frac{T}{\sigma}}, \label{eq:bound_regret_Delta_terms}
\end{align}
where the $\lesssim$ symbol only hides universal constants.

\paragraph{Bounding the regret term involving $\Lambda_k^{(p)}$.} We next turn to last regret term from the decomposition in \cref{eq:decomposed_regret}. We let
\begin{equation*}
    \tilde q_0^{(p)}  := 300 \cdot \max\paren{\sqrt{2\epsilon\frac{\ln T}{\sigma}}, \frac{w(T,\delta)}{L_p-1},\frac{ c_0\ln^3  T}{\sigma N_p}}.
\end{equation*}
If $\tilde q_0^{(p_0)} \geq 2$, the resulting regret bound is vacuous. Hence, we focus on the case when $\tilde q_0^{(p_0)} \geq 2$. As before, we let $\tilde q_s^{(p_0)} = 2^s \tilde q_0^{(p_0)}$ for $s\geq 0$ and let $\tilde s_{p_0}$ be the last index such that $\tilde q_s^{(p_0)} \leq 2$. As above, we have $\tilde s_{p_0} \leq c\ln T$ for some constant $c>0$. Then,  \cref{eq:bound_for_Lambda} implies that on an event $\Fcal_\delta$ of probability at least $1-c\delta \ln T$, for all $s\in\{0,\ldots,\tilde s_{p_0}\}$, we have
\begin{equation*}
    \abs{\set{ k\in[N_p]: \tilde q_s^{(p_0)} (T_k^{(p_0)}-T_{k-1}^{(p_0)}) \leq \Lambda_k^{(p_0)} <\tilde q_{s+1}^{(p_0)} (T_k^{(p_0)}-T_{k-1}^{(p_0)} )  }} \leq \frac{6c_0 \ln^2 T}{\tilde q_s^{(p_0)} \sigma}.
\end{equation*}
Using the same arguments as above shows that on $\Fcal_\delta$,
\begin{align}
    \sum_{k\in[N_{p_0}]} \Lambda_k^{(p_0)} &\leq \tilde q_0^{(p_0)} T + \frac{6c_0 \ln^2 T}{\sigma} \sum_{s=0}^{\tilde s_{p_0}} \frac{\tilde q_{s+1}^{(p_0)} L_{p_0}}{\tilde q_s^{(p_0)}} \notag \\
    &\leq \tilde q_0^{(p_0)} T + \frac{12c_0 \ln^2 T}{\sigma} (\tilde s_{p_0} +1) L_{p_0} \notag \\
    &\leq \tilde q_0^{(p_0)} T + \frac{24c c_0 \ln^3  T}{\sigma}  L_{p_0} \notag \\
    &\leq (1+c)\tilde q_0^{(p_0)} T.\label{eq:bound_Lambda_final}
\end{align}
In the last inequality, we used $N_{p_0}L_{p_0} \leq 2T$ and the definition of $\tilde q_0^{(p_0)}$. Note that \cref{eq:bound_Lambda_final} also trivially holds if $\tilde q_0^{(p_0)}\geq 2$.

\paragraph{Final regret bound}

We now combine the bounds from \cref{eq:bound_regret_Delta_terms,eq:bound_Lambda_final} within the regret decomposition from \cref{eq:decomposed_regret} which shows that on $\Ecal_\delta\cap\Fcal_\delta$ of probability at least $1-2c\delta \ln^2 T$, 
\begin{align}
     \sum_{t=1}^T \ell_t(\hat y_t ) - \ell_t(f^\star(x_t) ) &\lesssim \paren{ \sqrt {\epsilon TN_{p_0}} + N_{p_0} \sqrt{w(T,\delta)} + \ln^2 T \sqrt{\frac{T}{\sigma}} } \sqrt{\ln\paren{\Ncal(\Fcal;\epsilon,T)+1 }} \notag \\
     &\qquad\qquad  + \sqrt{\frac{\epsilon \ln T}{\sigma}}\cdot T + w(T,\delta) N_{p_0} + \frac{\ln^3 T}{\sigma N_{p_0}}T. \label{eq:regret_bound_fixed_p_0}
\end{align}
Here we used the fact that $w(T,\delta)\geq \ln\frac{1}{\delta}$ to delete the term $N_{p_0}\ln\frac{1}{\delta}$. This holds for all $p_0\in[P]$. Hence, we obtain the following result which implies in particular \cref{thm:main_regret_regression}.

\begin{theorem}\label{thm:oblivious_adversary}
    Let $\Fcal:\Xcal\to\{0,1\}$ be a function class with VC dimension $d$. Suppose that $(x_t)_{t\geq 1}$ is a $\sigma$-smooth sequence on $\Xcal$ with respect to some unknown base measure $\mu$. Then, $\textsc{R-Cover}$ (Recursive Covering) with the parameter $\epsilon\in[0,1]$ makes predictions $\hat y_t$ such that for any function $f^\star\in\Fcal$, with probability at least $1-\delta$,
    \begin{align}
        \sum_{t=1}^T  \ell_t(\hat y_t ) - & \sum_{t=1}^T \ell_t(f^\star(x_t) ) \notag\\
        &\lesssim \min_{N_0} \left\{  \paren{ \sqrt {\epsilon TN_0} + N_0 \sqrt{w(T,\delta)} + \ln^2 T \sqrt{\frac{T}{\sigma}} } \sqrt{\ln\paren{\Ncal(\Fcal;\epsilon,T)+1 }} \right. \notag\\
     &\qquad\qquad \qquad\qquad\qquad\qquad \qquad\qquad   \left. + \sqrt{\frac{\epsilon \ln T}{\sigma}}\cdot T + N_0 w(T,\delta) + \frac{\ln^3 T}{\sigma N_0}T  \right\}, \label{eq:main_regret_bound_oblivious}
    \end{align}
    where
    \begin{equation*}
        w(T,\delta) = \ln\Ebb_\mu\sqb{W_{8T\ln (\frac{cT\ln^2 T}{\delta})/\sigma} \paren{\Fcal}} + \ln\frac{T}{\delta},
    \end{equation*}
    for some universal constant $c>0$. we recall that the covering numbers can be bounded in terms of the fat-shattering dimension via \cref{thm:fat_shattering_bound}.

    For instance, if there exists $d\geq 1$ such that $\mathsf{fat}_\Fcal(r)\leq d\ln\frac{1}{r}$ for all $r>0$, the regret bound for $\textsc{R-Cover}$ with parameter $\epsilon=1/T$ becomes
    \begin{equation*}
        \sum_{t=1}^T \ell_t(\hat y_t ) - \sum_{t=1}^T \ell_t(f^\star(x_t) ) \leq C  \sqrt{\frac{\paren{d\ln^3 \paren{T\ln\frac{1}{\delta}}  +   \ln\frac{T}{\delta}} \ln^3 T}{\sigma} \cdot T},
    \end{equation*}
    for some constant $C>0$.

    If $\mathsf{fat}_\Fcal(r)\lesssim r^{-p}$ for $p>0$, the regret bound for $\textsc{R-Cover}$ with parameter $\epsilon = \paren{\frac{\ln T}{T}}^{\frac{1}{p+1}}$ becomes
    \begin{equation*}
        \sum_{t=1}^T \ell_t(\hat y_t ) - \sum_{t=1}^T \ell_t(f^\star(x_t) ) \lesssim_p \frac{\ln^3 T}{\sqrt \sigma} \cdot 
        T^{1-\frac{1}{2(p+1)}} + \frac{\ln^4 T \ln^3\frac{T}{\delta}}{\sigma} \cdot T^{1-\frac{1}{2(p+1)} - \frac{\min(p,1)}{2(p+1)(p+2)}}.
    \end{equation*}
    where $\lesssim_p$ only hides factors depending (potentially exponentially) in $p$.
\end{theorem}

\begin{proof}
    \cref{eq:regret_bound_fixed_p_0} that \cref{eq:main_regret_bound_oblivious} directly holds if the minimum is taken over $N_0\in\{N_{p_0}=2^{P-p_0},p_0\in[P]\}$. Hence, up to a factor of two, the regret bound holds if the minimum is taken over $N_0\in[T]$. We next observe that for $N_0\gtrsim T$ or $N_0\lesssim 1$, the bound exceeds $2T$, hence trivially holds.

    We now turn to the next two claims. Observe that in both cases, if $\sigma\lesssim \frac{1}{T}$, the bound trivially holds. Without loss of generality, we therefore suppose that $\sigma\gtrsim \frac{1}{T}$ from now. 
    
    When $\mathsf{fat}_\Fcal(r)\leq d\ln\frac{1}{r}$ for $r>0$, \cref{thm:fat_shattering_bound} with $\alpha\lesssim \frac{1}{\ln\ln T}$ then implies that for all $\epsilon\in[\frac{1}{T^2},\frac{1}{\sqrt T}]$,
    \begin{equation*}
        \ln(\Ncal(\Fcal;\epsilon,T)+1) \lesssim \mathsf{fat}_\Fcal(c\alpha\epsilon) \ln^{1+\alpha} \frac{T}{\epsilon} \lesssim d\ln^2 T.
    \end{equation*}
    We recall that we assumed $\sigma\gtrsim 1/T$. Similarly, given the target bound, if $d\gtrsim T$ the result is immediate. We also suppose that $d\lesssim T$ from now.
    Next, by \cref{prop:bounding_Will_functional}, we have
    \begin{equation*}
        w(T,\delta)\lesssim d\ln^3\paren{\frac{T}{\sigma}\ln\frac{T}{\delta}} + \ln\frac{T}{\delta}\lesssim d\ln^3\paren{T\ln\frac{1}{\delta}} + \ln\frac{T}{\delta}.
    \end{equation*}
    We then choose the parameter $\epsilon=1/T$ and the value $N_0=\sqrt{\frac{\ln ^3 T}{\sigma\paren{ d\ln^3\paren{T\ln\frac{1}{\delta}} + \ln\frac{T}{\delta}}} \cdot T}$, which gives the desired bound.

    We next turn to the case when $\mathsf{fat}_\Fcal(r)\lesssim r^{-p}$. In the rest of the proof, the symbols $\lesssim_p$ may hide factors in $p$ of the form $c^p$ for universal constants. In this case, \cref{thm:fat_shattering_bound} implies that for $\epsilon\in[\frac{1}{T^2},1]$,
    \begin{equation*}
        \ln(\Ncal(\Fcal;\epsilon,T)+1) \lesssim \frac{\ln^2 T}{\epsilon^p} .
    \end{equation*}
    
    Hence, by \cref{prop:bounding_Will_functional},
    \begin{equation*}
        w(T;\delta)\lesssim_p \paren{\frac{T}{\sigma}}^{\alpha(p)} \ln^3\frac{T}{\delta} ,\quad \text{where}\quad \alpha(p):=\begin{cases}
            \frac{p}{2+p} &0<p\leq 2\\
            1-\frac{1}{p} &p\geq 2.
        \end{cases}
    \end{equation*}
    Again, here we used $\sigma\gtrsim 1/T$.
    For intuition, the two main terms in the regret bound for \cref{eq:main_regret_bound_oblivious} are $\ln^2 T\sqrt{T\ln(\Ncal(\Fcal;\epsilon,T)+1)/\sigma}$ and $T\sqrt{\epsilon \ln T /\sigma}$. To minimize these, we then choose the parameter $\epsilon = \paren{\frac{\ln T}{T}}^{\frac{1}{p+1}}$. With $N_0=\frac{\ln^3 T}{\sqrt \sigma}\cdot T^{\frac{1}{2(p+1)}}$ we obtain the following regret bound,
    \begin{align*}
        \sum_{t=1}^T \ell_t(\hat y_t ) - \sum_{t=1}^T \ell_t(f^\star(x_t) ) &\lesssim_p  \frac{\ln^3 T}{\sqrt \sigma} 
        T^{1-\frac{1}{2(p+1)}}  + \frac{\ln^3 T \ln^{3}\frac{T}{\delta}}{\sigma^{\frac 12+\alpha(p)}} T^{\frac{1}{2(p+1)}+\alpha(p)}    + \frac{\ln^4 T \ln^{\frac{3}{2}}\frac{T}{\delta} }{\sigma^{\frac{\alpha(p)+1}{2}}} T^{\frac 12 + \frac{\alpha(p)}{2}}.
    \end{align*}
    Together with $\alpha(p)\leq 1-\frac{1}{p+1} -\frac{\min(p,1)}{(p+1)(p+2)}$, this implies the desired bound.
\end{proof}

As a remark, for function classes with finite VC dimension $d$, we can use the tighter bound on the Wills functional from \cref{prop:bounding_Will_functional} which gives $\ln W_m(\Fcal)\lesssim d\ln m$. Further, for VC classes, we can simply use $\epsilon=0$ since Sauer-Shela's lemma (\cref{lemma:sauer_lemma}) guarantees $\ln \Ncal(\Fcal;0,T) \lesssim d\ln T$. Altogether, this gives the following slightly improved bound.

\begin{proposition}\label{prop:oblivious_vc_class}
    Let $\Fcal:\Xcal\to \{0,1\}$ be a function class with VC dimension $d$. Suppose that $(x_t)_{t\geq 1}$ is a $\sigma$-smooth sequence on $\Xcal$ with respect to some unknown base measure $\mu$. Then, \textsc{R-Cover} with $\epsilon=0$ makes predictions $\hat y_t$ such that for any $f^\star\in\Fcal$, with probability at least $1-\delta$,
    \begin{equation*}
         \sum_{t=1}^T \ell_t(\hat y_t ) - \sum_{t=1}^T \ell_t(f^\star(x_t) ) \leq C \sqrt{\frac{(d\ln^2 T + d\ln\ln\frac{1}{\delta} + \ln\frac{1}{\delta})\ln^3 T}{\sigma}\cdot T}.
    \end{equation*}
    for some universal constant $C>0$.
\end{proposition}

\subsection{From oblivious to adaptive benchmarks for classification}
\label{subsec:oblivious_to_adaptive}

\cref{thm:oblivious_adversary} gives high-probability bounds for the oblivious regret of \textsc{R-Cover}. In the specific case of classification, we can further extend these bounds to the adaptive regret of the algorithm. In this section, we therefore focus on the case where $\Fcal:\Xcal\to\{0,1\}$ is a function class with finite VC dimension $d$, using ideas inspired from \cite{haghtalab2024smoothed}.

First construct an $\epsilon$-cover $\Hcal$ of the function class $\Fcal$ for the base measure $\mu$. Formally, an $\epsilon$-cover is a subset of $\Fcal$ such that for all $f\in\Fcal$ there exists $h\in\Hcal$ with
\begin{equation*}
    \Pbb_{x\sim\mu}(f(x)\neq h(x))\leq \epsilon.
\end{equation*}
Since $\Fcal$ has VC dimension $d$, we can ensure $\ln |\Hcal| \leq 2d\ln(e^2/\epsilon)$ (see \cite{haussler1995sphere} or \cite[Lemma 13.6]{boucheron2013concentration}). Taking the union bound for all (non-adaptive) benchmark functions in $\Hcal$, \cref{prop:oblivious_vc_class} implies that with probability at least $1-\delta$,
\begin{align}\label{eq:regret_cover}
    \sum_{t=1}^T \ell_t(\hat y_t ) - \inf_{f\in\Hcal} \sum_{t=1}^T \ell_t(f(x_t) ) & \leq C  \sqrt {\frac{(d\ln^2 T + d\ln\ln\frac{|\Hcal|}{\delta} + \ln\frac{|\Hcal|}{\delta}) \ln^3 T}{\sigma} \cdot T} \notag\\
    &\leq C'   \sqrt {\frac{(d\ln^2 T + d\ln\ln\frac{1}{\epsilon\delta} + d\ln\frac{1}{\epsilon} + \ln\frac{1}{\delta} ) \ln^3 T}{\sigma} \cdot T} ,
\end{align}
for some constant $C'\geq 1.$ In the last inequality, we used the fact that without loss of generality, $d\lesssim T$, otherwise the regret bound from \cref{thm:main_thm} is immediate.
Next, for any function $f\in\Fcal$, denote by $h_f\in\Hcal$ a function such that $\Pbb_\mu(f\neq h)\leq \epsilon$. Then,
\begin{align*}
    \sum_{t=1}^T \ell_t(f(x_t) ) \geq \sum_{t=1}^T \ell_t(h_f(x_t) ) - \sum_{t=1}^T \1(f(x_t)\neq h_f(x_t)).
\end{align*}
As a result, denoting by $\Gcal:=\{\1[f\neq h_f],f\in\Fcal\}$, we can decompose the adaptive regret via
\begin{equation}\label{eq:decomposition_adaptive_oblivious}
    \sum_{t=1}^T \ell_t(\hat y_t ) - \inf_{f\in\Fcal} \sum_{t=1}^T \ell_t(f(x_t) ) \leq \sum_{t=1}^T \ell_t(\hat y_t ) - \inf_{f\in\Hcal} \sum_{t=1}^T \ell_t(f(x_t) ) +
    \sup_{g\in\Gcal}\sum_{t=1}^T g(x_t).
\end{equation}
 \cite[Lemma 3.3]{haghtalab2024smoothed} directly bounds the expected value of $\sup_{g\in\Gcal}\sum_{t=1}^Tg(x_t)$. Combined with \cref{eq:regret_cover}, this already gives a bound for the expected adaptive regret. To give useful intuitions and get high-probability bounds, we detail the steps of the proof below. 

Importantly, by construction of the $\epsilon$-cover, we have $\Ebb_{x\sim\mu}[g(x)]\leq \epsilon$ for all $g\in\Gcal$. Also, $\Gcal$ has VC dimension at most $2d$. \cite{haghtalab2024smoothed} then use a coupling argument to reduce to the i.i.d. case for which VC theory yields uniform convergence bounds using \cref{lemma:coupling}.
On the event $\Ecal_k$ from \cref{lemma:coupling}, we have
\begin{equation*}
    \sup_{g\in\Gcal}\sum_{t=1}^Tg(x_t) \leq \sup_{g\in\Gcal}\sum_{t=1}^T\sum_{j=1}^k g(Z_{t,j}).
\end{equation*}
Because the variables $Z_{t,j}$ are i.i.d. and $\Gcal$ has VC dimension at most $2d$, the Vapnik-Chervonenkis inequality \cite[Theorem 2]{vapnik1971uniform} gives Heoffding-type high probability uniform deviation bounds. Recalling that for all $g\in\Gcal$ we have $\Ebb_\mu[g]\leq \epsilon$, we can use relative VC bounds to better control the tail deviations. For instance, \cite[Corollary 2]{cortes2019relative} implies that there is a constant $C$ such that with probability at least $1-\delta$,
\begin{equation*}
    \sup_{g\in\Gcal}\sum_{t=1}^T\sum_{j=1}^k g(Z_{t,j}) \leq \epsilon Tk + C\sqrt{\epsilon Tk\paren{d\ln\frac{Tk}{d} + \ln\frac{1}{\delta}}} + C\ln\frac{Tk}{d} + C\ln\frac{1}{\delta}.
\end{equation*}
We now put the two previous estimate with \cref{eq:regret_cover,eq:decomposition_adaptive_oblivious}, for $k=\frac{1}{\sigma}\ln\frac{T}{\delta}$ and $\epsilon=1/(Tk)$. We note that the bound from \cref{eq:regret_cover} is vacuous if $\frac{1}{\sigma}\ln\frac{T}{\delta}=k\gtrsim T$. Hence, without loss of generality, we can suppose $k\lesssim T$. Similarly, we can suppose that $\ln\frac{T}{\delta}\lesssim \sigma T$. Altogether, this shows that with probability at least $1-2\delta$, we still have
\begin{equation}\label{eq:high_probability_adaptive_bound_vc}
    \sum_{t=1}^T \ell_t(\hat y_t ) - \inf_{f\in\Fcal} \sum_{t=1}^T \ell_t(f(x_t) ) \lesssim  \sqrt {\frac{(d\ln^2 T + d\ln\ln\frac{1}{\delta} + \ln\frac{1}{\delta}) \ln^3 T}{\sigma} \cdot T} .
\end{equation}
This ends the proof of \cref{thm:main_thm}.

\subsection{Proof of \cref{lemma:main_bound_modified}}
\label{subsec:proof_main_lemma}

    Fix $p$ and $r$. To prove the desired bound, we first construct a subsequence $(z_a)_a$ of the process $(x_t)_{t\in[T]}$ that essentially only keeps times for which $\gamma^{(p,r)}(t) \geq q$. For readability, we omit all exponents $(p)$ and $(p,r)$ within this proof from now.

    \paragraph{Construction of the alternative smooth process.}
    Fix a value $q\in[0,1]$ satisfying \cref{eq:assumption_q},
    and fix the parameter $w\geq 1$. 
    We denote by $(\Hcal_t)_t$ the filtration corresponding to the smooth process $(x_t)_t$. We construct a random subsequence $(z_a)_a$ inductively for $k\in[N_p]$.     
    Let $a_0=0$. Suppose that for $k\in[N_p]$ we have constructed a non-decreasing sequence of indices $a_0,\ldots,a_{k-1}$ together with elements $z_1,\ldots,z_{a_{k-1}}$ on $\Xcal$ and values $\gamma_1,\ldots,\gamma_{a_{k-1}}\in[q,1]$ such that all these random variables are all $\Hcal_{T_{k-1}}$-measurable.
    We focus on the epoch $E_k$ and recall the notation $\Pcal_k$ from \cref{eq:def_Pcal_close_functions} for the set of pairs of functions $f,g\in\Fcal$ which had the same values prior $E_k$ up to $\epsilon$, as well as the notation $\gamma(t)$ for $t\in E_k$ from \cref{eq:definition_gama}.
    We then enumerate
    \begin{equation*}
        \{t\in E_k:\gamma(t)\geq   q\}=: \{t_1^{(k)}<\ldots <t_{b_k}^{(k)}\}.
    \end{equation*}
    For convenience, for all $l\in[b_k]$, we denote $\gamma_l^{(k)}:= \gamma(t_l^{(k)})$. We then let
    \begin{equation}\label{eq:def_c_k}
        c_k:= \min\{b_k\}\cup \set{l\in[b_k]: \sum_{l'\leq l}\gamma_{l'}^{(k)} \geq w}.
    \end{equation}
    We next pose $a_k=a_{k-1}+c_k$ and augment the sequences $z_1,\ldots,z_{a_{k-1}}$ and $\gamma_1,\ldots,\gamma_{a_{k-1}}$ as follows
    \begin{equation*}
        \paren{z_{a_{k-1}+l},\gamma_{a_{k-1}+l}} := \paren{x_{t_l^{(k)}}, \gamma_l^{(k)} },\quad l\in[c_k].
    \end{equation*}
    
    This concludes the construction of the sequence on epoch $k$. We can easily check that all these added random variables are $\Hcal_{T_k}$-measurable, which ends the construction of the sequences $(a_k)_{k\in[N_p]}$, $(z_a)_{a\in[a_{N_p}]}$, and $(\gamma_a)_{a\in[a_{N_p}]}$. For convenience, let us denote $A:=a_{N_P}$ the random length of these sequences. Note that all constructed quantities $(\gamma_a)_a$ are at least $q $ by construction. Also, for any $a_{k-1}<a\leq a_k$, since we added the element $z_a=x_{t_{a-a_{k-1}}^{(k)}}$, by definition of $c_k$ in \cref{eq:def_c_k}, we have
    \begin{equation}\label{eq:sum_proba_small_eq0}
        \sum_{s=a_{k-1}+1}^{a-1} \gamma_s <w.
    \end{equation}
    The next step is to bound the sum of the quantities $\gamma_a$ accumulated on this sequence.

    \paragraph{Construction of functions $g_a$ for $a\in[T]$}  
    Importantly, we can check that the stochastic process $z_1,\ldots,z_A$ can be constructed online. More precisely, this is a sub-sequence of the smoothed process $x_1,\ldots,x_T$ and is adapted to the filtration $(\Hcal_t)_t$ in the following sense. Knowing whether to add $x_t$ in the sequence $z_1,\ldots ,z_A$ is $\Hcal_{t-1}$-measurable because this only requires constructing $\gamma(l)$ for all $l\leq t$, which is $\Hcal_{t-1}$-measurable. As a result, $z_1,\ldots,z_A$ is also a $\sigma$-smooth stochastic process for the unknown base measure $\mu$, with the only subtlety being that its horizon is also stochastic. Note that because $(z_a)_{a\in[A]}$ is a subsequence of $(x_t)_{t\in[T]}$, we always have $A\leq T$. For convenience, we complete the sequence $z_1,\ldots,z_{T}$ arbitrarily for $t>A$, for instance with independent samples from $\mu$, as long as the complete process $(z_a)_{a\in[T]}$ remains $\sigma$-smooth with respect to $\mu$.

    For any $a\in[T]$, we define a random function $g_a$ as follows. If $a>A$, we can simply pose $g_a=0$. Note that knowing whether $a\leq A$ can be done in an online process adapted to the filtration $(\Hcal_t)_{t\in[T]}$ with the same ideas presented above.
    Provided $a\leq A$, we denote by $k\in[N_p]$ the index such that $a_{k-1}<a\leq a_k$ and let $l\in[b_k]$ such that we used the time $t_l^{(k)}$ to construct $z_a=x_{t_l^{(k)}}$. We recall that knowing whether we are using $t_l^{(k)}$ to construct $z_a$ is $\Hcal_{t_l^{(k)}-1}$-measurable since we only need to know the past history as well as $\gamma(t_l^{(k)})$. By construction, we had $\gamma(t_l^{(k)})=\gamma_l^{(k)}\geq  q>0$.
    Hence we can fix $f_l^{(k)} \in\Pcal_k$ such that
    \begin{align}
        \Ebb\sqb{\abs{f_l^{(k)}(x_{t_l^{(k)}}) - f^\star(x_{t_l^{(k)}}) }^r\mid\Hcal_{t_l^{(k)}-1} } &\geq (1-\zeta) \sup_{f \in\Pcal_k} \Ebb \sqb{\abs{f(x_{t_l^{(k)}}) -f^\star (x_{t_l^{(k)}}) }^r\mid\Hcal_{t_l^{(k)}-1} }\notag \\
        &= (1-\zeta)\gamma_a, \label{eq:definition_f_k_l_and_h_k_l}
    \end{align}
    for a fixed value $\zeta>0$.
    We then pose $g_a:=|f_l^{(k)}-f^\star|^r$. 

    \paragraph{Upper bound on $\sum_{a=1}^{A} \Ebb\sqb{g_a(z_a) \mid \Hcal_{t(a)-1}}$.}
    By construction, we can ensure that for all $a\in [T]$, provided $a\leq A$, $g_a$ is $\Hcal_{t_l^{(k)}-1}$-measurable, where we used the same notations as above for which $z_a$ was constructed via $z_a=x_{t_l^{(k)}}$. To avoid confusions, we denote $t(a):=t_l^{(k)}$. In particular, $z_a=x_{t(a)}\mid \sigma( z_{a'},a'<a,g_a)$ is still $\sigma$-smooth since $\sigma(z_{a'},a'<a,g_a)\subset \Hcal_{t(a)-1}$. Last, $A$ is a stopping time for the filtration given by the sigma-algebras $\Hcal_{t(a)-1}$. We are now in position to use \cref{lemma:decoupling_stronger} to the rescaled functions $g_a/4$ which implies that for a given sequence $z_1',\ldots,z_T'$ tangent to $z_1,\ldots,z_T$, 
    \begin{align}
        \sum_{a=1}^{A} \Ebb\sqb{g_a(z_a) \mid \Hcal_{t(a)-1}} &\leq  12 \sqrt{\frac{2A\ln(eA)}{\sigma}\paren{\ln(eA) + \frac{1}{4}  \sum_{a=1}^{A} \frac{1}{a}\sum_{s=1}^{a-1} \Ebb\sqb{g_a(z_s')\mid\Hcal_{t(a)-1} }}}
        . \label{eq:regret_variant_smoothed_game}
    \end{align}
    We now fix $a\in [T]$ such that $a\leq A$. Using the same notations as before, let $k\in[N_p]$ such that $a_{k-1}<a\leq a_k$ and $l\in[b_k]$ such that we constructed $z_a$ via $z_a=x_{t_l^{(k)}}$. Importantly,
    \begin{equation}\label{eq:constructed_good_labels}
        g_a(z_s) \leq (2\epsilon)^r,\quad  \forall s\leq a_{k-1}.
    \end{equation}
    Indeed, recall that $g_a=(f_l^{(k)}-f^\star)^2$ where $f_l^{(k)}\in\Pcal_k$. By definition of $\Pcal_k$, $f_l^{(k)}$ and $f^\star$ agree on all queries $x_t$ for $t\in[T_{k-1}]$ up to $\epsilon$ in absolute value.
    Next, let $a_1<a$. Assuming that $a\leq A$, we let $k_1\in[N_p]$ and $l_1\in[b_{k_1}]$ such that we constructed $z_{a_1}:=x_{t_{l_1}^{(k_1)}}$. Note that because $a_1<a$, we have $(k_1,l_1)<_{lex}(k,l)$, where $<_{lex}$ denotes the lexicographical order. Then, we have
    \begin{align}
        \Ebb_{z_{a_1}'}\sqb{g_a(z_{a_1}')\mid \Hcal_{t(a)-1},a\leq A} &=  \Ebb_{x\sim x_{t_{l_1}^{(k_1)}}\mid \Hcal_{t_{l_1}^{(k_1)}-1}}\sqb{g_a(x)}  \notag \\
        &= \Ebb_{x\sim x_{t_{l_1}^{(k_1)}}\mid \Hcal_{t_{l_1}^{(k_1)}-1}}\sqb{\abs {f_{l}^{(k)}(x) - f^\star(x) }^r}   \notag \\
        &\leq  \gamma\paren{t_{l_1}^{(k_1)}}  = \gamma_{a_1} . \label{eq:low_proba_success_tangent}
    \end{align}    
    In the last inequality, we used the definition of the function $\gamma(\cdot)$ from \cref{eq:definition_gama} and the fact that by construction $f_l^{(k)}\in \Pcal_k \subset \Pcal_{k_1}$ (note that $\Pcal_k$ only has more constraints on the functions compared to $\Pcal_{k_1}$) .
    Putting these equations together, we obtain
    \begin{align}
        \sum_{s=1}^{a-1} &\Ebb\sqb{g_a(z_s')\mid\Hcal_{t(a)-1} ,a\leq A} \overset{(i)}{\leq} \sum_{s=a_{k-1}+1}^{a-1}  \gamma_s +  \sum_{s=1}^{a_{k-1}}  \Ebb\sqb{g_a(z_s')\mid\Hcal_{t(a)-1} ,a\leq A} \notag \\
        &\overset{(ii)}{\leq} w + \underbrace{\Ebb_{z_1',\ldots,z_{a_{k-1}}'}\sqb{ \sum_{s=1}^{a_{k-1}}   g_a(z_s')   - 2g_a(z_s) \mid \Hcal_{t(a)-1},a\leq A}}_{E(a)} + (2\epsilon)^r a_{k-1} . \label{eq:upper_bound_sum_ga(z_s')}
    \end{align}
    In $(i)$ we used \cref{eq:low_proba_success_tangent} and in $(ii)$ we used \cref{eq:constructed_good_labels,eq:sum_proba_small_eq0}. Now note that conditionally on $a\leq A$ and $\Hcal_{t(a)-1}$, we have
    \begin{align*}
        \sum_{s=1}^{a_{k-1}}   g_a(z_s')   - 2g_a(z_s)&=\sum_{s=1}^{a_{k-1}} \abs{f_l^{(k)}(z_s') - f^\star(z_s') }^r - 2\abs{f_l^{(k)}(z_s) - f^\star(z_s)}^r  \\
        &\leq \sup_{f\in\Fcal}\sum_{s=1}^{a_{k-1}} |f(z_s') - f^\star(z_s') |^r - 2|f(z_s) - f^\star(z_s)|^r \\
        &\overset{(i)}{\leq} \sup_{\tilde a\leq T}\sup_{f\in\Fcal}\sum_{s=1}^{\tilde a} |f(z_s') - f^\star(z_s') |^r - 2|f(z_s) - f^\star(z_s)|^r.
    \end{align*}
    Note that we perform step $(i)$ because $a_{k-1}$ is not a fixed horizon a priori: it may depend on the elements of the smooth sequence $z_{b}$ for $b>a_{k-1}$ (precisely, the elements $a_{k-1}<b\leq a$).
    We now bound the right-hand side using a high-probability variant of \cite[Theorem 2]{block2024performance}, given in \cref{lemma:block_whp}. Precisely, we apply \cref{lemma:block_whp} to the function class $\Fcal_p := \{\frac{1}{r}|f-f^\star|^r:f\in\Fcal\}$ with the parameter $c=1/2$. Together with the union bound this implies that with probability at least $1-\delta^2$,
    \begin{equation}\label{eq:bound_supremum_will_functional}
        \sup_{\tilde a\leq T}\sup_{f,h\in\Fcal}\sum_{s=1}^{\tilde a} |f(z_s') - h(z_s') |^r - 2|f(z_s) - h(z_s)|^r \leq r C_1 \paren{\ln\Ebb_\mu\sqb{W_{4T\ln(\frac{T}{\delta})/\sigma} \paren{\frac{1}{3}\Fcal_p}}  + \ln\frac{T}{\delta}},
    \end{equation}
    for some universal constant $C_1\geq 1$. Here we used the fact that the Wills functional $W_m(\Fcal_p)$ is non-decreasing in $m$ (e.g. see \cite[Lemma 10]{block2024performance}) and that $a\leq T$. Now note that the function $\psi:z\in [-1,1]\mapsto \frac{1}{r}|x|^r$ is $1$-Lipschitz. Hence \cite[Theorem 4.1]{mourtada2023universal} implies that $W_m(\frac{1}{3}\Fcal_p) \leq W_m(\tilde \Fcal)$, where $\tilde \Fcal=\{f-f^\star:f\in\Fcal\}$. Next, because the Wills functional is invariant under translation from \cite[Proposition 3.1.5]{mourtada2023universal}, we finally obtain
    \begin{equation*}
        W_m\paren{\frac{1}{3}\Fcal_p}\leq W_m(\Fcal),\quad m\geq 1.
    \end{equation*}
    Denote by $\Ecal_{\delta^2}$ the event when \cref{eq:bound_supremum_will_functional} holds.
     Then,
     \begin{equation*}
         E(a) \leq r  C_1 \paren{\ln\Ebb_\mu\sqb{W_{4T\ln(\frac{T}{\delta})/\sigma} \paren{ \Fcal}}  + 2\ln\frac{T}{\delta}} + T \Pbb\paren{\Ecal_{\delta^2}^c\mid \Hcal_{t(a)-1},a\leq A}.
     \end{equation*}
     where the probability on the last term is taken over $z_1',\ldots,z_{T}'$.
     Now by Markov's inequality, we have
     \begin{equation*}
         \Pbb_{z_1,\ldots,z_T}\sqb{\Pbb_{z_1',\ldots,z_T'}\paren{\Ecal_{\delta^2}^c\mid \Hcal_{t(a)-1},a\leq A} \geq \delta} \leq \frac{\Pbb(\Ecal_{\delta^2}^c)}{\delta}\leq \delta,
     \end{equation*}
     
     Denote by $\Fcal_\delta(a)$ the complementary event, which has probability at least $1-\delta$. On this event, the previous bound from \cref{eq:upper_bound_sum_ga(z_s')} implies that
     \begin{equation*}
         \sum_{s=1}^{a-1} \Ebb\sqb{g_a(z_s')\mid\Hcal_{t(a)-1} ,a\leq A} \leq w +r C_1 \paren{\ln\Ebb_\mu\sqb{W_{4T\ln(\frac{T}{\delta})/\sigma} \paren{ \Fcal}}  + \ln\frac{T}{\delta}} +\delta T + (2\epsilon)^r a.
     \end{equation*}
    Plugging this bound into \cref{eq:regret_variant_smoothed_game} shows that on $\bigcap_{a\leq T} \Fcal_\delta(a)$ which has probability at least $1-\delta T$,
    \begin{align}
        &\sum_{a=1}^{A}  \Ebb \sqb{g_a(z_a) \mid \Hcal_{t(a)-1}} \notag\\
        &\leq 12 \left( \frac{2A\ln(eA)}{\sigma} \sqb{\ln(eA)+\paren{r  C_1\ln\Ebb_\mu\sqb{W_{4T\ln (\frac{T}{\delta})/\sigma} \paren{ \Fcal}} + r  C_1\ln \frac{T}{\delta}+w+\delta T} \sum_{a=1}^{T} \frac{1}{4a}} \right. \notag\\
        &\qquad\qquad \qquad\qquad \qquad\qquad\qquad\qquad \qquad\qquad \qquad\qquad\qquad\qquad\qquad\qquad \left. + \frac{(2\epsilon)^r A^2\ln(eA)}{2\sigma}  \right)^{1/2} \notag\\
        &\leq C  \ln  T \sqrt{\frac{rA}{\sigma}\paren{\ln\Ebb_\mu\sqb{W_{4T\ln (\frac{T}{\delta})/\sigma} \paren{\Fcal}} + \ln \frac{T}{\delta}+w+\delta T} } + 6A\sqrt{ (2\epsilon)^r \frac{ 2\ln (eT)}{\sigma}}, \label{eq:upper_bound_sum_ga}
    \end{align}
    for some universal constant $C\geq 1$.
    In the last inequality, we used $A\leq T$ and the inequality $\sqrt{a+b}\leq \sqrt a+\sqrt b$ for all $a,b\geq 0$.
For convenience, we introduce the notation
\begin{equation*}
    C_{w,\delta}(T):= C \ln  T \sqrt{\frac{r}{\sigma}\paren{\ln\Ebb_\mu\sqb{W_{4T\ln (\frac{T}{\delta})/\sigma} \paren{ \Fcal}} + \ln \frac{T}{\delta}+w+\delta T}} .
\end{equation*}

    \paragraph{Lower bound on $\sum_{a=1}^{A} \Ebb\sqb{g_a(z_a) \mid \Hcal_{t(a)-1}}$.}
    We now turn to the lower bound. Note that for any $a\in[T]$ provided that $a\leq A$, using the same notations as above we have
    \begin{equation}\label{eq:lower_bound_proba1}
         \Ebb\sqb{g_a(z_a) \mid \Hcal_{t(a)-1},a\leq A} = \Ebb\sqb{\abs{f_l^{(k)}(x_{t_l^{(k)}}) - h_l^{(k)}(x_{t_l^{(k)}}) }^r\mid \Hcal_{t(a)-1},a\leq A } \overset{(i)}{\geq} (1-\zeta)\gamma_a,
    \end{equation}
    where in $(i)$ we used \cref{eq:definition_f_k_l_and_h_k_l}.
    As a result,
    \begin{align*}
        \sum_{a=1}^{A} \Ebb\sqb{g_a(z_a) \mid \Hcal_{t(a)-1}} 
        \geq(1-\zeta)  \sum_{a=1}^{A} \gamma_a .
    \end{align*}

    \paragraph{Putting the two bounds together.}
    Putting together this lower bound with the upper bound from \cref{eq:upper_bound_sum_ga}, we obtain that with probability at least $1-\delta$
    \begin{align*}
         (1-\zeta)  \sum_{a=1}^{A} \gamma_a &\leq C_{w,\delta/T}(T)\sqrt{ A} + 6A\sqrt{(2\epsilon)^r\frac{2\ln (eT)}{\sigma}}\\
         &\overset{(i)}{\leq} C_{w,\delta/T}(T)\sqrt{\frac{1}{q} \sum_{a=1}^{A} \gamma_a} +\frac{6}{q}\sqrt{(2\epsilon)^r\frac{2 \ln (eT)}{\sigma}} \sum_{a=1}^A \gamma_a \\
         &\overset{(ii)}{\leq} C_{w,\delta/T}(T)\sqrt{\frac{1}{q} \sum_{a=1}^{A} \gamma_a} +\frac{1}{2}\sum_{a=1}^A \gamma_a
    \end{align*}
    where in $(i)$ we recalled that for all $a\leq A$, we have $p_a\geq q$ and in $(ii)$ we used the assumption on $q$ from \cref{eq:assumption_q}. This holds for any $\zeta>0$, which implies that there exists a universal constant $C_1\geq 1$ such that for any $\delta\in(0,1/2]$, with probability at least $1-\delta$,
        \begin{equation*}
            \sum_{a=1}^A \gamma_a \leq \frac{2 C^2_{w,\delta/T}(T)}{q} \leq \frac{C_1 r\ln^2  T}{q \sigma} \paren{\ln\Ebb_\mu\sqb{W_{8T\ln (\frac{T}{\delta})/\sigma} \paren{ \Fcal}} + \ln \frac{T}{\delta}+w}.
        \end{equation*}

    Going back to the construction of the sequence $z_1,\ldots,z_A$, for any epoch $E_k^{(p)}$, in its construction we always try to include as many times $t_1^{(k)},t_2^{(k)},\ldots$ as possible until the threshold $w$ for the sum of their probabilities $\gamma_l^{(k)}$ is passed. Denote by $\Kcal\subset[N_p]$ the set of all epochs $k$ for which not all elements $t_l^{(k)}$ for $l\in[b_k]$ have been used, that is $\Kcal=\{k\in[N_p]:c_k<b_k\}$. On one hand, if $k\notin \Kcal$, \cref{eq:sum_proba_small_eq0} implies that
    \begin{equation*}
        \sum_{t\in E_k}\gamma(t)\1_{\gamma(t)\geq q} = \sum_{l\in[b_k]}\gamma_l^{(k)} \leq w+1 \leq 2w.
    \end{equation*}
    In the last inequality we used $\gamma_{b_k}^{(k)} \leq 1$. On the other hand, for any $k\in\Kcal$,
    \begin{equation*}
        \sum_{a=a_{k-1}+1}^{a_k} \gamma_a >w.
    \end{equation*}
    Therefore, with probability at least $1-\delta$,
    \begin{align*}
        |\Kcal| <\frac{1}{w}\sum_{k\in[N_p]}\sum_{a=a_{k-1}+1}^{a_k} \gamma_a &= \frac{1}{w}\sum_{a=1}^A \gamma_a\\
        &\leq \frac{C_1 r\ln^2  T}{q \sigma w} \paren{\ln\Ebb_\mu\sqb{W_{8T\ln (\frac{T}{\delta})/\sigma} \paren{ \Fcal}} + \ln \frac{T}{\delta}+w}.
    \end{align*}
    Considering $2w$ instead of $w$ and up to changing the constant $C_1$, this ends the proof of the first claim.

    To prove the second claim, let $w\geq 2$ and $q\in(0,1]$ satisfying \cref{eq:stronger_assumption_q}. For any $k\in[N_p]$, we have
    \begin{equation*}
        \Gamma_k =\sum_{t\in E_k}\gamma(t) \leq \frac{q}{2}(T_k-T_{k-1}) + \sum_{t\in E_k}\gamma(t)\1_{\gamma(t)\geq q/2}.
    \end{equation*}
    Applying the bound proved above for $q/2$ together with the fact that $w\leq \frac{q}{2}(L_p-1)\leq \frac{q}{2}(T_k-T_{k-1})$ ends the proof of the second claim.

\acks{ The author is very grateful to Abhishek Shetty and Adam Block for bringing this problem to his attention and for insightful discussions. This work was supported by a Columbia Data Science Institute postdoctoral fellowship.}

\bibliographystyle{alpha}
\bibliography{refs}

\newpage
\appendix

\section{Bounds on the Wills functional}
\label{sec:wills_functional}

We recall the definition of the Wills functional
\begin{equation*}
    W_{m,Z}(\Fcal) := \Ebb_\xi\sqb{\exp\paren{\sup_{f\in\Fcal}\sum_{i=1}^m \xi_i f(Z_i) - \frac{1}{2}f^2(Z_i)}},
\end{equation*}
where $\xi$ is a vector of $m$ i.i.d.\ standard Gaussians. 
A first way to bound the Wills functional is to bound either the Gaussian complexity or the Rademacher complexity of the function class. We recall their definitions below.
\begin{align*}
    \Rcal_{m,Z}(\Fcal) &:=\Ebb_\epsilon\sqb{\sup_{f\in\Fcal} \sum_{i=1}^m \epsilon_i f(Z_i)}\\
    \Gcal_{m,Z}(\Fcal) &:= \Ebb_\xi\sqb{\sup_{f\in\Fcal} \sum_{i=1}^m \xi_i f(Z_i)},
\end{align*}
where $\xi$ is a vector of $m$ i.i.d.\ standard Gaussians and $\epsilon$ is a vector of $m$ i.i.d.\ Rademacher variables. We may omit the dependency in the values $Z=(Z_1,\ldots,Z_m)$ when clear from context. We have the following

\begin{proposition}[Proposition 3.2 of \cite{mourtada2023universal}, Proposition 3 of \cite{block2024performance}, Exercise 5.5 of \cite{wainwright2019high}]
    For any function class $\Fcal$, $m\in\Nbb$, and values $Z_1,\ldots,Z_m\in\Xcal$, we have
    \begin{equation*}
        \ln W_m(\Fcal) \leq \Gcal_m(\Fcal) \lesssim \sqrt{\ln m} \cdot  \Rcal_m(\Fcal).
    \end{equation*}
\end{proposition}

More precisely, \cite{mourtada2023universal} gave a characterization for the Wills functional, in terms of the local Gaussian complexity and covering numbers which we now define. Having fixed $Z_1,\ldots,Z_m$, we introduce the notation $\mu_m=\frac{1}{m}\sum_{i=1}^m\delta_{Z_i}$ for the uniform distribution on the values $Z_1,\ldots,Z_n$ and define the norm $\|f \|_{L_2(\mu_m)} := (\Ebb_{Z\sim\mu_m}|f(Z)|^2)^{1/2}$ for any function $f$. The local Gaussian complexity is defined as follows
\begin{equation*}
    \Gcal_{m,Z}(\Fcal,r):=\sup_{f_0\in\Fcal} \Gcal_{m,Z}( B_r(f_0;\Fcal) ),
\end{equation*}
where $B_r(f_0;\Fcal) = \{f\in\Fcal: \|f-g \|_{L_2(\mu_m)} \leq r \}$ is the ball within $\Fcal$ centered at $f_0$ of radius $r$. Again, we may omit the dependency in $Z$. The covering number $\Ncal_{2,m}(\Fcal,r)$ is defined as the minimal cardinality of an $r$-cover of $\Fcal$ with respect to $\|\cdot\|_{L_2(\mu_m)}$. As a remark, these notations differ from those in \cref{thm:bound_mourtada} by a factor $\sqrt m$ for the scale $r$. This choice of scaling will be easier to work with when computing covering numbers.

\begin{theorem}[Theorem 4.2 of \cite{mourtada2023universal}]\label{thm:bound_mourtada}
    There exist constants $c,C>0$ such that the following holds. For any function class $\Fcal$, $m\in\Nbb$, and values $Z_1,\ldots,Z_m\in\Xcal$, we have
    \begin{equation*}
       c\cdot  \inf_{r>0} \set{\Gcal_m(\Fcal,r) + \ln\Ncal_{2,m}(\Fcal,r)} \leq \ln W_m(\Fcal) \leq C\cdot \inf_{r>0} \set{\Gcal_m(\Fcal,r) + \ln\Ncal_{2,m}(\Fcal,r)}.
    \end{equation*}
\end{theorem}

In particular, we obtain the following bounds for classical behaviors of function classes.

\begin{proposition}\label{prop:bounding_Will_functional}
    Fix any values $Z_1,\ldots,Z_m\in\Xcal$. If $\Fcal$ is finite, then $\ln W_m(\Fcal) \lesssim \ln|\Fcal|$.
    If $\Fcal$ has finite VC dimension $d$, then $\ln W_m(\Fcal) \lesssim d\ln m$.

    More generally, for any $r>0$,
    \begin{equation}\label{eq:bound_local_gaussian_complexity}
        \Gcal_m(\Fcal,r) \lesssim \inf_{0\leq r'\leq r} \set{ r' m + \sqrt m\cdot \int_{r'}^r \sqrt{\mathsf{fat}_\Fcal(\epsilon)} \ln\frac{16\cdot \mathsf{fat}_\Fcal(\epsilon)}{\epsilon} d\epsilon  }
    \end{equation}

    In particular, if there exists $d\geq 1$ such that for all $r>0$, one has $\mathsf{fat}_\Fcal(r) \leq d\ln\frac{1}{r}$, we have
    \begin{equation*}
        \ln W_m(\Fcal) \lesssim d\ln^3 (dm).
    \end{equation*}
    
    In particular, if there exists some $\gamma>1$ such that for any $r>0$, $\mathsf{fat}_\Fcal(r) \leq \gamma r^{-p}$, for all $r>0$,
    \begin{equation*}
        \ln W_m(\Fcal) \lesssim_p \begin{cases}
            \gamma^{\frac{2}{2+p}} m^{\frac{p}{2+p}} \cdot \ln^{\frac{4}{2+p}}(\gamma m)& 0<p<2\\
            \sqrt{\gamma m}\cdot \ln^2(\gamma m) +\gamma \ln^2 \gamma &p=2\\
            \gamma^{\frac{1}{p}}m^{1-\frac{1}{p}} \cdot \ln^{\frac{2}{p}}(\gamma m)  +\gamma \ln^2 \gamma &p>2.
        \end{cases}
    \end{equation*}
    where $\lesssim_p$ only hides factors that depend (possibly exponentially) only on $p$.
    These bounds can be simplified as follows
    \begin{equation*}
        \ln W_m(\Fcal) \lesssim_{\gamma,p} m^{\alpha(p)} \ln^2 m,\quad \text{where}\quad \alpha(p) := :=\begin{cases}
            \frac{p}{2+p} &0<p\leq 2\\
            1-\frac{1}{p} &p\geq 2,
        \end{cases}
    \end{equation*}
    where $\lesssim_{p,\gamma}$ hides factors and additive terms depending on $p,\gamma$ only.
\end{proposition}

\begin{proof}
    For function classes $\Fcal$ with finite VC dimension $d$, Sauer-Shelah's \cref{lemma:sauer_lemma} implies that $\ln \Ncal_2(\Fcal,r)\lesssim d\ln m$ for any $r\in[0,1]$, which directly implies that $\ln W_m(\Fcal) \lesssim d\ln m$. Similarly, for any finite class $\Fcal$, we obtain $\ln W_m(\Fcal) \lesssim \ln|\Fcal|$.

    Next, from \cite[Theorem 3.2]{mendelson2002rademacher}, we have for any $r>0$,
    \begin{equation}\label{eq:covering_number_bounds}
        \ln \Ncal_{2,m}(\Fcal,r) \lesssim \mathsf{fat}_\Fcal(r/8) \ln^2 \paren{\frac{2\mathsf{fat}_\Fcal(r/8)}{r}}.
    \end{equation}
    We can combine this estimate with the chaining bounds for Gaussian complexities from  \cite[Lemma 3.7]{mendelson2002rademacher} together with the fact $\Ncal_{2,m}(B_r(f_0;\Fcal),r')=1$ for all $r'>r$ and $f_0\in\Fcal$, which implies the desired bound on the local Gaussian complexity \cref{eq:bound_local_gaussian_complexity}.

    Suppose that we have $\mathsf{fat}_\Fcal(r) \leq d\ln\frac{1}{r}$ for all $r>0$. Then, we can choose $r=1/m$ in \cref{thm:bound_mourtada} and $r'=r$ in \cref{eq:bound_local_gaussian_complexity} which gives the desired result.
    
    Now suppose that $\mathsf{fat}_\Fcal(r) \leq \gamma r^{-p}$ for all $r>0$ for some $\gamma> 1$ and $p>0$. Then, for $r\in(0,1]$, \cref{eq:bound_local_gaussian_complexity} yields
    \begin{equation*}
        \Gcal_m(\Fcal,r) \lesssim_p \begin{cases}
             \sqrt{\gamma m} \cdot r^{1-\frac{p}{2}} \ln\frac{8\gamma}{r} & 0<p<2\\
             \min\set{ rm, \sqrt{\gamma m}\cdot \ln m \cdot  \ln(\gamma m) } &p=2\\
             \min\set{rm, \gamma^{\frac{1}{p}}m^{1-\frac{1}{p}} \cdot \ln^{\frac{2}{p}}(\gamma m)  }&p>2.
        \end{cases}
    \end{equation*}
    This can be obtained directly from \cref{eq:bound_local_gaussian_complexity} by plugging in the value $r'=0$ for $0<p<2$. For $p=2$, we take $r' = \min\set{r,\sqrt{\gamma/m}}$. Last, for $p>2$, we take $r' = \min\set{r,(\gamma \ln^2(\gamma m)/m)^{\frac{1}{p}} }$. 

    We then use \cref{thm:bound_mourtada} together with \cref{eq:covering_number_bounds} and the previous estimates on the local Gaussian complexity to obtain the desired bound on the Wills functional $W_m(\Fcal)$. For $0<p<2$, we used the value $r=\paren{\gamma\ln^2(\gamma m) /m}^{\frac{1}{p+2}}$. For $p\geq 2$, we used the value $r=1$.
\end{proof}

\section{Learning with expert advice guarantee for \textsc{A-Exp}}
\label{sec:proof_learning_expert}

In this section, we prove \cref{lemma:regret_exponentially_weighted}.
    Note that \textsc{A-Exp} proceeds by periods $k\geq 1$. Let $T_0=0$ and denote by $T_k$ the end of period $k$ for $k\geq 1$. That is,
    \begin{equation*}
        T_k = \min\set{t> T_{k-1}: \sum_{l=T_{k-1}+1}^t    \sum_{i\in[K]}p_{l,i}r_{l,i}^2 > \Delta_{max,k} = 4^{k-1} },\quad k\geq 1
    \end{equation*}

    On period $(T_{k-1},T_k]$, \textsc{A-Exp} exactly implements the exponential weights forecaster with parameter $\eta_k=\sqrt{2\ln K/(\Delta_{max,k}+1)}$. Hence, we can use \cref{eq:base_regret_bound} to bound the regret accumulated on this period which gives for all $T\in(T_{k-1},T_k]$,
    \begin{align*}
        \sum_{t=T_{k-1}+1}^{T} \Ebb_{\hat i_t}[\ell_{t,\hat i_t} \mid\Hcal_t] -\min_{i\in[K]}\sum_{t=T_{k-1}+1}^{T} \ell_{t,i} &\leq \frac{\ln K}{\eta_k} + \frac{\eta_k}{2} \sum_{t=T_{k-1}+1}^{T}    \sum_{i\in[K]}p_{t,i}r_{t,i}^2\\
        &\overset{(i)}{\leq} \frac{\ln K}{\eta_k} + \frac{\eta_k}{2} (\Delta_{max,k}+1)\\
        &= \sqrt{2(\Delta_{max,k}+1)\ln K  } =  \sqrt{2(4^{k-1}+1)\ln K  } \leq 2^k\sqrt{\ln K},
    \end{align*}
    where in $(i)$ we used the fact that $r_{T_k,i}\in[0,1]$ for all $i\in[K]$. Now for $T\geq 1$ denote by $k$ the last period, such that $T\in(T_{k-1},T_k]$. Provided $k\geq 2$, we can sum the previous equations for periods $k'\leq k$ to obtain
    \begin{align*}
        \textnormal{PReg}(T) \leq \sum_{k'\leq k} 2^k\sqrt{ \ln K } \leq 2^{k+1}\sqrt{\ln K} \overset{(ii)}{\leq} 8 \sqrt{\Delta_T \ln K}. 
    \end{align*}
    In $(ii)$ we used the fact that if $k\geq 2$ then
    \begin{equation*}
        \Delta_T \geq \sum_{t=T_{k-2}+1}^{T_{k-1}} \sum_{i\in[K]}p_{t,i}r_{t,i}^2 > \Delta_{max,k-1}=4^{k-2}.
    \end{equation*}
    If $k=1$, we have directly
    $\textnormal{PReg}(T) \leq 2\sqrt{\ln K}$. This ends the proof for the bound on the pseudo-regret.

    To obtain high-probability bounds on the regret $\textnormal{Reg}(T)$, we could simply use Azuma-Hoeffding's inequality. However, this would add an additional term $\sqrt {T\ln \frac{1}{\delta}}$ that is prohibitive for our bounds: potentially we have $\Delta_T\ll T$. Instead, we use Freedman's inequality which gives a more precise control on tail probabilities for martingales.  \cref{lemma:freedman_inequality} applied with $Z_t=r_{t,\hat i_t}^2 - \Ebb_{\hat i_t}[r_{t,\hat i_t}^2\mid\Hcal_t]$ for $t\in[T]$ and $\eta=1/2$ implies that with probability at least $1-\delta$,
    \begin{align*}
        \textnormal{Reg}(T) \overset{(i)}{\leq} \textnormal{PReg}(T) + \frac{1}{2}\sum_{t=1}^T \Ebb[r_{t,\hat i_t}^4\mid\Hcal_t] + 2\ln \frac{1}{\delta} \overset{(ii)}{\leq} \frac{3}{2} \textnormal{PReg}(T) + 2\ln\frac{1}{\delta},
    \end{align*}
    where in $(i)$ we used the fact that $Var(Y)\leq \Ebb[Y^2]$ for any random variable $Y$ and in $(ii)$ we used the fact that $|r_{t,i}|\leq 1$ for all $i\in[K]$ and $t\in[T]$. This ends the proof.

\section{Proof of the regret lower bound}
\label{sec:proof_lower_bound}

In this section, we prove that the regret bound from \cref{thm:main_thm} is tight up to logarithmic terms. We recall the statement of the lower bound below.

\LowerBound*

\begin{proof}
    The template function class that we use are simply the threshold functions on $[0,1]\mapsto \{0,1\}$, which have VC dimension one. To extend this to a function class with VC dimension $d$, we take $d$ copies. That is, we pose $\Xcal=\{1,\ldots,d\}\times[0,1] = [d]\times [0,1]$ and we let
    \begin{equation*}
        \Fcal:= \set{ f_{\mb\theta}:(k,x)\in \Xcal \mapsto \1[x\geq \theta_k],\; \mb \theta \in[0,1]^d }.
    \end{equation*}
    For convenience, we define $\bar x = (1,0)$, where the value $1$ was chosen arbitrarily, we also let $\Xcal_k=\{k\}\times [0,1]$. By definition, we have $\Xcal=\Xcal_1\sqcup\ldots\sqcup \Xcal_d$.
    
    Now fix a horizon $T\geq 1$ and $\sigma\in(0,1)$. Suppose for now that
    \begin{equation}\label{eq:assumption_T}
        T > \frac{4d(1-\sigma)}{\sigma}.
    \end{equation}
    We now fix a parameter $q = \sqrt{\frac{d(1-\sigma)}{\sigma T}} $ and let $N=\floor{qT/d}$. 
    Note that from the assumption on $T$, we have $q< 1/2$.
    Next, suppose that $N\leq 1$. Then, this corresponds to $q\leq 2d/T$. Classical lower bounds for VC classes show that we can construct a distribution $\mu$ (uniform on $d$ shattered points), which corresponds to $\sigma=1$ together with a function $f^\star\in\Fcal$ such that with $x_t\overset{iid}{\sim} \mu$, the expected number of mistakes of any algorithm is at least $\min(d,T)/4$.
    Now because $qT\leq 2d$, this directly implies the desired result when $N\leq 1$. 
    
    From now, we suppose that $N\geq 2$.
    Let $\mb \epsilon = (\epsilon_{k,t})_{k\in[d],t\in[N]}$ be a sequence of i.i.d.\ uniform variables on $\{0,1\}$. We now construct a generating process for the sequence $(x_t,y_t)_{t\in[T]}$ coupled with $\mb\epsilon$. In addition to the variables $(x_t,y_t)_{t\in[T]}$, the process also iteratively constructs variables $a_{k,t}<b_{k,t}$ for $k\in[d]$ and $t\in[T]$. For each $k\in[d]$, the interval $\{k\} \times (a_{k,t},b_{k,t})$ will intuitively represent the region of $\Xcal_k$ on which the learner does not have information yet at the beginning of round $t$. 
    
    We initialize the process at time $t=0$ by setting $a_{k,1}=0$ and $b_{k,1}=1$ for all $k\in[d]$. We also initialize index variables $i(k,1)=1$ for all $k\in[d]$. Suppose that we have constructed $a_{k,t},b_{k,t}$ for $k\in[d]$ at some iteration $t\in[T]$, as well as the indices $i(k,t)$ for $k\in[d]$. We then define the distribution
    \begin{equation}\label{eq:def_mu_t}
        \mu_t := (1-q) \delta_{\bar x} + \sum_{k=1}^d \frac{q}{d} \paren{ \delta_{ (k,(a_{k,t}+b_{k,t})/2 )} \1_{i(k,t)\leq N} + \delta_{\bar x} \1_{i(k,t)>N} },
    \end{equation}
    where $\delta_z$ denotes the Dirac distribution at $z$, and $q\in(0,1)$ is a fixed probability value. We then sample $x_t\sim\mu_t$ independently from $\mb\epsilon$ and let
    \begin{equation*}
        y_t := \begin{cases}
            0 & \text{if } x_t=\bar x\\
            \epsilon_{i(k,t)} & x_t\in\Xcal_k.
        \end{cases}
    \end{equation*}
    We then pose for all $k\in[d]$,
    \begin{equation*}
        (a_{k,t+1},b_{k,t+1}) :=\begin{cases}
            (a_{k,t},b_{k,t}) &\text{if } x_t=\bar x \text{ or } x_t \notin \Xcal_k\\
            (a_{k,t}, (a_{k,t}+b_{k,t})/2 )&\text{if } x_t=(k,(a_{k,t}+b_{k,t})/2 ) \text{ and }\epsilon_{i(k,t)}=1 \\
            ( (a_{k,t}+b_{k,t})/2, b_{k,t} )&\text{if } x_t=(k,(a_{k,t}+b_{k,t})/2 ) \text{ and }\epsilon_{i(k,t)}=0.
        \end{cases}
    \end{equation*}
    Last, we pose for all $k\in[d]$,
    \begin{equation*}
        i(k,t+1) :=\begin{cases}
            i(k,t) & \text{if } x_t=\bar x \text{ or } x_t\notin \Xcal_k\\
            i(k,t)+1 &\text{otherwise}. 
        \end{cases}
    \end{equation*}

    This concludes the construction of the process $(x_t,y_t)_{t\in[T]}$. Note that by construction, whenever $x_t\neq \bar x_t$, a fresh random variable from $\mb\epsilon$ is used to define $y_t$. In particular, we can check that conditionally on the history up to time $t$, we have $y_t=0$ if $x_t=\bar x$ and $y_t\sim\Ucal(\{0,1\})$ otherwise. In particular, we always have
    \begin{align*}
        \Ebb\sqb{ \sum_{t=1}^T \1_{\hat y_t\neq y_t}} 
        &= \Ebb\sqb{\sum_{t=1}^T \Ebb\sqb{ \1_{\hat y_t\neq y_t} \mid \hat y_t,\; (x_l,y_l)_{l\leq t-1} }} \\
        &\geq \frac{1}{2}\Ebb\sqb{\sum_{t=1}^T  \Ebb\sqb{ \1_{x_t\neq\bar x} \mid  (x_l,y_l)_{l\leq t-1} }}\\
        &= \frac{1}{2}\Ebb\sqb{\sum_{t=1}^T  \frac{q}{d}\sum_{k\in[d]} \1_{i(k,t)\leq N}  }\\
        &\overset{(i)}{=}\frac{q}{2}\Ebb\sqb{\sum_{t=1}^T  \1_{i(1,t)\leq N}  } = \frac{q}{2}\Ebb_{Z\sim \text{NB}(N,q/d)}\sqb{\max(N+Z,T)  } ,
    \end{align*}
    where $\text{NB}(r,p)$ denotes the negative binomial distribution with $r$ successes and probability of success $p$. Indeed, $i(1,t)$ grows when $x_t=(1,(a_{1,t}+b_{1,t})/2)$, which has probability $q/d$ conditionally on the history.
    In $(i)$ we use the fact that all coordinates are treated symmetrically.
    From now let $Z$ be a random variable distributed according to $\text{NB}(N,q/d)$. From \cite{van1993bounds} since $\Ebb[Z] = \frac{N(1-q/d)}{q/d} > N$, letting $\eta$ be the median of $Z$, we have
    \begin{equation*}
        T-N \geq \Ebb[Z] \geq\eta \geq 1+ \frac{N-1}{N}\Ebb[Z] = 2+ \frac{(N-1)d}{q}-N.
    \end{equation*}
    As a result, we have
    \begin{equation*}
        \Ebb\sqb{ \sum_{t=1}^T \1_{\hat y_t\neq y_t}} \geq  \frac{q(N+\eta-1)}{4} \geq \frac{(N-1)d }{4}.
    \end{equation*}
    \comment{
    In summary, we obtained
    \begin{equation*}
        \Ebb\sqb{ \sum_{t=1}^T \1_{\hat y_t\neq y_t}} \geq \frac{qT + q-d}{4}.
    \end{equation*}
    Now note that the event $i(1,T)>N$ corresponds to the case where there were at least $N$ times for which $x_t=(1,(a_{1,t}+b_{1,t})/2)$, which has probability $q/d$ conditionally on the history. Hence, since $N\geq 2qT/d$, Bernstein's inequality implies that for any $k\in[d]$,
    \begin{align*}
        \Pbb\paren{i(k,T) > N} &\leq \exp\paren{-\frac{\frac{1}{2}(N-qT/d)^2}{qT/d + (N-qT/d)/3}} \\
        &\leq \exp\paren{-\frac{3qT}{8d}} \\
        &= \exp\paren{-\frac{1}{8}\sqrt{\frac{T(1-\sigma)}{d\sigma}}} \leq \frac{1}{2d},
    \end{align*}
    where in the last inequality we used \cref{eq:assumption_T}.
    In summary, we have
    \begin{equation*}
        \Ebb\sqb{ \sum_{t=1}^T \1_{\hat y_t\neq y_t}} \geq \frac{q T}{2} \paren{1-\sum_{k=1}^d \Pbb\paren{i(k,T) > N} } \geq \frac{qT}{4}
    \end{equation*}
    }
    By the law of total probabilities, there is a realization of $\mb\epsilon$ which we denote $\tilde{ \mb\epsilon}$ such that
    \begin{equation*}
        \Ebb\sqb{ \sum_{t=1}^T \1_{\hat y_t\neq y_t}  \mid \mb\epsilon = \tilde{\mb\epsilon}} \geq \frac{(N-1)d}{4}.
    \end{equation*}
    By construction of the process, to each realization of $\mb\epsilon$ is associated a function in class $f_{\mb\theta(\mb\epsilon)}\in\Fcal$ that realizes all the values $(x_t,y_t)$. Indeed, we can take for instance
    \begin{equation*}
        \theta(\mb\epsilon)_i := \frac{1}{2^{T+1}} + \sum_{t=1}^T \frac{1-\epsilon_{k,t}}{2^k}.
    \end{equation*}
    Indeed, the main point is that defining the intervals $[a_{k,t},b_{k,t}]$ for $t\in[T]$, only used the variables $\epsilon_{k,t}$ for $t\in[T]$. These are only updated when we sample $x_t=(k,(a_{k,t} + b_{k,t})/2)$ in which case we use the first value within $\{\tilde \epsilon_{k,1},\ldots,\tilde \epsilon_{k,T}\}$ that was not used up to this point. In particular, this implies that the number of possible values that the sequence $(x_t)_{t\in[T]}$ can take is at most $1+dN$ where the term $1$ corresponds to the value $\bar x$. For convenience, let $\nu$ denote the uniform distribution on these $dN$ points where we deleted the value $\bar x$.

    It now remains to argue that the sequence $(x_t)_{t\in[T]}$ constructed with $\tilde{\mb\epsilon}$ is $\sigma$-smooth. To do so, we construct the measure
    \begin{equation*}
        \mu := \sigma \delta_{\bar x} + (1-\sigma) \nu.
    \end{equation*}
    Importantly, this distribution is fixed \emph{a priori} (it does not depend on the actions of the learner, only on $\tilde{\mb\epsilon}$ that is fixed).
    Given its definition in \cref{eq:def_mu_t}, to check that at any time $t\in[T]$, the distribution $\mu_t$ is $\sigma$-smooth compared to the base measure $\mu$, it suffices to check that
    \begin{equation*}
        \frac{q/d}{(1-\sigma)/(dN)}  = \frac{qN}{1-\sigma} \leq\frac{q^2 T }{d(1-\sigma)} \leq \frac{1}{\sigma}.
    \end{equation*}
    In the last inequality we used the definition of $q$.
    As a summary, the sequence $(x_t)_{t\in[T]}$ is $\sigma$-smooth compared to $\mu$ and using the realizable values $y_t=f_{\mb\theta(\tilde{\mb\epsilon})}(x_t)$, we obtained
    \begin{equation*}
        \Ebb\sqb{ \sum_{t=1}^T \1_{\hat y_t\neq y_t}} \geq \frac{(N-1)d}{4} \geq \frac{qT}{12} = \frac{1}{12}\sqrt{\frac{dT(1-\sigma)}{\sigma}}.
    \end{equation*}
    In the second inequality we used the assumption that $N\geq 2$ to show that $N-1\geq qT/3d$.

    We now consider the case when \cref{eq:assumption_T} is not necessarily satisfied. Then, with $T_0 = \ceil{\frac{4d(1-\sigma)}{\sigma}}$, the previous arguments imply that for some realizable data and a $\sigma$-smooth adversary, we have
    \begin{equation*}
        \Ebb\sqb{ \sum_{t=1}^{T_0} \1_{\hat y_t\neq y_t}} \geq \frac{1}{12}\sqrt{\frac{dT_0(1-\sigma)}{\sigma}} \geq \frac{T_0}{24}.
    \end{equation*}
    As a result, considering the interval of time that incurred the most regret, this shows that for all $T\leq T_0$, there is a $\sigma$-smooth realizable adversary under which the expected number of mistakes for any learning algorithm is at least $T/24$.
    This ends the proof.
\end{proof}

\section{Proofs from \cref{sec:technical_overview}}
\label{sec:simpler_proofs}

In this section, we prove the results related to the simplified algorithm \textsc{Cover}. 
We start by giving a detailed proof sketch for \cref{prop:simplified}. 

\paragraph{Proof sketch of \cref{prop:simplified}.}
Fix a fixed threshold $q$. For convenience, we denote by $E_k=(T_{k-1},T_k]$ the $k$-th epoch for $k\in[K]$. We also fix a value $w\asymp d\ln\frac{T}{\sigma}$ which will play the role of $w(T,\delta)$. Given the desired result, we will only focus on times $t\in E_k$ for $k\in[K]$, for which $\gamma_{T_{k-1}}(t) \geq q$. For simplicity, we now assume that we have discarded all other times, so that for all $t\in E_k$ one has $\gamma_{T_{k-1}}(t)\geq q$. Further, note that the quantity of interest for \cref{prop:simplified} only involves epochs $k\in[K]$ for which $\sum_{t\in E_k} \gamma_{T_{k-1}}(t) \geq w$. Therefore, for each epoch $E_k$ such that $\sum_{t\in E_k} \gamma_{T_{k-1}}(t) \geq w$, we also discard all times $t\in E_k$ with $t>t_{max,k}$, where $t_{max,k}\in E_k$ is the first time for which $\sum_{t\in E_k,t\leq t_{max,k}}\gamma_{T_{k-1}}(t) \geq w$.
Up to renaming the remaining times in the sequence, we now consider that the epochs containing remaining times correspond to $E_k=(t_{k-1},t_k]$ for $k\in[K]$ where $t_0\leq t_1\leq \cdots\leq t_K$. The remaining sequence therefore has horizon $t_K\leq T$ (within the formal proof this is denoted by $A$). We can in fact check that these times can be deleted online, and as a result, the corresponding subsequence of $(x_t)_{t\in[T]}$ is still $\sigma$-smooth.

Clearly, we have
\begin{equation}\label{eq:initial_bound_sketch}
    \abs{\set{k\in[K]: \sum_{t\in E_k}\gamma_t \geq w}} \leq \frac{1}{w} \sum_{t=1}^{t_K} \gamma_t ,
\end{equation}
where we used the shorthand $\gamma_t$ to denote $\gamma_{T_{k-1}}(t)$ where $t\in E_k$. Therefore, in the rest of the proof, we aim to bound the right-hand side $\sum_{t=1}^{t_K} \gamma_t$.

To each time $t\in[t_K]$ we associate a function $g_t$ as follows. By construction, we recall that $\gamma_t\geq q$. Hence, by definition of $\gamma_t$, we can fix $f_t,h_t\in\Fcal$ such that
\begin{equation*}
    \Pbb_{x\sim \mu_t}(f_t(x)\neq h_t(x)) \geq \frac{\gamma_t}{2},
\end{equation*}
and such that $f_t(x_s)=h_t(x_s)$ for all times $s\in[t_{k-1}]$ from previous epochs. We then pose $g_t:x\mapsto \1[f_t(x)\neq h_t(x)]$. In particular,
\begin{equation*}
    \sum_{t=1}^{t_K} \Ebb[g_t(x_t)\mid\Hcal_t] \geq \frac{1}{2}\sum_{t=1}^{t_K} \gamma_t.
\end{equation*}
We next upper bound on the right-hand side. To do so, we use the decoupling result \cref{lemma:decoupling_stronger} (akin to \cite[Lemma 3]{block2024performance}), which gives
\begin{equation}\label{eq:upper_lower_bound_sketch}
    \frac{1}{2}\sum_{t=1}^{t_K} \gamma_t \leq \sum_{t=1}^{t_K} \Ebb[g_t(x_t)\mid\Hcal_t,g_t] \lesssim \sqrt{\frac{t_K\ln t_K}{\sigma} \paren{\ln t_K + \sum_{t=1}^{t_K} \frac{1}{t} \sum_{s=1}^{t-1} \Ebb[g_t(x_s')\mid\Hcal_t,g_t]  }},
\end{equation}
where $(x_t')_{t\leq t_K}$ is a tangent sequence to $(x_t)_{t\leq t_K}$. We next fix $t\in E_k$ and bound\begin{align*}
    \sum_{s=1}^{t-1} \Ebb[g_t(x_s')\mid\Hcal_t,g_t] \leq \sum_{s=1}^{t_{k-1}} \Ebb[g_t(x_s') \mid\Hcal_t,g_t] + \sum_{s=t_{k-1}+1}^{t-1} \gamma_t \leq \Ebb_{x_1',\ldots,x_{t_K}'}\sqb{\sum_{s=1}^{t_{k-1}} g_t(x_s') - 2 g_t(x_s) } + w.
\end{align*}
In the last inequality, we used the fact that by construction, $t_k$ is the first index for which $\sum_{t\in E_k,t\leq T_k}\gamma_t \geq w$, as well as the fact that for all $s\leq t_{k-1}$, $g_t(x_s)=f_t(x_s)-h_t(x_s)=0$. An analog term to the right-hand side expectation was bounded in \cite[Theorem 2]{block2024performance} in terms of the Wills functional of the function class $\Fcal$. For our purposes, we need to use a high-probability variant \cref{lemma:block_whp}. After computations, we obtain with high probability,
\begin{equation*}
    \sum_{s=1}^{t-1} \Ebb[g_t(x_s')\mid\Hcal_t,g_t] \lesssim \ln\Ebb_\mu\sqb{W_{4T\ln T/\sigma} \paren{ \Fcal}}  + \ln T +w \lesssim d\ln\frac{T}{\sigma} + w \lesssim w.
\end{equation*}
We plug this into \cref{eq:upper_lower_bound_sketch} and using the fact that $\ln t_K\leq \ln T$ and $t_K \leq \frac{1}{q}\sum_{t=1}^{t_K}\gamma_t$ by construction, which gives
\begin{equation*}
    \sum_{t=1}^{t_K}\gamma_t \lesssim \frac{\ln^2 T}{q\sigma} \cdot w.
\end{equation*}
Together with \cref{eq:initial_bound_sketch} this proves the desired bound.

\vspace{2mm}

We now formally prove \cref{prop:simplified}. As a remark, how the epochs are constructed is very flexible: we considered in \cref{prop:simplified} the fixed schedule of epochs for \textsc{Cover} ($(T_k)_{k\in[K]}$) but randomized epochs are also possible, which may be useful for improved regret bounds in the regression case. The proof is essentially a simplification from its counterpart \cref{lemma:main_bound_modified}, hence we will only highlight the main differences. We start by proving \cref{prop:simplified}. In this result, how the epochs are constructed is very flexible: \cref{prop:simplified} considered fixed schedule for \textsc{Cover} but randomized epochs are also possible, which may be useful for improved regret bounds in the regression case. We prove the corresponding generalization below.

\begin{proposition}\label{prop:simplified_generalized}
    Let $T\geq 2$ and $\Fcal:\Xcal\to\{0, 1\}$ be a function class with VC dimension $d$. 
    
    Consider any online mechanism to construct epochs $(T_{k-1},T_k]$ for $k\in[K]$. That is, let $(T_k)_{k\geq 0}$ be random times such that (1) $T_0=0$, (2) for all $k\geq 1$, $T_k\mid \{T_{k-1},T_{k-1}<T\}$ is a stopping time adapted to the filtration $(\Hcal_t)_{t\geq T_{k-1}}$, and (3) for all $k\geq 1$ almost surely, $T_{k-1}<T_k\leq T$ conditionally on $T_{k-1}<T$. Let $K\leq T$ denote the first index such that $T_K=T$.

    Fix any parameters $q,\delta\in (0,1]$ and denote $w(T,\delta):=d\ln\paren{\frac{T}{\sigma}\ln\frac{1}{\delta}} + \ln\frac{T}{\delta}+2$. Then, with probability at least $1-\delta$,
    \begin{equation*}
        \abs{ \set{k\in[K]: \sum_{t=T_{k-1}+1}^{T_k} \gamma_{T_{k-1}}(t) \cdot \1[\gamma_{T_{k-1}}(t) \geq q] \geq w(T,\delta)  }} \leq C \frac{\ln^2 T}{q\sigma},
    \end{equation*}
    for some universal constant $C\geq 1$. For a bound in expectation we can simply take $w(T):=d\ln\frac{T}{\sigma} +2$.
\end{proposition}

\begin{proof}[of \cref{prop:simplified_generalized}]
    \cref{lemma:main_bound_modified} essentially proves this result. The main difference is that in \cref{lemma:main_bound_modified} the proof was adapted to the specific schedule of the depths-$p$ epochs $(T_{k-1}^{(p)},T_k^{(p)}]$ for $k\in[N_p]$ for some $p\in[P]$. Within \cref{prop:simplified}, because the epochs are also constructed online, we can replicate the same proof arguments with the online epochs $(T_{k-1},T_k]$ for $k\in[K]$.

    Fix $w\geq 2$. We construct the equivalent alternative smooth process $(z_a)_a$ together with probabilities $(\gamma_a)_a$ as follows. On each epoch $k\in[K]$, we enumerate
    \begin{equation*}
        \set{t\in(T_{k-1},T_k]:\gamma_{T_{k-1}}(t) \geq q} =:\{t_1^{(k)}<\ldots<t_{b_k}^{(k)} \}.
    \end{equation*}
    Using the same notations as in the proof of \cref{lemma:main_bound_modified}, we denote $\gamma_l^{(k)}:=\gamma_{T_{k-1}}(t_l^{(k)})$ for all $l\in[b_k]$. From now the construction of the alterative smooth process is identical. The length of the sequence is now $A=a_{K}$.

    We now construct the functions $g_a$ for $a\in[T]$. As in the original proof we let $g_a=0$ for $a>A$. For $a\leq A$, letting $t_l^{(k)}$ be the time used to construct $z_a=x_{t_l^{(k)}}$, we let $f_l^{(k)},h_l^{(k)}$ such that
    \begin{equation*}
        \Pbb\paren{ f_l^{(k)}(x_{t_l^{(k)}}) \neq h_l^{(k)}(x_{t_l^{(k)}}) \mid \Hcal_{t_l^{(k)}-1}} \geq (1-\zeta)\gamma_a,
    \end{equation*}
    for some fixed value $\zeta>0$ then pose $g_a=\1[f_l^{(k)}\neq h_l^{(k)}]$. Another difference with the proof of \cref{lemma:main_bound_modified} is that we essentially have $\epsilon=0$-covers, which significantly simplifies the analysis. All the rest of the proof holds by using $\tilde\Fcal:=\{\1[f\neq g]:f,g\in\Fcal\}$ instead of $\Fcal_p$. Altogether, we obtain that with probability at least $1-\delta$,
    \begin{align*}
        \sum_{a=1}^A \gamma_a &\lesssim \frac{\ln^2 T}{q\sigma} \paren{\ln\Ebb_\mu\sqb{W_{8T\ln (\frac{T}{\delta})/\sigma} \paren{ \tilde \Fcal}} + \ln \frac{T}{\delta}+w}\\
        &\lesssim \frac{\ln^2 T}{q\sigma} \paren{ d\ln\paren{\frac{T}{\sigma}\ln\frac{1}{\delta}}  + \ln \frac{T}{\delta}+w}.
    \end{align*}
    where in the last inequality we use the fact that $\tilde\Fcal$ has VC dimension at most $2d$ and \cref{prop:bounding_Will_functional} to bound the Wills functional. Furthering the bounds with the same arguments as in the proof of \cref{lemma:main_bound_modified} and letting $w=w(T,\delta) =  d\ln\paren{\frac{T}{\sigma}\ln\frac{1}{\delta}}  + \ln \frac{T}{\delta}+2\geq 2$ ends the proof.

    For the bound in expectation, it suffices to take $w=w(T)\geq 2$ and use the high probability bound with $\delta=1/T$, which implies
    \begin{equation*}
        \Ebb \abs{ \set{k\in[K]: \sum_{t=T_{k-1}+1}^{T_k} \gamma_{T_{k-1}}(t) \cdot \1[\gamma_{T_{k-1}}(t) \geq q] \geq w(T,\delta)  }} \leq \delta T + C\frac{\ln^2 T}{q\sigma} \lesssim \frac{\ln^2 T}{q\sigma}.
    \end{equation*}
\end{proof}

We are now ready to prove the main regret bound for \textsc{Cover}, which is also essentially the same as for its counterpart within \cref{sec:regret_bound}.

\vspace{3mm}

\begin{proof}[of \cref{thm:regret_simple_algo}]
    Again, this is a simplified version of the proof of \cref{thm:main_thm}. Fix $f^\star\in\Fcal$. Instead of using \cref{lemma:regret_exponentially_weighted}, we can simply use the equivalent classical regret bound for the Hedge algorithm. Taking the union bound over all runs of Hedge on each epoch $k\in[K]$ and assuming that $K\leq T$, the regret decomposition \cref{eq:decomposed_regret} simply becomes with probability at least $1-\delta T$,
    \begin{align*}
        \sum_{t=1}^T \ell_t(\hat y_t) - \ell_t(f^\star(x_t)) &\lesssim \sum_{k=1}^K \sqrt{(T_k-T_{k-1}) d\ln T} + K\ln\frac{T}{\delta} + \sum_{k=1}^K\sum_{t=T_{k-1}+1}^{T_k} \ell_t\paren{f_{k,S}(x_t)} - \ell_t(f^\star(x_t))\\
        &\lesssim \sqrt{K d T\ln T} + K\ln\frac{T}{\delta} + \sum_{k=1}^K\sum_{t=T_{k-1}+1}^{T_k} \ell_t\paren{f_{k,S}(x_t)} - \ell_t(f^\star(x_t)).
    \end{align*}
    where we denoted by $f_{k,S}$ the function from the cover constructed at the beginning of epoch $(T_{k-1},T_k]$ that had the same values as $f^\star$ on prior epoch queries (see line 3 of \cref{alg:simplified_algo}). In the last inequality we use Jensen's inequality. The exact same arguments as for bounding $\Lambda_k^{(p)}$ in \cref{lemma:bounding_deltas} imply that with probability at least $1-\delta$,
    \begin{equation*}
        \sum_{t=T_{k-1}+1}^{T_k} \ell_t\paren{f_{k,S}(x_t)} - \ell_t(f^\star(x_t)) \leq 2 \sum_{t=T_{k-1}+1}^{T_k} \gamma_{T_{k-1}}(t) + 3\ln\frac{T}{\delta},\quad k\in[K].
    \end{equation*}

    From there the rest of the proof is essentially the same as for \cref{thm:main_thm}. As in \cref{eq:bound_for_Lambda}, \cref{prop:simplified} together with the bound above implies that with probability at least $1-\delta$,
    \begin{equation*}
        \abs{\set{ k\in[K]: \sum_{t=T_{k-1}+1}^{T_k} \ell_t\paren{f_{k,S}(x_t)} - \ell_t(f^\star(x_t)) \geq 5q(T_{k-1}-T_k)}} \leq C\frac{\ln^2 T}{q\sigma}
    \end{equation*}
    whenever $q\geq C'\frac{w(T,\delta)}{L-1}$ where $C,C'>0$ are some universal constants, $w(T,\delta)$ is as defined in \cref{prop:simplified}, and $L=\max_{k\in[K]} T_k-T_{k-1}$ is the minimum length of a period. Note that because $K\leq T$, we have $L=\ceil{T/K}$. From there, as when bounding the terms $\Lambda_k^{(p)}$, we define
    \begin{equation*}
        q_0:=C_1 \cdot \max\paren{\frac{w(T,\delta)}{L-1},\frac{\ln^3T}{\sigma K}},
    \end{equation*}
    where $C_1$ is a constant that may depend on the constants $C,C'$ from above. Then, we obtain that on an event of probability at least $1-c\delta\ln T$ for some constant $c>0$, we have
    \begin{equation*}
        \sum_{k=1}^K\sum_{t=T_{k-1}+1}^{T_k} \ell_t\paren{f_{k,S}(x_t)} - \ell_t(f^\star(x_t)) \leq (1+c)q_0 T.
    \end{equation*}
    This is the equivalent of \cref{eq:bound_Lambda_final}. Altogether, we obtain that with probability at least $1-\delta$,
    \begin{align*}
        \sum_{t=1}^T \ell_t(\hat y_t) - \ell_t(f^\star(x_t)) \lesssim \sqrt{K d T\ln T} + K w(T,\delta) + \frac{\ln^3 T}{\sigma K}T.
    \end{align*}
    We then take the value $K=\lfloor\ln T \cdot (T/d)^{1/3} \sigma^{-2/3} \rfloor$. Note that of $K\geq T$, the regret bound from \cref{thm:regret_simple_algo} is immediate. This is also the case if $d/\sigma\gtrsim T$. Hence, from now we suppose that $K\leq T$ and $d/\sigma \leq T$. Then, we obtain with probability at least $1-\delta$,
    \begin{equation*}
         \sum_{t=1}^T \ell_t(\hat y_t) - \ell_t(f^\star(x_t)) \lesssim \ln^2 T  \paren{\frac{d T^2}{\sigma}}^{1/3} + K w(T,\delta).
    \end{equation*}
    We then turn this oblivious regret guarantee into an adaptive regret guarantee using the same arguments as for \cref{thm:main_thm} in \cref{subsec:oblivious_to_adaptive}. Altogether, we obtain that with probability at least $1-\delta$,
    \begin{equation*}   
        \sum_{t=1}^T \ell_t(\hat y_t) - \inf_{f\in\Fcal}\sum_{t=1}^T \ell_t(f(x_t)) \lesssim \ln^2 T  \paren{\frac{d T^2}{\sigma}}^{1/3} + K d\ln\frac{T}{\delta} \lesssim \ln T  \paren{\frac{d T^2}{\sigma}}^{1/3} \ln\frac{T}{\delta}.
    \end{equation*}
    In the last inequality we used $d/\sigma\leq T$. This ends the proof of the theorem.
\end{proof}

\section{Technical lemmas}
\label{sec:technical_lemmas}

\subsection{Covering bounds}

We first state the classical Sauer-Shelah's lemma \cite{sauer1972density,shelah1972combinatorial} which bounds the size of the projection of a function class with finite VC dimension onto a set $\{x_1,\ldots,x_n\}\subset \Xcal$.

\begin{lemma}[Sauer-Shelah's lemma]\label{lemma:sauer_lemma}
    Let $\Fcal$ be a function class from $\Xcal$ to $\{0, 1\}$ of VC dimension $d$. Then, for any $x_1,\ldots,x_n\in\Xcal$,
    \begin{equation*}
        \abs{\set{(f(x_i))_{i\in[n]},f\in\Fcal}} \leq \sum_{i=0}^d \binom{n}{i}.
    \end{equation*}
    In particular, the above quantity is bounded by $2n^d$ and if $n\geq d$ it is bounded by $\paren{\frac{2en}{d}}^d$.
\end{lemma}

Generalizing Sauer-Shelah's lemma, it is known that the fat-shattering dimension can be used to bound the size of empirical covers regression function classes.

\begin{theorem}[Theorem 4.4 from \cite{rudelson2006combinatorics}]
\label{thm:fat_shattering_bound}
    Let $\Fcal:\Xcal\to [0,1]$ be a function class and let $S\subset \Xcal$ be a finite set. Then, for any $\alpha\in(0,1)$ there are constants $c,C>0$ such that
    \begin{equation*}
        \ln \Ncal(\Fcal;\epsilon,S) \lesssim \mathsf{fat}_\Fcal(c\alpha\epsilon) \ln^{1+\alpha}\paren{\frac{C|S|}{\mathsf{fat}_\Fcal(c\epsilon) \epsilon}}.
    \end{equation*}
\end{theorem}

\subsection{Concentration inequalities}

We first state Freedman's inequality \cite{freedman1975tail} which gives tail probability bounds for martingales. The following statement is for instance taken from \cite[Theorem 1]{beygelzimer2011contextual} or \cite[Lemma 9]{agarwal2014taming}.
    
    \begin{lemma}[Freedman's inequality]\label{lemma:freedman_inequality}
        Let $(Z_t)_{t\in T}$ be a real-valued martingale difference sequence adapted to filtration $(\Fcal_t)_t$. If $|Z_t|\leq R$ almost surely, then for any $\eta\in(0,1/R)$ it holds that with probability at least $1-\delta$,
        \begin{equation*}
            \sum_{t=1}^T Z_t \leq \eta \sum_{t=1}^T \Ebb[Z_t^2\mid\Fcal_{t-1}] + \frac{\ln1/\delta}{\eta}.
        \end{equation*}
    \end{lemma}

\subsection{Technical tools for smoothed analysis}

For our purposes, we need strengthened versions of tools that were used in prior works on smoothed online learning. We start by giving a strengthened version of \cite[Lemma 3]{block2024performance}.

\begin{lemma}\label{lemma:decoupling_stronger}
    Let $(x_t)\subset \Xcal$ be a sequence of random variables and let $g_t:\Xcal\to[0,1]$ be a sequence of random functions adapted to a filtration $(\Hcal_t)_{t\geq 0}$ such that $g_t$ is $\Hcal_{t-1}$-measurable and $x_t\mid (\Hcal_{t-1},g_t)$ is $\sigma$-smooth with respect to some measure $\mu$. Let $x_s'$ be a tangent sequence. Finally, let $\tau$ be a stopping time for the filtration $(\Hcal_t)_{t\geq 0}$. Then,
    \begin{equation*}
        \sum_{t=1}^{\tau} \Ebb[ g_t(x_t)\mid \Hcal_{t-1},g_t] \leq 3 \sqrt{\frac{\tau (1+2\ln\tau)}{\sigma}\paren{1+\ln\tau + \sum_{t=1}^{\tau}  \frac{1}{t}\sum_{s=1}^{t-1} \Ebb[g_t(x_s')\mid \Hcal_{t-1},g_t] } }.
    \end{equation*}
\end{lemma}

As an important remark, compared to \cite[Lemma 3]{block2024performance}, the bound from \cref{lemma:decoupling_stronger} has an improved dependency in $\sigma$. The bound is proportional $1/\sqrt \sigma$ instead of $1/\sigma$, which is needed to achieve the tight regret bounds from \cref{thm:main_thm}. Indeed, the lower bound from \cref{thm:lower_bound} also grows as $1/\sqrt \sigma$.

To prove \cref{lemma:decoupling_stronger} we first need to generalize \cite[Lemma 2]{block2024performance} as follows.
\begin{lemma}\label{lemma:few_surprises_lemma}
    Let $(a_t)_{t\in\Nbb}$ be a sequence of numbers in $[0,1]$ such that $a_0>0$. For $K\geq 1$ and $T\geq 1$, define
    \begin{equation*}
        B_T(a,K) := \set{t\in [T] : a_t\geq \frac{K}{t}\sum_{s=0}^T a_s}.
    \end{equation*}
    Then, for any $\epsilon\in(0,1]$, it holds that $|B_T(a,K)|\leq \epsilon T + \ln\frac{T}{a_0}$ for any $K \geq \frac{1}{\epsilon}\ln\frac{T}{a_0}$.
\end{lemma}

\begin{proof}
    The proof is a simple adaptation from that of \cite[Lemma 2]{block2024performance}, we only detail the modifications. As in the original proof, we define a new sequence $(b_t)_{t\in\{0,\ldots,T\}}$ such that $b_0=a_0$ and for $t\in[T]$,
    \begin{equation*}
        b_t=\begin{cases}
            0 & t\notin B_T(a,K)\\
            \frac{K}{t}\sum_{s=0}^t b_s &t\in B_T(a,K).
        \end{cases}
    \end{equation*}
    Their arguments show that $b_t\in[0,1]$ for all $t\in[T]$ and $B_T(a,K)= B_T(b,K)$ hence it suffices to focus on the sequence $b$. We enumerate $B_T(b,K) = \{t_1<\ldots<t_i\}\subset[T]$. Their arguments show that
    \begin{equation*}
        1\geq b_{t_i} = \frac{K}{t_i}\cdot\prod_{j=1}^{i-1}\paren{1+\frac{K}{t_j}} b_0.
    \end{equation*}
    We recall that $b_0=a_0$. Following their arguments, we obtain
    \begin{equation*}
        |B_T(a,K)| = i \leq \frac{\ln\frac{T}{Ka_1}}{\ln\paren{1+\frac{K}{T}}} \leq \paren{ \frac{T}{K} + 1} \ln\frac{T}{Ka_1} \leq  \paren{ \frac{T}{K} + 1} \ln\frac{T}{a_1},
    \end{equation*}
    where in the second inequality we used $\ln(1+x) \geq \frac{x}{1+x}$ for all $x\geq 0$. This ends the proof.
\end{proof}

We are now ready to prove \cref{lemma:decoupling_stronger}.
The proof is essentially the same as \cite[Lemma 3]{block2024performance}, we give it for completeness.

\vspace{3mm}

\begin{proof}[of \cref{lemma:decoupling_stronger}]
    Using the same notations as in \cite{block2024performance}, let $p_t$ denote the law of $x_t$ conditioned on $\sigma(\Hcal_{t-1},g_t)$. By assumption, $\tau$ is a stopping, hence $\{\tau\geq t\}$ is $\Hcal_{t-1}$-measurable. Then, denoting by $Z\sim \mu$ a random variable independent from $(x_t,g_t)_{t\geq 0}$ we have
    \begin{align*}
        \sum_{t=1}^{\tau} \Ebb[ g_t(x_t)\mid \Hcal_{t-1},g_t]
        =\sum_{t=1}^{\tau}  \Ebb_Z\sqb{ \frac{dp_t}{d\mu}(Z)g_t(Z) \mid p_t,g_t } = \Ebb_{Z,g_t}\sqb{\sum_{t=1}^{\tau} \frac{dp_t}{d\mu}(Z)g_t(Z) \mid \tau,g_t,p_t,t\leq \tau}.
    \end{align*}
    Next, for any $K= \frac{1}{\epsilon}(1+\ln \frac{\tau}{\sigma}) \geq 1$ where $\epsilon\in(0,1]$ will be specified later, we define $B_\tau(K)$ as in \cref{lemma:few_surprises_lemma} to the sequence $(\sigma\frac{dp_t}{d\mu}(Z))_{t\in[\tau]}$ augmented with the value $a_0=\sigma$ at $t=0$. That is, we let
    \begin{equation*}
        B_\tau(K):=\set{t\leq \tau:\frac{dp_t}{d\mu}(Z) \geq  \frac{K}{t} \paren{1 + \sum_{s<t} \frac{dp_s}{d\mu}(Z) }}.
    \end{equation*}
    Note that because $(x_t)_{t\in[T]}$ is $\sigma$-smooth, the constructed sequence has values in $[0,1]$. Then, \cref{lemma:few_surprises_lemma} shows that $|B_\tau(K)|\leq \epsilon\tau + \ln\frac{T}{\sigma}$.
    Furthering the previous bounds and taking $K= 2\ln(\tau)/\epsilon$, we then obtain
    \begin{align}
        \sum_{t=1}^{\tau} \frac{dp_t}{d\mu}(Z)g_t(Z)
        &\overset{(i)}{\leq} \frac{|B_\tau(K)|}{\sigma} +\sum_{t=1}^{\tau}  \frac{K}{t}\paren{1+\sum_{s=1}^{t-1} \frac{dp_s}{d\mu}(Z)g_t(Z) } \notag \\
        &\leq \frac{\epsilon\tau+\ln\frac{\tau}{\sigma}}{\sigma} + \frac{1+\ln\frac{\tau}{\sigma}}{\epsilon}\paren{ 1+\ln\tau +\sum_{t=1}^{\tau}  \frac{1}{t}\sum_{s=1}^{t-1} \frac{dp_s}{d\mu}(Z)g_t(Z)}. \label{eq:bound_deterministic}
    \end{align}
    In $(i)$ we used the fact that $g_t$ has values in $[0,1]$ and that the process $(x_t)_t$ is $\sigma$-smooth. The additional $1$ comes from the fact that $\tau\notin B_\tau(K)$. We take the value 
    \begin{equation*}
        \epsilon = \sqrt{\frac{\sigma(1+2\ln\tau)}{\tau}\paren{1+\ln\tau + \sum_{t=1}^{\tau}  \frac{1}{t}\sum_{s=1}^{t-1} \Ebb[g_t(x_s')\mid \Hcal_{t-1},g_t] } }.
    \end{equation*}
    Note that if $\epsilon>1$, the bound from \cref{lemma:decoupling_stronger} is immediate since $\sigma\in(0,1]$ and we could have bounded the sum by $\tau$ directly. Similarly, if $\sigma\leq 1/\tau$ the bound is also immediate. Therefore, from now we suppose that $\epsilon\leq 1$ and $\sigma\geq 1/\tau$. In particular, this implies that $\epsilon\tau \geq \ln\tau$
    Then, taking the expectation over $Z$ in \cref{eq:bound_deterministic} gives
    \begin{align*}
        \sum_{t=1}^{\tau} \Ebb[ g_t(x_t)\mid \Hcal_{t-1},g_t] &\leq \frac{\epsilon\tau +2 \ln\tau}{\sigma} + \frac{1+2\ln\tau}{\epsilon}\paren{ 1+\ln\tau+\sum_{t=1}^{\tau}  \frac{1}{t}\sum_{s=1}^{t-1} \Ebb[g_t(x_s')\mid \Hcal_{t-1},g_t] }\\
        &\leq \frac{2\epsilon\tau}{\sigma} + \frac{1+2\ln\tau}{\epsilon}\paren{ 1+\ln\tau+\sum_{t=1}^{\tau}  \frac{1}{t}\sum_{s=1}^{t-1} \Ebb[g_t(x_s')\mid \Hcal_{t-1},g_t] }\\
        &\leq 3 \sqrt{\frac{\tau (1+2\ln\tau)}{\sigma}\paren{1+\ln\tau + \sum_{t=1}^{\tau}  \frac{1}{t}\sum_{s=1}^{t-1} \Ebb[g_t(x_s')\mid \Hcal_{t-1},g_t] } }.
    \end{align*}
    This gives the desired result.
\end{proof}

Next, we provide a high-probability version of \cite[Theorem 2]{block2024performance}. As a remark, this is only needed to obtain our high-probability oblivious regret bounds. In order to get expected oblivious regret bounds it suffices to use \cite[Theorem 2]{block2024performance} directly. This is however necessary to obtain our adaptive regret bounds in the case of function classes $\Fcal$ with finite VC dimension, since these use the high-probability oblivious regret bounds to achieve low regret compared to a covering of the function class $\Fcal$. 

\begin{lemma}\label{lemma:block_whp}
    Let $\Fcal:\Xcal\to [0,1]$ be a function class and let $(x_t)_{t\in[T]}$ be a smooth stochastic process with respect to some base measure $\mu$ on $\Xcal$. Denote by $(x_t')_{t\in[T]}$ a tangent sequence to $(x_t)_{t\in[T]}$. Then, there exists a constant $C_0\geq 1$ such that for any $c>0$ and $\delta\in(0,1/2]$, with probability at least $1-\delta$,
    \begin{equation*}
        \sup_{f\in\Fcal} \sum_{t=1}^\tau f(x_t') - (1+2c)f(x_t) \leq  C_0\frac{(1+c)^2}{c}\paren{  \ln \Ebb_\mu\sqb{W_{2T\ln (\frac{T}{\delta})/\sigma}\paren{\frac{c}{1+c}\Fcal}}+\frac{1}{c}\ln\frac{1}{\delta} }.
    \end{equation*}
\end{lemma}

Note that compared to \cite[Theorem 2]{block2024performance}, \cref{lemma:block_whp} bounds the sum of the values $f(x_t)-(1+2c)f(x_t')$ instead of $f^2(x_t)-(1+2c)f^2(x_t')$. Up to considering the function class $\Fcal^2=\{f^2:f\in\Fcal\}$, this implies the same result up to constants in light of \cite[Theorem 4.1]{mourtada2023universal} which implies that for any $1$-Lipschitz real-valued function $\psi$, we have $W_m(\psi\circ\Fcal)\leq W_m(\Fcal)$ where $\psi\circ\Fcal = \{\psi\circ f:f\in\Fcal\}$.

\vspace{3mm}

\begin{proof}
    We follow similar arguments as in the proof of \cite[Theorem 2]{block2024performance}. At the high level, the result is obtained by following the proof therein and turning each expectation step to a high-probability one. Using the same notations therein, their proof implies that the left hand side $\sup_{f\in\Fcal} \sum_{t=1}^\tau f(x_t') - (1+2c)f(x_t)$ has the same distribution as
    \begin{multline*}
        \sup_{f\in\Fcal} \sum_{t=1}^\tau (1+c) \epsilon_t \paren{f(\mb x_t(\epsilon)) - f(\mb x_t'(\epsilon))} - c\paren{f(\mb x_t(\epsilon)) + f(\mb x_t'(\epsilon))}\\
        \leq \underbrace{\sup_{f\in\Fcal} \set{\sum_{t=1}^\tau (1+c) \epsilon_t f(\mb x_t(\epsilon)) -c f(\mb x_t(\epsilon)) }}_A  +  \underbrace{\sup_{f\in\Fcal} \set{\sum_{t=1}^\tau -(1+c) \epsilon_t f(\mb x_t(\epsilon)) -c f(\mb x_t(\epsilon)) }}_{A'}.
    \end{multline*}
    They then note that $A$ and $A'$ have the same distribution by the symmetry of the Rademacher variables $\epsilon_t$, hence we can focus on bounding $A$ then use the union bound. 
    Now introduce i.i.d. standard Gaussians $\xi_1,\ldots,\xi_T$ independent from all other random variables. We also fix a function $\hat f\in\Fcal$ such that
    \begin{equation*}
        \sum_{t=1}^T (1+c) \epsilon_t \hat f(\mb x_t(\epsilon)) -c \hat f(\mb x_t(\epsilon)) \geq (1-\eta) A,
    \end{equation*}
    for some fixed parameter $\eta\in(0,1)$.  Conditionally on other variables, including $\hat f$, the variables $|\xi_1|,\ldots,|\xi_T|$ are still i.i.d. and we note that $\sqrt{\frac{\pi}{2}}(1+c)\epsilon_t |\xi_t| \hat f(\mb x_t(\epsilon))$ is sub-Gaussian with parameter $C(1+c)^2\hat f^4(\mb x_t(\epsilon))$ for some universal constant $C\geq 1$. Applying the classical concentration bound for independent sub-Gaussian random variables, we obtain
    \begin{align}\label{eq:concentration_subgaussian}
        \sum_{t=1}^T \epsilon_t \paren{\sqrt{\frac{\pi}{2}}|\xi_t|-1 } \hat f(\mb x_t(\epsilon))
        \leq \sqrt{2C_1(1+c)^2\sum_{t=1}^T \hat f(\mb x_t(\epsilon)) \cdot \ln\frac{1}{\delta}} .
    \end{align}
    Here, we also used the fact that $\hat f$ takes values in $[0,1]$.
    Denote by $\Fcal_\delta$ this event. We next consider the event
    \begin{equation*}
        \Gcal_\delta:=\set{ \sum_{t=1}^T \hat f(\mb x_t(\epsilon)) \leq \frac{8C_1(1+c)^2}{c^2}\ln\frac{1}{\delta} }.
    \end{equation*}
    Note that on the event $\Gcal_\delta$, we directly have
    \begin{equation*}
        A\leq \frac{1}{1-\eta}\sum_{t=1}^T \hat f(\mb x_t(\epsilon)) \leq \frac{8C_1(1+c)^2}{c^2(1-\eta)}\ln\frac{1}{\delta}
    \end{equation*}
    On the other hand, on $\Fcal_\delta \cap \Gcal_\delta^c$, we can further bound \cref{eq:concentration_subgaussian} by $\frac{c}{2}\sum_{t=1}^T \hat f(\mb x_t(\epsilon))$. Then, we obtain
    \begin{align*}
        A&\leq \frac{1}{1-\eta}\paren{\sum_{t=1}^T (1+c) \epsilon_t \hat f(\mb x_t(\epsilon)) -c \hat f(\mb x_t(\epsilon))  } \\
        &\leq  \frac{1}{1-\eta}\paren{\sum_{t=1}^T \sqrt{\frac{\pi}{2}}(1+c) \epsilon_t |\xi_t| \hat f(\mb x_t(\epsilon)) -\frac{c}{2} \hat f(\mb x_t(\epsilon))  }\\
        &\leq  \frac{1}{1-\eta}\paren{\sum_{t=1}^T \sqrt{\frac{\pi}{2}}(1+c) \epsilon_t |\xi_t| \hat f(\mb x_t(\epsilon)) -\frac{c}{2} \hat f^2(\mb x_t(\epsilon))  }\\
        &\leq\frac{\pi(1+c)^2}{2c(1-\eta)}  \underbrace{ \sup_{f\in\Fcal}\sum_{t=1}^T c' \epsilon_t |\xi_t|  f(\mb x_t(\epsilon)) -\frac{c'^2}{2} f^2(\mb x_t(\epsilon)) }_B ,
    \end{align*}
    where $c'=\sqrt{\frac{2}{\pi}}\frac{c}{1+c}$. In the third inequality we used the fact that the functions have values in $[0,1]$.
    Note that $\epsilon_t |\xi_t|$ has the same distribution as $\xi_t$. Hence defining $\mb x_t(\xi):=\mb x_t(\text{sign}(\xi))$,  $B$ has the same distribution as
    if we replaced $\epsilon_t|\xi_t|$ by $\xi_t$, and replaced $\mb x_t(\epsilon)$ with $\mb x_t(\xi)$. 
    Let $\Ecal_\delta$ be the same event as in the proof of \cite[Theorem 2]{block2024performance} on which $\mb x_t(\epsilon)\in\{Z_{t,j},j\in[k]\}$ for all $t\in[T]$, where $k=\ceil{\frac{1}{\sigma}\ln\frac{T}{\delta}}$. We have $\Pbb(\Ecal)\geq 1-Te^{-\sigma k}\geq 1-\delta$. Then, the arguments in \cite[Theorem 2]{block2024performance} show that
    \begin{align*}
        \Ebb\sqb{\exp\paren{\1[\Ecal_\delta]\cdot B}} &\leq \Ebb_{Z_{t,j}\sim\mu} W_{kT}\paren{c'\cdot\Fcal}.
    \end{align*}
    In particular, Markov's inequality shows that with probability at least $1-\delta$,
    \begin{equation*}
        \1[\Ecal_\delta]\cdot B \leq \ln \Ebb_{Z_{t,j}\sim\mu} W_{kT}\paren{c'\cdot\Fcal} + \ln\frac{1}{\delta}.
    \end{equation*}
    Denote by $\Hcal_\delta$ this event. Putting everything together shows that on $\Ecal_\delta\cap\Fcal_\delta\cap\Hcal_\delta$,
    \begin{equation*}
        A \leq \frac{8C_1(1+c)^2}{c^2(1-\eta)}\ln\frac{1}{\delta} + \frac{\pi(1+c)^2}{2c(1-\eta)} \paren{\ln \Ebb_{Z_{t,j}\sim\mu} W_{kT}\paren{c'\cdot\Fcal} + \ln\frac{1}{\delta}},
    \end{equation*}
    which has probability at least $1-3\delta$. We then use the union bound to similarly bound $A'$. This shows that for some universal constant $C_2\geq 1$, with probability at least $1-6\delta$,
    \begin{equation*}
        \sup_{f\in\Fcal} \sum_{t=1}^T  f(x_t)- (1+2c)f(x_t') \leq C_2\paren{\frac{(1+c)^2}{c} \ln \Ebb_{Z_{t,j}\sim\mu} W_{2T\ln(\frac{T}{\delta}) /\sigma }\paren{c'\cdot\Fcal} + \frac{(1+c)^2}{c^2}\ln\frac{1}{\delta} }.
    \end{equation*}
    Noting that $c'\leq \frac{c}{1+c}$, this gives the desired result.
\end{proof}

\end{document}

%% file: shortcuts.tex
\newcommand{\trw}{\text{\small TRW}}
\newcommand{\maxcut}{\text{\small MAXCUT}}
\newcommand{\maxcsp}{\text{\small MAXCSP}}
\newcommand{\suol}{\text{SUOL}}
\newcommand{\wuol}{\text{WUOL}}
\newcommand{\crf}{\text{CRF}}
\newcommand{\sual}{\text{SUAL}}
\newcommand{\suil}{\text{SUIL}}
\newcommand{\fs}{\text{FS}}
\newcommand{\fmv}{{\text{FMV}}}
\newcommand{\smv}{{\text{SMV}}}
\newcommand{\wsmv}{{\text{WSMV}}}
\newcommand{\trwp}{\text{\small TRW}^\prime}
\newcommand{\alg}{\text{ALG}}
\newcommand{\rhos}{\rho^\star}
\newcommand{\brhos}{\brho^\star}
\newcommand{\bzero}{{\mathbf 0}}
\newcommand{\bs}{{\mathbf s}}
\newcommand{\bw}{{\mathbf w}}
\newcommand{\bws}{\bw^\star}
\newcommand{\ws}{w^\star}
\newcommand{\Prt}{{\mathsf {Part}}}
\newcommand{\Fs}{F^\star}

\newcommand{\Hs}{{\mathsf H} }

\newcommand{\hL}{\hat{L}}
\newcommand{\hU}{\hat{U}}
\newcommand{\hu}{\hat{u}}

\newcommand{\bu}{{\mathbf u}}
\newcommand{\ubf}{{\mathbf u}}
\newcommand{\hbu}{\hat{\bu}}

\newcommand{\primal}{\textbf{Primal}}
\newcommand{\dual}{\textbf{Dual}}

\newcommand{\Ptree}{{\sf P}^{\text{tree}}}
\newcommand{\bv}{{\mathbf v}}

\newcommand{\bq}{\boldsymbol q}

\newcommand{\rvM}{\text{M}}

\newcommand{\Acal}{\mathcal{A}}
\newcommand{\Bcal}{\mathcal{B}}
\newcommand{\Ccal}{\mathcal{C}}
\newcommand{\Dcal}{\mathcal{D}}
\newcommand{\Ecal}{\mathcal{E}}
\newcommand{\Fcal}{\mathcal{F}}
\newcommand{\Gcal}{\mathcal{G}}
\newcommand{\Hcal}{\mathcal{H}}
\newcommand{\Ical}{\mathcal{I}}
\newcommand{\Kcal}{\mathcal{K}}
\newcommand{\Lcal}{\mathcal{L}}
\newcommand{\Mcal}{\mathcal{M}}
\newcommand{\Ncal}{\mathcal{N}}
\newcommand{\Pcal}{\mathcal{P}}
\newcommand{\Scal}{\mathcal{S}}
\newcommand{\Tcal}{\mathcal{T}}
\newcommand{\Ucal}{\mathcal{U}}
\newcommand{\Vcal}{\mathcal{V}}
\newcommand{\Wcal}{\mathcal{W}}
\newcommand{\Xcal}{\mathcal{X}}
\newcommand{\Ycal}{\mathcal{Y}}
\newcommand{\Ocal}{\mathcal{O}}
\newcommand{\Qcal}{\mathcal{Q}}
\newcommand{\Rcal}{\mathcal{R}}

\newcommand{\brho}{\boldsymbol{\rho}}

\newcommand{\Cbb}{\mathbb{C}}
\newcommand{\Ebb}{\mathbb{E}}
\newcommand{\Nbb}{\mathbb{N}}
\newcommand{\Pbb}{\mathbb{P}}
\newcommand{\Qbb}{\mathbb{Q}}
\newcommand{\Rbb}{\mathbb{R}}
\newcommand{\Sbb}{\mathbb{S}}
\newcommand{\Vbb}{\mathbb{V}}
\newcommand{\Wbb}{\mathbb{W}}
\newcommand{\Xbb}{\mathbb{X}}
\newcommand{\Ybb}{\mathbb{Y}}
\newcommand{\Zbb}{\mathbb{Z}}

\newcommand{\Rbbp}{\Rbb_+}

\newcommand{\bX}{{\mathbf X}}
\newcommand{\bx}{{\boldsymbol x}}

\newcommand{\btheta}{\boldsymbol{\theta}}

\newcommand{\Pb}{\mathbb{P}}

\newcommand{\hPhi}{\widehat{\Phi}}

\newcommand{\Sigmah}{\widehat{\Sigma}}
\newcommand{\thetah}{\widehat{\theta}}

\newcommand{\indep}{\perp \!\!\! \perp}
\newcommand{\notindep}{\not\!\perp\!\!\!\perp}

\newcommand{\one}{\mathbbm{1}}
\newcommand{\1}{\mathbbm{1}}
\newcommand{\aprx}{\alpha}

\newcommand{\ST}{\Tcal(\Gcal)}
\newcommand{\x}{\mathsf{x}}
\newcommand{\y}{\mathsf{y}}
\newcommand{\Ybf}{\textbf{Y}}
\newcommand{\smiddle}[1]{\;\middle#1\;}

\definecolor{dark_red}{rgb}{0.2,0,0}
\newcommand{\detail}[1]{\textcolor{dark_red}{#1}}

\newcommand{\ds}[1]{{\color{red} #1}}
\newcommand{\rc}[1]{{\color{green} #1}}

\newcommand{\mb}[1]{\ensuremath{\boldsymbol{#1}}}

\newcommand{\metric}{\rho}
\newcommand{\proj}{\text{Proj}}

\newcommand{\paren}[1]{\left( #1 \right)}
\newcommand{\sqb}[1]{\left[ #1 \right]}
\newcommand{\set}[1]{\left\{ #1 \right\}}
\newcommand{\floor}[1]{\left\lfloor #1 \right\rfloor}
\newcommand{\ceil}[1]{\left\lceil #1 \right\rceil}
\newcommand{\abs}[1]{\left|#1\right|}
\newcommand{\norm}[1]{\left\|#1\right\|}

\newcommand{\todo}[1]{{\color{red} TODO: #1}}

%% file: STOC_submission_revised.bbl
\newcommand{\etalchar}[1]{$^{#1}$}
\begin{thebibliography}{CBCC{\etalchar{+}}23}

\bibitem[ADK14]{asor2014additive}
Ohad Asor, Hubert~Haoyang Duan, and Aryeh Kontorovich.
\newblock On the additive properties of the fat-shattering dimension.
\newblock {\em IEEE transactions on neural networks and learning systems}, 25(12):2309--2312, 2014.

\bibitem[AHK{\etalchar{+}}14]{agarwal2014taming}
Alekh Agarwal, Daniel Hsu, Satyen Kale, John Langford, Lihong Li, and Robert Schapire.
\newblock Taming the monster: A fast and simple algorithm for contextual bandits.
\newblock In {\em International conference on machine learning}, pages 1638--1646. PMLR, 2014.

\bibitem[BCH22]{blanchard2021universal}
Mo{\"i}se Blanchard, Romain Cosson, and Steve Hanneke.
\newblock Universal online learning with unbounded losses: Memory is all you need.
\newblock In {\em International Conference on Algorithmic Learning Theory}, pages 107--127. PMLR, 2022.

\bibitem[BDGR22]{block2022smoothed}
Adam Block, Yuval Dagan, Noah Golowich, and Alexander Rakhlin.
\newblock Smoothed online learning is as easy as statistical learning.
\newblock In {\em Conference on Learning Theory}, pages 1716--1786. PMLR, 2022.

\bibitem[BDPSS09]{ben2009agnostic}
Shai Ben-David, D{\'a}vid P{\'a}l, and Shai Shalev-Shwartz.
\newblock Agnostic online learning.
\newblock In {\em COLT}, volume~3, page~1, 2009.

\bibitem[BHJ23a]{blanchard2023adversarial}
Mo\"{i}se Blanchard, Steve Hanneke, and Patrick Jaillet.
\newblock Adversarial rewards in universal learning for contextual bandits.
\newblock {\em arXiv preprint arXiv:2302.07186}, 2023.

\bibitem[BHJ23b]{blanchard2023contextual}
Mo\"{i}se Blanchard, Steve Hanneke, and Patrick Jaillet.
\newblock Contextual bandits and optimistically universal learning.
\newblock {\em arXiv preprint arXiv:2301.00241}, 2023.

\bibitem[BHS24]{bhatt2024smoothed}
Alankrita Bhatt, Nika Haghtalab, and Abhishek Shetty.
\newblock Smoothed analysis of sequential probability assignment.
\newblock {\em Advances in Neural Information Processing Systems}, 36, 2024.

\bibitem[BJ23]{blanchard2023universal}
Mo{\"\i}se Blanchard and Patrick Jaillet.
\newblock Universal regression with adversarial responses.
\newblock {\em The Annals of Statistics}, 51(3):1401--1426, 2023.

\bibitem[Bla22]{blanchard2022universal}
Mo\"{i}se Blanchard.
\newblock Universal online learning: An optimistically universal learning rule.
\newblock In {\em Conference on Learning Theory}, pages 1077--1125. PMLR, 2022.

\bibitem[BLL{\etalchar{+}}11]{beygelzimer2011contextual}
Alina Beygelzimer, John Langford, Lihong Li, Lev Reyzin, and Robert Schapire.
\newblock Contextual bandit algorithms with supervised learning guarantees.
\newblock In {\em Proceedings of the Fourteenth International Conference on Artificial Intelligence and Statistics}, pages 19--26. JMLR Workshop and Conference Proceedings, 2011.

\bibitem[BLM13]{boucheron2013concentration}
Stéphane Boucheron, Gábor Lugosi, and Pascal Massart.
\newblock {\em Concentration Inequalities: A Nonasymptotic Theory of Independence}.
\newblock OUP Oxford, 2013.

\bibitem[BLW94]{bartlett1994fat}
Peter~L Bartlett, Philip~M Long, and Robert~C Williamson.
\newblock Fat-shattering and the learnability of real-valued functions.
\newblock In {\em Proceedings of the seventh annual conference on Computational learning theory}, pages 299--310, 1994.

\bibitem[BP23]{block2023sample}
Adam Block and Yury Polyanskiy.
\newblock The sample complexity of approximate rejection sampling with applications to smoothed online learning.
\newblock In {\em The Thirty Sixth Annual Conference on Learning Theory}, pages 228--273. PMLR, 2023.

\bibitem[BRS24]{block2024performance}
Adam Block, Alexander Rakhlin, and Abhishek Shetty.
\newblock On the performance of empirical risk minimization with smoothed data.
\newblock {\em arXiv preprint arXiv:2402.14987}, 2024.

\bibitem[BS22]{block2022efficient}
Adam Block and Max Simchowitz.
\newblock Efficient and near-optimal smoothed online learning for generalized linear functions.
\newblock {\em Advances in neural information processing systems}, 35:7477--7489, 2022.

\bibitem[BSR23]{block2023oracle}
Adam Block, Max Simchowitz, and Alexander Rakhlin.
\newblock Oracle-efficient smoothed online learning for piecewise continuous decision making.
\newblock In {\em The Thirty Sixth Annual Conference on Learning Theory}, pages 1618--1665. PMLR, 2023.

\bibitem[CBCC{\etalchar{+}}23]{cesa2023repeated}
Nicol{\`o} Cesa-Bianchi, Tommaso~R Cesari, Roberto Colomboni, Federico Fusco, and Stefano Leonardi.
\newblock Repeated bilateral trade against a smoothed adversary.
\newblock In {\em The Thirty Sixth Annual Conference on Learning Theory}, pages 1095--1130. PMLR, 2023.

\bibitem[CBL06]{cesa2006prediction}
Nicolo Cesa-Bianchi and G{\'a}bor Lugosi.
\newblock {\em Prediction, learning, and games}.
\newblock Cambridge university press, 2006.

\bibitem[CGM19]{cortes2019relative}
Corinna Cortes, Spencer Greenberg, and Mehryar Mohri.
\newblock Relative deviation learning bounds and generalization with unbounded loss functions.
\newblock {\em Annals of Mathematics and Artificial Intelligence}, 85:45--70, 2019.

\bibitem[DGKL94]{devroye1994strong}
Luc Devroye, Laszlo Gyorfi, Adam Krzyzak, and G{\'a}bor Lugosi.
\newblock On the strong universal consistency of nearest neighbor regression function estimates.
\newblock {\em The Annals of Statistics}, pages 1371--1385, 1994.

\bibitem[DGL13]{devroye2013probabilistic}
Luc Devroye, L{\'a}szl{\'o} Gy{\"o}rfi, and G{\'a}bor Lugosi.
\newblock {\em A probabilistic theory of pattern recognition}, volume~31.
\newblock Springer Science \& Business Media, 2013.

\bibitem[DHZ23]{durvasula2023smoothed}
Naveen Durvasula, Nika Haghtalab, and Manolis Zampetakis.
\newblock Smoothed analysis of online non-parametric auctions.
\newblock In {\em Proceedings of the 24th ACM Conference on Economics and Computation}, pages 540--560, 2023.

\bibitem[DlPG12]{de2012decoupling}
Victor De~la Pena and Evarist Gin{\'e}.
\newblock {\em Decoupling: from dependence to independence}.
\newblock Springer Science \& Business Media, 2012.

\bibitem[Fre75]{freedman1975tail}
David~A Freedman.
\newblock On tail probabilities for martingales.
\newblock {\em the Annals of Probability}, pages 100--118, 1975.

\bibitem[GG09]{gray2009probability}
Robert~M Gray and RM~Gray.
\newblock {\em Probability, random processes, and ergodic properties}, volume~1.
\newblock Springer, 2009.

\bibitem[GKKW02]{gyorfi:02}
L.~Gy\"{o}rfi, M.~Kohler, A.~Krzy\.{z}ak, and H.~Walk.
\newblock {\em A Distribution-Free Theory of Nonparametric Regression}.
\newblock Springer-Verlag New York, 2002.

\bibitem[GLM99]{gyorfi1999simple}
L~Gyorfi, G{\'a}bor Lugosi, and Guszt{\'a}v Morvai.
\newblock A simple randomized algorithm for sequential prediction of ergodic time series.
\newblock {\em IEEE Transactions on Information Theory}, 45(7):2642--2650, 1999.

\bibitem[GW21]{gyorfi2021universal}
L{\'a}szl{\'o} Gy{\"o}rfi and Roi Weiss.
\newblock Universal consistency and rates of convergence of multiclass prototype algorithms in metric spaces.
\newblock {\em Journal of Machine Learning Research}, 22(151):1--25, 2021.

\bibitem[Had75]{hadwiger1975will}
Hugo Hadwiger.
\newblock Das will'sche funktional.
\newblock {\em Monatshefte f{\"u}r Mathematik}, 79(3):213--221, 1975.

\bibitem[Han21]{hanneke2021learning}
Steve Hanneke.
\newblock Learning whenever learning is possible: Universal learning under general stochastic processes.
\newblock {\em Journal of Machine Learning Research}, 22(130):1--116, 2021.

\bibitem[Hau95]{haussler1995sphere}
David Haussler.
\newblock Sphere packing numbers for subsets of the boolean n-cube with bounded vapnik-chervonenkis dimension.
\newblock {\em Journal of Combinatorial Theory, Series A}, 69(2):217--232, 1995.

\bibitem[HHSY22]{haghtalab2022oracle}
Nika Haghtalab, Yanjun Han, Abhishek Shetty, and Kunhe Yang.
\newblock Oracle-efficient online learning for beyond worst-case adversaries.
\newblock {\em arXiv preprint arXiv:2202.08549}, 2022.

\bibitem[HKSW21]{hanneke2021bayes}
Steve Hanneke, Aryeh Kontorovich, Sivan Sabato, and Roi Weiss.
\newblock {Universal Bayes consistency in metric spaces}.
\newblock {\em The Annals of Statistics}, 49(4):2129 -- 2150, 2021.

\bibitem[HRS20]{haghtalab2020smoothed}
Nika Haghtalab, Tim Roughgarden, and Abhishek Shetty.
\newblock Smoothed analysis of online and differentially private learning.
\newblock {\em Advances in Neural Information Processing Systems}, 33:9203--9215, 2020.

\bibitem[HRS24]{haghtalab2024smoothed}
Nika Haghtalab, Tim Roughgarden, and Abhishek Shetty.
\newblock Smoothed analysis with adaptive adversaries.
\newblock {\em Journal of the ACM}, 71(3):1--34, 2024.

\bibitem[KS94]{kearns1994efficient}
Michael~J Kearns and Robert~E Schapire.
\newblock Efficient distribution-free learning of probabilistic concepts.
\newblock {\em Journal of Computer and System Sciences}, 48(3):464--497, 1994.

\bibitem[Lit88]{littlestone1988learning}
Nick Littlestone.
\newblock Learning quickly when irrelevant attributes abound: A new linear-threshold algorithm.
\newblock {\em Machine learning}, 2(4):285--318, 1988.

\bibitem[McM91]{mcmullen1991inequalities}
Peter McMullen.
\newblock Inequalities between intrinsic volumes.
\newblock {\em Monatshefte f{\"u}r Mathematik}, 111:47--53, 1991.

\bibitem[Men02]{mendelson2002rademacher}
Shahar Mendelson.
\newblock Rademacher averages and phase transitions in glivenko-cantelli classes.
\newblock {\em IEEE transactions on Information Theory}, 48(1):251--263, 2002.

\bibitem[Mou23]{mourtada2023universal}
Jaouad Mourtada.
\newblock Universal coding, intrinsic volumes, and metric complexity.
\newblock {\em arXiv preprint arXiv:2303.07279}, 2023.

\bibitem[MYG96]{morvai1996nonparametric}
Guszt{\'a}v Morvai, Sidney Yakowitz, and L{\'a}szl{\'o} Gy{\"o}rfi.
\newblock Nonparametric inference for ergodic, stationary time series.
\newblock {\em The Annals of Statistics}, 24(1):370--379, 1996.

\bibitem[RST11]{rakhlin2011online}
Alexander Rakhlin, Karthik Sridharan, and Ambuj Tewari.
\newblock Online learning: Stochastic, constrained, and smoothed adversaries.
\newblock {\em Advances in neural information processing systems}, 24, 2011.

\bibitem[RV06]{rudelson2006combinatorics}
Mark Rudelson and Roman Vershynin.
\newblock Combinatorics of random processes and sections of convex bodies.
\newblock {\em Annals of Mathematics}, pages 603--648, 2006.

\bibitem[Sau72]{sauer1972density}
Norbert Sauer.
\newblock On the density of families of sets.
\newblock {\em Journal of Combinatorial Theory, Series A}, 13(1):145--147, 1972.

\bibitem[She72]{shelah1972combinatorial}
Saharon Shelah.
\newblock A combinatorial problem; stability and order for models and theories in infinitary languages.
\newblock {\em Pacific Journal of Mathematics}, 41(1):247--261, 1972.

\bibitem[SHS09]{steinwart2009learning}
Ingo Steinwart, Don Hush, and Clint Scovel.
\newblock Learning from dependent observations.
\newblock {\em Journal of Multivariate Analysis}, 100(1):175--194, 2009.

\bibitem[Val84]{valiant1984theory}
Leslie~G Valiant.
\newblock A theory of the learnable.
\newblock {\em Communications of the ACM}, 27(11):1134--1142, 1984.

\bibitem[VC71]{vapnik1971uniform}
V.~N. Vapnik and A.~Ya. Chervonenkis.
\newblock On the uniform convergence of relative frequencies of events to their probabilities.
\newblock {\em Theory of Probability \& Its Applications}, 16(2):264--280, 1971.

\bibitem[VC74]{vapnik1974theory}
Vladimir Vapnik and Alexey Chervonenkis.
\newblock Theory of pattern recognition, 1974.

\bibitem[vdVW93]{van1993bounds}
R~van~de Ven and NC~Weber.
\newblock Bounds for the median of the negative binomial distribution.
\newblock {\em Metrika}, 40:185--189, 1993.

\bibitem[Wai19]{wainwright2019high}
Martin~J Wainwright.
\newblock {\em High-dimensional statistics: A non-asymptotic viewpoint}, volume~48.
\newblock Cambridge university press, 2019.

\bibitem[Wil73]{wills1973gitterpunktanzahl}
J{\"o}rg~M Wills.
\newblock Zur gitterpunktanzahl konvexer mengen.
\newblock {\em Elemente der Mathematik}, 28:57--63, 1973.

\bibitem[XFB{\etalchar{+}}22]{xie2022role}
Tengyang Xie, Dylan~J Foster, Yu~Bai, Nan Jiang, and Sham~M Kakade.
\newblock The role of coverage in online reinforcement learning.
\newblock {\em arXiv preprint arXiv:2210.04157}, 2022.

\end{thebibliography}
